\documentclass[10pt,twocolumn,letterpaper]{article}

\usepackage{cvpr}              

\def\cvprPaperID{9080} 
\def\confName{CVPR}
\def\confYear{2022}

\usepackage{amsmath,amssymb,amsthm}
\usepackage{kbordermatrix}
\usepackage{xcolor}
\usepackage{graphicx}
\usepackage[pagebackref,colorlinks]{hyperref}
\usepackage{stackengine}

\newcommand{\vect}[1]{v(#1)}

\DeclareMathOperator{\hull}{hull}
\DeclareMathOperator{\lm}{\mathtt{LM}}
\newcommand{\tm}{\!\!\times\!\!}
\newcommand{\ig}[2]{\includegraphics[width=#1\textwidth]{./figures/#2}}

\newcommand{\cln}{\,:\,}
\newcommand{\N}{\mathbb N}
\newcommand{\B}{\mathcal B}
\newcommand{\LM}[1]{\mathtt{LM}(#1)}
\newcommand{\hide}[1]{}

\theoremstyle{plain}
\newtheorem{theorem}{Theorem}
\newtheorem{prop}{Proposition}
\theoremstyle{definition}
\newtheorem{example}{Example}
\newtheorem{definition}{Definition}
\theoremstyle{remark}

\begin{document}

\newcommand{\TitleText}{Optimizing Elimination Templates by Greedy Parameter Search}
\newcommand{\NameEM}{Evgeniy	Martyushev\\
South Ural State University\\
{\tt\small martiushevev@susu.ru}}
\newcommand{\NameJV}{Jana Vrablikova\\
 Department of Algebra \\MFF, Charles University\\
{\tt\small j.vrablikov@gmail.com}}
\newcommand{\NameTP}{Tomas Pajdla\\
CIIRC - CTU in Prague\\
{\tt\small pajdla@cvut.cz}}
\title{\TitleText\thanks{The research was supported by projects EU RDF IMPACT No.~CZ.02.1.01/0.0/0.0/15 003/0000468 and EU H2020 No.~871245 SPRING. T.~Pajdla is with the Czech Institute of Informatics, Robotics and Cybernetics, Czech Technical University in Prague.}}

\author{\NameEM \and \NameJV \and \NameTP}

\maketitle

\begin{abstract}
\noindent We propose a new method for constructing elimination templates for efficient polynomial system solving of minimal problems in structure from motion, image matching, and camera tracking. We first construct a particular affine parameterization of the elimination templates for systems with a finite number of distinct solutions. Then, we use a heuristic greedy optimization strategy over the space of parameters to get a template with a small size. We test our method on 34 minimal problems in computer vision. For all of them, we found the templates either of the same or smaller size compared to the state-of-the-art. For some difficult examples, our templates are, e.g., 2.1, 2.5, 3.8, 6.6 times smaller. For the problem of refractive absolute pose estimation with unknown focal length, we have found a template that is 20 times smaller. Our experiments on synthetic data also show that the new solvers are fast and numerically accurate. We also present a fast and numerically accurate solver for the problem of relative pose estimation with unknown common focal length and radial distortion.
\end{abstract}

\section{Introduction}
\label{sec:intro}

Many tasks in 3D reconstruction~\cite{snavely2008modeling,schonberger2016structure} and camera tracking~\cite{nister2004visual,taira2018inloc} lead to solving minimal problems~\cite{nister2004efficient,stewenius2006recent,kukelova2008automatic,byrod2008column,li2006five,kuang2013pose,saurer2015minimal,agarwal2017existence,barath2018efficient,larsson2018beyond}, which can be formulated as systems of polynomial equations.

The state-of-the-art approach to efficient solving polynomial systems for minimal problems is to use symbolic-numeric solvers based on elimination templates~\cite{kukelova2008automatic,larsson2017efficient,bhayani2020sparse}. These solvers have two main parts. In the first offline part, an elimination template is constructed. The template consists of a map (formulas) from input data to a (Macaulay) coefficient matrix. The template is the same for different input generic (noisy) data. In the second online phase, the coefficient matrix is filled by the data of a particular problem, reduced by the Gauss--Jordan (G--J) elimination and used to construct an eigenvalue/eigenvector computation problem of an action matrix that delivers the solutions of the system.

While the offline phase is not time critical, the online phase has to be computed very fast (mostly in sub-millisecond time) to be useful for robust optimization based on RANSAC schemes~\cite{fischler1981random}. Therefore, it is important to build templates (\ie, Macaulay matrices) that are as small as possible to make the G--J elimination fast. Besides the size, we also need to pay attention to building templates that lead to numerically stable computation.

\subsection{Contribution}

We develop a new approach to constructing elimination templates for efficiently solving minimal problems. First, using the general syzygy-based parameterization of elimination templates from~\cite{larsson2017efficient}, we construct a partial (but still generic enough) parameterization of templates. Then, we apply a greedy heuristic optimization over the space of parameters to find as small a template as possible.

We demonstrate our method on 34 minimal problems in geometric computer vision. For all of them, we found the templates either of the same or smaller size compared to the state-of-the-art. For some difficult examples, our templates are, \eg, 2.1, 2.5, 3.8, 6.6 times smaller. For the problem of refractive absolute pose estimation with unknown focal length, we have found a template that is 20 times smaller. Our experiments on synthetic data also show that the new solvers are fast and numerically accurate.

We propose a practical solver for the problem of relative pose estimation with unknown common focal length and radial distortion. All previously presented solvers for this problem are either extremely slow or numerically unstable.

\subsection{Related work}

Elimination templates are matrices that encode the transformation from polynomials of the initial system to polynomials needed to construct the action matrix. Knowing an action matrix, the solutions of the system are computed from its eigenvectors. \emph{Automatic generator} (AG) is an algorithm that inputs a polynomial system and outputs an elimination template for the action matrix computation.

\noindent\textbf{Automatic generators:} The first automatic generator was built in~\cite{kukelova2008automatic}, where the template was constructed iteratively by expanding the initial polynomials with their multiples of increasing degree. This AG has been widely used by the computer vision community to construct polynomial solvers for a variety of minimal problems, \eg,~\cite{bujnak20093d,bujnak2010new,kukelova2013fast,li20134,zheng2015minimal,saurer2015minimal,nakano2015globally}, see also~\cite[Tab.~1]{larsson2017efficient}. Paper~\cite{larsson2017efficient} introduced a non-iterative AG based on tracing the Gr\"obner basis construction and subsequent syzygy-based reduction. This AG allowed fast constructing templates even for hard problems. An alternative AG based on using sparse resultants was recently proposed in~\cite{bhayani2020sparse}. This method, along with~\cite{larsson2018beyond}, are currently the state-of-the-art automatic template generators.

\noindent\textbf{Improving stability:} The standard way of constructing the action matrix from a template requires performing its LU decomposition. For large templates, this operation often leads to significant round-off and truncation errors and hence to numerical instabilities. The series of papers~\cite{byrod2007improving,byrod2008column,byrod2009fast} addressed this problem and proposed several methods of improving stability, \eg, by performing a QR decomposition with column pivoting on the step of constructing the action matrix from a template.

\noindent\textbf{Optimizing formulations:} Choosing a proper formulation of a minimal problem can drastically simplify finding its solutions. Paper~\cite{kukelova2017clever} proposed the variable elimination strategy that reduces the number of unknowns in the initial polynomial system. For some problems, this strategy led to notably smaller templates~\cite{larsson2017making,kileel2018distortion}.

\noindent\textbf{Optimizing templates:} Much effort has been spent on speeding up the action matrix method by optimizing the template construction step. Paper~\cite{naroditsky2011optimizing} introduced a method of optimizing templates by removing some unnecessary rows and columns. The method~\cite{kukelova2014singly} utilized the sparsity of elimination templates by converting a large sparse template into the so-called singly-bordered block-diagonal form. This allowed splitting the initial problem into several smaller subproblems, which are easier to solve. In paper~\cite{larsson2018beyond}, the authors proposed two methods that significantly reduced the sizes of elimination templates. The first method used the so-called Gr\"obner fan of a polynomial ideal for constructing templates \wrt all possible standard bases of the quotient space. The second method went beyond Gr\"obner bases and introduced a random sampling strategy for constructing non-standard bases.

\noindent\textbf{Optimizing root solving:} Complex roots are spurious for most problems arising in applications. Paper~\cite{bujnak2012making} introduced two methods of avoiding the computation of complex roots, which resulted in a significant speed-up of polynomial solvers.

\noindent\textbf{Discovering symmetries:} Polynomial systems for certain minimal problems may have hidden symmetries. Uncovering these symmetries is another way of optimizing templates. This approach was demonstrated for the simplest partial $p$-fold symmetries in~\cite{kuang2014partial,larsson2016uncovering}. A more general case was recently investigated in~\cite{duff2021galois}. Paper~\cite{larsson2017polynomial} proposed a method of handling special polynomial systems with a (possibly) infinite subset of spurious solutions.

\noindent\textbf{The most related work:} Our work is essentially based on the results of papers~\cite{byrod2009fast,larsson2017efficient,larsson2018beyond,bhayani2020sparse}.

\section{Solving polynomial systems by templates}

Here we review solving polynomial systems with a finite number of solutions by eigendecomposition of action matrices. We also show how are the action matrices constructed using elimination templates in computer vision. We build on nomenclature from~\cite{cox2006using,Cox-IVA-2015,byrod2007improving}.

\subsection{Gr\"obner bases and action matrices}
\label{subs:gb}

Here we introduce action matrices and explain how they are related to Gr\"obner bases. 

We use $\mathbb K$ for a field, $X = \{x_1, \ldots, x_k\}$ for a set of $k$ variables, $[X]$ for the set of \emph{monomials} in $X$ and $\mathbb K[X]$ for the polynomial ring over $\mathbb K$. Let $F = \{f_1, \ldots, f_s\} \subset \mathbb K[X]$ and $J = \langle F\rangle$ for the ideal generated by $F$. A set $G \subset \mathbb K[X]$ is a \emph{Gr\"obner basis} of ideal $J$ if $J = \langle G \rangle$ and for every $f \in J \setminus \{0\}$ there is $g \in G$ such that the leading monomial of $g$ divides the leading monomial of $f$. The Gr\"obner basis $G$ is called \emph{reduced} if $c(g, \lm(g)) = 1$\footnote{We denote the coefficient of $g$ at $m$ by $c(g, m)$.} for all $g \in G$ and $\lm(g)$ does not divide any monomial of $g' \in G$ when $g' \neq g$.

For a fixed monomial ordering (see SM Sec.~\ref{sec:basic}), the reduced Gr\"obner basis is defined uniquely for each ideal. Moreover, for any polynomial ideal $J$, there are finitely many distinct reduced Gr\"obner bases, which all can be found using the \emph{Gr\"obner fan} of $J$~\cite{mora1988grobner,larsson2018beyond}.

For an ideal $J \subset \mathbb K[X]$, the \emph{quotient ring}
$\mathbb K[X]/J$ consists of all equivalence classes $[f]$ under the equivalence relation $f \sim g$ iff $f - g \in J$. If $J = \langle F \rangle$ is zero-dimensional, \ie, the set of roots of $F = 0$ is finite, then $\mathbb K[X]/J$ is a finite-dimensional vector space. Moreover, $\dim \mathbb K[X]/J$ equals the number of solutions to $F = 0$, when counting the multiplicities~\cite{Cox-IVA-2015}.

Given a Gr\"obner basis $G$ of ideal $J$, we can construct the \emph{standard (linear) basis} $\mathcal B$ of the quotient ring $\mathbb K[X]/J$ as the set of all monomials not divisible by any leading monomial from $G$, \ie,
$
\mathcal B = \{b \cln \lm(g) \nmid b, \forall g \in G\}.
$

Fix a polynomial $a \in \mathbb K[X]$ and define the linear operator
\[
T_a \colon \mathbb K[X]/J \to \mathbb K[X]/J \colon [f] \mapsto [a\cdot f].
\]
Selecting a basis in $\mathbb K[X]/J$, \eg, the standard one, allows to represent the operator $T_a$ as a $d\times d$ matrix, where $d = \dim \mathbb K[X]/J$. This matrix, which is also denoted by $T_a$, is called the \emph{action matrix} and the polynomial $a$ is called the \emph{action polynomial}.

The action matrix can be found using a Gr\"obner basis $G$ of ideal $J$ as follows. Let $\{b_1, \ldots, b_d\}$ be a basis in the quotient ring $\mathbb K[X]/J$. For a given $a$, we use $G$ to construct the \emph{normal forms} of $a\,b_i$:
\[
\overline{(a\, b_i)}^{G} = \sum_j t_{ij} b_j, \quad i = 1, \ldots, d,
\]
where $t_{ij} \in \mathbb K$. Then, we have $T_a = (t_{ij})$.

\subsection{Solving polynomial systems by action matrices}

Action matrices are useful for computing the solutions of polynomial systems with a finite number $d$ of solutions. The situation is particularly simple when (i) all solutions $p_j \in \mathbb K^k$, $j = 1, \ldots, d$, are of multiplicity one and (ii) the action polynomial $a$ evaluates to pairwise different values on the solutions, \ie, $a(p) \neq a(q)$ for all solutions $p \neq q$. Then, the action matrix $T_a$ has $d$ one-dimensional eigenspaces, and $d$ vectors $\begin{bmatrix}b_1(p_j) & \ldots & b_d(p_j)\end{bmatrix}^\top$ of polynomials $b_i$ evaluated at the solutions $p_j$, $i, j = 1, \ldots, d$, are basic vectors of the $d$ eigenspaces~\cite[p.~59 Prop.~4.7]{cox2006using}. Having one-dimensional eigenspaces leads to a straightforward method for extracting all solutions $p_j$. Thus, the classical approach to finding solutions to a polynomial system $F$ with a finite number of solutions is as follows.

\noindent {\bf 1. Choose an action polynomial $a$:} Assuming that the solutions $p_j$ are of multiplicity one, \ie, the ideal $J = \langle F \rangle$ is radical~\cite[p.~253 Prop.~7]{Cox-IVA-2015}, our goal is to choose $a$ such that it has pairwise different values $a(p_j)$. This is always possible by choosing $a = x_\ell$, \ie, a variable, after a linear change of coordinates~\cite[p.~59]{cox2006using}. As we will see, such a choice leads to a simple solving method. 

In computer vision, we are particularly interested in solving polynomial systems that consist of the union of two sets of equations $F = F_1 \cup F_2$ where $F_1$ do not depend on the image measurements (\eg, Demazure constraints on the Essential matrix~\cite{demazure1988deux}) and $F_2$ depend on the image measurements affected by random noise (\eg, linear epipolar constraints on the Essential matrix~\cite{longuet81}). Then, the linear change of coordinates can be done only once in the offline phase to transform $F_1$. In the next, we will assume that there is $a$ with pairwise distinct values on the solutions $p_j$.

\noindent{\bf 2. Choose a basis $\mathcal B$ of ${\mathbb K}[X]/J$:} There are infinitely many bases of ${\mathbb K}[X]/J$. Our goal is to choose a basis that leads to a simple and numerically stable solving method. Elements of $\mathcal B$ are equivalence classes represented by polynomials that are $\mathbb K$-linear combinations of monomials. Hence, the simplest bases consist of equivalence classes represented by monomials. It is important that ${\mathbb K}[X]/J$ has a standard monomial basis~\cite{cox2006using,Sturmfels02solvingsystems} for each reduced Gr\"obner basis. In generic situations, the elements of $\mathcal B$ represented by monomials are equivalent to (infinitely) many different linear combinations of the standard monomials and thus provide (infinitely) many different vectors to construct (infinitely) many different bases of $\mathcal B$. In the following, we assume \emph{monomial bases}, \ie, the bases consisting of the classes represented by monomials.

\noindent{\bf 3. Construct the action matrix $T_a$ \wrt $\mathcal B$:} Once $a$ and $\mathcal B$ have been chosen, it is straightforward to construct $T_a$ by the process described in Sec.~\ref{subs:gb}. However, in computer vision, polynomial systems often have the same support for different values of their coefficients. Then, it is efficient~\cite{byrod2009fast,wiesinger2015thesis} to construct $T_a$ by (i) building a Macaulay matrix $M$ using a fixed procedure -- \emph{a template} -- designed in the offline phase, and then (ii) produce $T_a$ in the online phase by the G--J elimination of $M$~\cite{kukelova2008automatic,byrod2008column}. Our main contribution, Sec.~\ref{sec:constr} and Sec.~\ref{sec:red}, in this work is an efficient approach to constructing Macaulay matrices.

\noindent{\bf 4. Computing the eigenvectors $v_j$, $j = 1, \ldots, d$ of $T_a$:} Computing the eigenvectors of $T_a$ is a straightforward task when there are $d$ one-dimensional eigenspaces.

\noindent{\bf 5. Recovering the solutions from eigenvectors:} To find the solutions, it is enough to evaluate all unknowns $x_l$, $l = 1, \ldots, k$, on the solutions $p_j$. It can be done by writing unknowns $x_l$ in the standard basis $b_i$ as $x_l = \sum_i c_{li} b_i$. Then, $x_l(p_j) = \sum_i c_{li} b_i(p_j) = \sum_i c_{li} (v_j)_i$, where $(v_j)_i$ is the $i$th element of vector~$v_j$.

\subsection{Macaulay matrices and elimination templates}

Let us now introduce Macaulay matrices and elimination templates. 

To simplify the construction, we restrict ourselves to the following assumptions: (i) the elements of basis $\mathcal B$ are represented by monomials and (ii) the action polynomial $a$ is a monomial and $a \neq 1$.

Given an $s$-tuple of polynomials $F = (f_1, \ldots, f_s)$, let $[X]_F$ be the set of all monomials from $F$. Let the cardinality $\# [X]_F$ of $[X]_F$ be $n$. Then, the \emph{Macaulay matrix} $M(F) \in \mathbb K^{s\times n}$ has coefficient $c(f_i, m_j)$, with $m_j \in [X]_F$, in the $(i, j)$ element: ${M(F)}_{ij} = c(f_i, m_j)$.

A \emph{shift} of a polynomial $f$ is a multiple of $f$ by a monomial $m \in [X]$. Let $A = (A_1, \ldots, A_s)$ be an $s$-tuple of sets of monomials $A_j \subset [X]$. We define the \emph{set of shifts} of $F$ as
\begin{equation}
A\cdot F = \{m \cdot f_j \cln m \in A_j, f_j \in F\}.
\end{equation}
Let $\mathcal B$ be a monomial basis of the quotient ring $\mathbb K[X]/\langle F \rangle$ and $a$ be an action monomial. The sets $\mathcal B$, $\mathcal R = \{a\,b \cln b \in \mathcal B\} \setminus \mathcal B$, and $\mathcal E = [X]_{A\cdot F} \setminus (\mathcal R \cup \mathcal B)$ are the sets of {\em basic}, {\em reducible} and {\em excessive monomials}, respectively~\cite{byrod2009fast}.

\begin{definition}
\label{def:ElimTempl}
Let $\overline{\mathcal B} = \mathcal B \cap [X]_{A\cdot F}$. A Macaulay matrix $M(A \cdot F)$ with columns arranged in ordered blocks $M(A \cdot F) = \begin{bmatrix}
M_{\mathcal E} & M_{\mathcal R} & M_{\overline{\mathcal B}}\end{bmatrix}$ is called the \emph{elimination template} for $F$ \wrt $a$ if the following conditions hold true:
\begin{enumerate}
\item $\mathcal R \subset [X]_{A \cdot F}$;
\item the reduced row echelon form of $M(A\cdot F)$ is 
\begin{equation}
\label{eq:tildeM}
\widetilde M(A\cdot F) =
\begin{bmatrix}
* & 0 & * \\ 0 & I & \widetilde M_{\overline{\mathcal B}} \\
0 & 0 & 0
\end{bmatrix},
\end{equation}
where $*$ means a submatrix with arbitrary entries, $0$ is the zero matrix of a suitable size, $I$ is the identity matrix of order $\# \mathcal R$ and $\widetilde M_{\overline{\mathcal B}}$ is a matrix of size $\# \mathcal R \times \# \overline{\mathcal B}$.
\end{enumerate}
\end{definition}

\begin{theorem}
\label{thm:ElimTempl}
The elimination template is well defined, \ie, for any $s$-tuple of polynomials $F = (f_1, \ldots, f_s)$ such that ideal $\langle F \rangle$ is zero-dimensional, there exists a set of shifts $A \cdot F$ satisfying the conditions from Definition~\ref{def:ElimTempl}.
\end{theorem}

\begin{proof}
See SM Sec.~\ref{sec:proof}.
\end{proof}

In SM Sec.~\ref{sec:examples}, we provide several examples of solving polynomial systems by elimination templates.

\subsection{Action matrices from elimination templates}
\label{subsec:actmat}

We will now explain how to construct action matrices from elimination templates.

Given a finite set $A \subset \mathbb K[X]$, let $\vect{A}$ denote the vector consisting of the elements of $A$. If $A$ is a set of monomials, then the elements of $\vect{A}$ are ordered by the chosen monomial ordering on $[X]$. For a set of polynomials, the order of elements in $\vect{A}$ is irrelevant.

Let $a \in [X]$ be an action monomial and let $M(A\cdot F) = \begin{bmatrix}
M_{\mathcal E} & M_{\mathcal R} & M_{\overline{\mathcal B}}
\end{bmatrix}$ be an elimination template for $F$ \wrt $a$. Denote for short $M = M(A \cdot F)$ and $\mathcal X = [X]_{A\cdot F}$ the set of monomials corresponding to columns of $M$. Since $M$ is a Macaulay matrix, $M\,\vect{\mathcal X} = 0$ represents the expanded system of equations.

It may happen that $\overline{\mathcal B} = \mathcal B \cap \mathcal X$ is a proper subset of $\mathcal B$, see Examples~\ref{exa:2conics} and~\ref{example5} in SM. Let us construct matrix $M'_{\mathcal B}$ by adding to $M_{\overline{\mathcal B}}$ the zero columns corresponding to each $b \in \mathcal B \setminus \overline{\mathcal B}$. Then, the template $M$ is transformed into
$
M' = \begin{bmatrix}
M_{\mathcal E} & M_{\mathcal R} & M'_{\mathcal B}
\end{bmatrix},
$
which is clearly a template too. Therefore, the reduced row echelon form of $M'$ must be of the form~\eqref{eq:tildeM}. Thus we are getting
\begin{equation}
\label{eq:reducible}
\vect{\mathcal R} = -\widetilde M'_{\mathcal B}\; \vect{\mathcal B}.
\end{equation}

To provide an explicit formula for the action matrix, let the set of basic monomials $\mathcal B$ be partitioned as $\mathcal B = \mathcal B_1 \cup \mathcal B_2$, where $\mathcal B_2 = \{a\, b \cln b \in \mathcal B\} \cap \mathcal B$ and $\mathcal B_1 = \mathcal B \setminus \mathcal B_2$. Then $\vect{\mathcal B} = \begin{bmatrix}\vect{\mathcal B_1} \\ \vect{\mathcal B_2}\end{bmatrix}$ and the action matrix can be read off as follows:
\begin{equation}
\label{eq:actmat}
T_a = \begin{bmatrix} -\widetilde M'_{\mathcal B} \\ P \end{bmatrix},
\end{equation}
where $P$ is a binary matrix, \ie, a matrix consisting of $0$ and $1$, such that $\vect{\mathcal B_2} = P\,\vect{\mathcal B}$.

\section{Constructing parameterized templates}
\label{sec:constr}

Let $\B$ be a monomial basis of $\mathbb K[X]/J$. We distinguish a standard basis, which comes from a given Gr\"obner basis of $J$, and a non-standard basis, which may be represented by arbitrary monomials from $[X]$. Given a polynomial $f \in \mathbb K[X]$, let $[f] = \sum_i c_i [b_i]$, where $b_i \in \B$ and $c_i \in \mathbb K$, be the unique representation of $[f]$ in the basis $\B$. Then, the polynomial $\sum_i c_i b_i$ is called the \emph{normal form} for $f$ \wrt $\B$ and is denoted by $\overline f^{\mathcal B}$. Let us fix the action monomial $a$, and construct $\overline{a\, \vect{\B}}^\B$, \ie, the vector of normal forms for each $a\, b_i$. If $\B$ is the standard basis corresponding to a Gr\"obner basis $G$, the normal form \wrt $\B$ is found in a straightforward way as the unique remainder after dividing by polynomials from $G$, \ie, $\overline{a\, \vect{\B}}^\B = \overline{a\,\vect{\B}}^G = T_a\, \vect{\B}$, where $T_a \in \mathbb K^{d\times d}$ is the action matrix.

Now, consider an arbitrary (possibly non-standard) basis $\B$. To construct the normal form for $a\, \vect{\B}$ \wrt $\B$, we select a Gr\"obner basis $G$ of ideal $J$ and find the related (standard) basis $\widehat \B$. Then, we get
$
\overline{\vect{\B}}^G = S\, \vect{\widehat \B}.
$
As $\B$ is a basis, the square matrix $S$ is invertible. We can also compute $\overline{a\, \vect{\widehat \B}}^G = \widehat T_a \vect{\widehat \B}$, where $\widehat T_a \in \mathbb K^{d\times d}$ is the matrix of the action operator in the standard basis $\widehat \B$. Then, we have
$
\overline{a\, \vect{\B}}^\B = T_a \vect{\B},
$
where $T_a = S\,\widehat T_a S^{-1}$ is the matrix of the action operator in the basis $\B$.

Let us define
\begin{equation}
\label{eq:Vdef}
V = a\, \vect{\B} - T_a\vect{\B}
\end{equation}
and compute
\begin{multline*}
\overline V^G = S \bigl[\overline{a\cdot (S^{-1}\overline{\vect{\B}}^G)}^G - \widehat T_a S^{-1}\, \overline{\vect{\B}}^G\bigr] \\= S \bigl[\overline{a\, \vect{\widehat \B}}^G - \widehat T_a \vect{\widehat \B}\bigr] = 0.
\end{multline*}
It follows that the elements of vector $V$ belong to $J$. Therefore, there is a matrix $H \in \mathbb K[X]^{d\times s}$ such that
\begin{equation}
\label{eq:Vrepr}
V = H \vect{F}.
\end{equation}
Knowing matrix $H$ is enough for constructing an elimination template for $F$ according to Definition~\ref{def:ElimTempl}. Equation~\eqref{eq:Vrepr} can be rewritten in the form $V = \sum_k h_k f_k$, where $h_k$ is the $k$th column of $H$. Let $[X]_k$ be the support of $h_k$, $A = ([X]_1, \ldots, [X]_s)$ and $A\cdot F$ be the related set of shifts. Then, the Macaulay matrix $M(A\cdot F)$ is the elimination template for $F$, see SM Sec.~\ref{sec:proof}.

Now, we discuss how to construct the matrix $H$ so that~\eqref{eq:Vrepr} holds true. As noted in~\cite{larsson2017efficient}, such matrix is not defined uniquely, reflecting the ambiguity in constructing elimination templates. One such matrix, say $H_0$, can be found as a byproduct of the Gr\"obner basis computation.\footnote{In practice, matrix $H_0$ can be derived by using an additional option in the Gr\"obner basis computation command, \eg, \texttt{ChangeMatrix=>true} in Macaulay2~\cite{macaulay2} or \texttt{output=extended} in Maple.} On the other hand, there is a simple algorithm for computing generators of the first syzygy module of any finite set of polynomials~\cite{cox2006using}. For the $s$-tuple of polynomials $F$, the algorithm outputs a matrix $H_1 \in \mathbb K[X]^{l\times s}$ such that $H_1 \vect{F} = 0$. Let
\begin{equation}
\label{eq:matH}
H = H_0 + \Theta H_1,
\end{equation}
where $\Theta$ is a $d\times l$ matrix of parameters $\theta_{ij} \in \mathbb K$. We call the elimination template associated with the matrix $H$ the \emph{parametrized elimination template}.

We note that since the rows of matrix $H_1$ generate the syzygy module, formula~\eqref{eq:matH} would give us the complete set of solutions to Eq.~\eqref{eq:Vrepr} provided that $\theta_{ij} \in \mathbb K[X]$. However, in this paper we restrict ourselves to the much simpler case $\theta_{ij} \in \mathbb K$.

In general, the parametrized template may be very large. In the next section we propose several approaches for its reduction.

\section{Reduction of the template}
\label{sec:red}

\subsection{Adjusting parameters by a greedy search}

The $k$th column of matrix $H$, defined in~\eqref{eq:matH}, can be written as
$
h_k = Z_k c_k,
$
where $Z_k$ is the $k$th coefficient matrix whose entries are affine functions in the parameters $\theta_{ij}$ and $c_k$ is the related monomial vector. Let
$
W = \begin{bmatrix}Z_1 & \ldots & Z_s\end{bmatrix}.
$
The columns of matrix $W$ are in one-to-one correspondence with the shifts of polynomials in the expanded system and hence with the rows of the elimination template. Thus, the problem of template reduction leads to the combinatorial optimization problem of adjusting the parameters with the aim of minimizing the number of non-zero columns in $W$. Below we propose two heuristic strategies for handling this problem. We call the first strategy ``row-wise'' as it tends to remove the rows of the template. The second strategy is ``column-wise'' as it removes the columns of $W$ that correspond to excessive monomials and hence to columns of the template.

First, we notice that if a column of matrix $W$ contains an entry, which is a nonzero scalar, then this column can not be zeroed out by adjusting the parameters. Hence, we further assume that all such columns were removed from $W$. 

\noindent\textbf{Row-wise reduction:}
Let $w_k$ be the $k$th column of matrix $W$. To zero out a column of matrix $W$ means to solve linear equations $w_k = 0$. As each row of $W$ has its own set of parameters, solving $w_k = 0$ splits into solving $d$ single equations. For each $k$, we assign to $w_k$ the score $\sigma(k)$ which is the number of columns that are zeroed out by solving $w_k = 0$. Our row-wise greedy strategy implies that at each step we zero out the column with the maximal score. We proceed while $\sigma(k) > 0$ for at least one $k$.

\noindent\textbf{Column-wise reduction:}
Let $\mathcal E$ be the set of excessive monomials for the parametrized elimination template. For each $e \in \mathcal E$, we denote by $\mathcal W_e$ the subset of columns of $W$ such that the respective shifts contain $e$. For each $e \in \mathcal E$, we assign to $e$ the score $\sigma(e)$ which is the number of columns that are zeroed out by solving $w = 0$ for all $w \in \mathcal W_e$. Our column-wise greedy strategy implies that at each step we zero out the columns from $\mathcal W_e$ corresponding to the excessive monomial $e$ with the maximal score. We proceed while $\sigma(e) > 0$ for at least one $e$.

The column-wise strategy is faster as it zeroes out several columns of matrix $W$ at each step. On the other hand, the row-wise strategy outputs smaller templates for some cases. Our automatic template generator tries both strategies and outputs the smallest template.

In Fig.~\ref{fig1}, we compare our adjusting strategy with the template reduction method from~\cite{larsson2017efficient} on several minimal problems. Each box plot on the figure represents the distribution of the normalized template sizes for $100$ standard monomial bases corresponding to randomly selected monomial orderings. The action variable for each basis is also taken randomly. The problem instance is the same (fixed) for each problem. For visibility, we also show the sizes of the parametrized templates before applying any reductions.

Our reduction method produces smaller elimination templates in most cases. It can be seen that for some cases the syzygy-based reduction produces templates which are larger than the parametrized templates.

\begin{figure}
\centering
\ig{0.48}{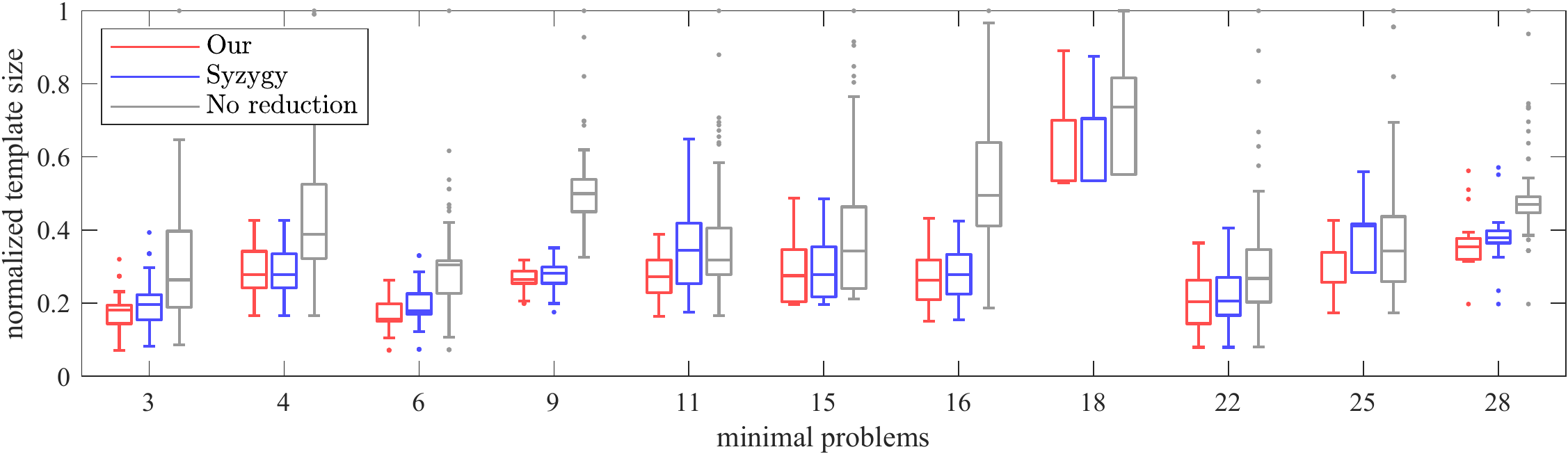}
\caption{A comparison of our adjusting strategy (Our) with the syzygy-based reduction from~\cite{larsson2017efficient} (Syzygy). We also show sizes of the initial templates before reduction (No reduction). Each box plot represents the distribution of the normalized template sizes for $100$ randomly selected standard monomial bases. The action variable for each basis is also taken randomly. The problem numbering is the same as in Tab.~\ref{tab:templates} and Tab.~\ref{tab:templates2}.}
\label{fig1}
\end{figure}

\subsection{Schur complement reduction}
\label{subsec:schur}

\begin{prop}
\label{prop:schur}
Let $M$ be an elimination template represented in the following block form
\begin{equation}
\label{eq:blockM}
M = \begin{bmatrix}A & B \\ C & D\end{bmatrix},
\end{equation}
where $A$ is a square invertible matrix and its columns correspond to some excessive monomials. Then the Schur complement of $A$, \ie, matrix $M/A = D - CA^{-1}B$, is an elimination template too.
\end{prop}

\begin{proof}
See SM Sec.~\ref{sec:schur}.
\end{proof}

In practice, Prop.~\ref{prop:schur} can be used as follows. Suppose that the set of polynomials $F = \{f_1, \ldots, f_s\}$ contains a subset, say $F^* = \{f_1, \ldots, f_k\}$, such that (i) all polynomials from $F^*$ are sparse, \ie, consist of a relatively small number of terms, and (ii) the coefficients of polynomials from $F^*$ are unchanged for all instances of the problem. Such polynomials may arise, \eg, from the normalization condition. Let an elimination template $M$ for $F$ be represented in the block form~\eqref{eq:blockM}, where the submatrix $\begin{bmatrix}A & B\end{bmatrix}$ corresponds to the shifts of polynomials from $F^*$, matrix $A$ is square and invertible, its columns correspond to some excessive monomials and its entries are the same for all instances of the problem. Then, by Prop.~\ref{prop:schur}, we can safely reduce the template by replacing $M$ with the Schur complement $M/A$.

Since the polynomials from $F^*$ are sparse, the blocks $A$ and $B$ in~\eqref{eq:blockM} are sparse too. It follows that the nonzero entries of matrix $M/A$ are simple (polynomial) functions of the entries of $M$ that can be easily precomputed offline. The Schur complement reduction allows one to significantly reduce the template for some minimal problems, see Tab.~\ref{tab:templates} and Tab.~\ref{tab:templates2} below.

\subsection{Removing dependent rows and columns}

\begin{prop}
\label{prop:nco}
Let $M''$ be an elimination template of size $s''\times n''$ whose columns arranged \wrt the partition $\mathcal E \cup \mathcal R \cup \overline{\B}$. Then there exists a template $M$ of size $s\times n$ so that $s \leq s''$, $n \leq n''$ and $n - s = \# \overline{\B}$.
\end{prop}

\begin{proof}
See SM Sec.~\ref{sec:nco}.
\end{proof}

By Prop.~\ref{prop:nco}, given an elimination template, say $M''$, we can always select a maximal subset of linearly independent rows and remove from $M''$ all the remaining (dependent) rows. The result is an elimination template $M'$. Similarly, we can always select a maximal subset of linearly independent columns corresponding to the set of excessive monomials and remove from $M'$ all the remaining columns corresponding to the excessive monomials. This is accomplished by twice applying the G--J elimination, first on matrix $M''^\top$ to remove dependent rows and then on the resulting matrix $M'$ to remove dependent columns.

\section{Experiments}
\label{sec:exper}

In this section we test our template generator on two sets of minimal problems. The first one consists of the $21$ problems covered in papers~\cite{larsson2017efficient}, \cite{larsson2018beyond} and~\cite{bhayani2020sparse}. They provide the state-of-the-art template generators denoted by Syzygy, BeyondGB and SparseR respectively. The results for the first set of problems are presented in Tab.~\ref{tab:templates}.

The second set consists of the $12$ additional problems which were not presented in~\cite{larsson2018beyond,bhayani2020sparse}. The results for the second set of problems are reported in Tab.~\ref{tab:templates2}. Below we give several remarks regarding Tab.~\ref{tab:templates} and Tab.~\ref{tab:templates2}.

\begin{table*}
\centering
\footnotesize
\begin{tabular}{rlcrlcrlc}
\hline\\[-8pt]
\# &
Problem &
$d$ & \multicolumn{2}{c}{\begin{tabular}{cc}\multicolumn{2}{c}{Our} \\\hline std~ & nstd\end{tabular}} & Syzygy~\cite{larsson2017efficient} & \multicolumn{2}{c}{\begin{tabular}{cc}\multicolumn{2}{c}{BeyondGB~\cite{larsson2018beyond}} \\\hline ~~~~std & nstd\end{tabular}} & SparseR~\cite{bhayani2020sparse} \\[4pt]
\hline\\[-6pt]
1 &
Rel. pose $F$+$\lambda$ 8pt \cite{kuang2014minimal} & 
$8$ & $11 \tm 19$ & $\bf 7 \tm 15$ & $11 \tm 19$ & $11 \tm 19$ & $\bf 7 \tm 15$ & $7 \tm 16$ \\
2 &
Rel. pose $E$+$f$ 6pt \cite{bujnak20093d} & 
$9$ & $\bf 11 \tm 20$ & $\bf 11 \tm 20$ & $21 \tm 30$ & $\bf 11 \tm 20$ & $\bf 11\tm 20$ & $\bf 11 \tm 20$ \\
3 &
Rel. pose $f$+$E$+$f$ 6pt \cite{stewenius2008minimal}, \cite{kukelova2008automatic} & 
$15$ & $12 \tm 27$ & $\bf {\color{blue}11 \tm 26}$ & $31 \tm 46$ & $31 \tm 46$ & $21 \tm 40$ & $12 \tm 30$ \\
4 &
Rel. pose $E$+$\lambda$ 6pt \cite{kuang2014minimal} & 
$26$ & $34 \tm 60$ & $\bf 14 \tm 40$ & $34 \tm 60$ & $34 \tm 60$ & $\bf 14 \tm 40$ & $\bf 14 \tm 40$ \\
5 &
Stitching $f\lambda$+$R$+$f\lambda$ 3pt \cite{naroditsky2011optimizing} & 
$18$ & $48 \tm 66$ & $\bf 18 \tm 36$ & $48 \tm 66$ & $48 \tm 66$ & $\bf 18 \tm 36$ & $\bf 18 \tm 36$ \\
6 &
Abs. pose P4P+fr \cite{bujnak2010new} & 
$16$ & $\bf 52 \tm 68$ & $\bf 52 \tm 68$ & $140 \tm 156$ & $54 \tm 70$ & $54 \tm 70$ & $\bf 52 \tm 68$ \\
7 &
Abs. pose P4P+fr (el. $f$) \cite{larsson2017making} & 
$12$ & $\bf 28 \tm 40$ & $\bf 28 \tm 40$ & $\bf 28 \tm 40$ & $\bf 28 \tm 40$ & $\bf 28 \tm 40$ & $\bf 28 \tm 40$ \\
8 &
Rel. pose $\lambda$+$E$+$\lambda$ 6pt \cite{kukelova2008automatic} & 
$52$ & $73 \tm 125$ & $\bf 39 \tm 95$ & $149 \tm 201$ & $-~~~~$ & $53 \tm 105$ & $\bf 39 \tm 95$ \\
9 &
Rel. pose $\lambda_1$+$F$+$\lambda_2$ 9pt \cite{kukelova2008automatic} & 
$24$ & $\bf {\color{blue}76 \tm 100}$ & $\bf {\color{blue}76 \tm 100}$ & $165 \tm 189$ & $87 \tm 111$ & $87 \tm 111$ & $90\tm 117$ \\
10 &
Rel. pose $E$+$f\lambda$ 7pt \cite{kuang2014minimal} & 
$19$ & $\bf {\color{blue}55 \tm 74}$ & $56 \tm 75$ & $185 \tm 204$ & $69 \tm 88$ & $69 \tm 88$ & $61 \tm 80$ \\
11 &
Rel. pose $E$+$f\lambda$ 7pt (el. $\lambda$) \cite{bhayani2020sparse} & 
$19$ & $37 \tm 56$ & $\bf 22 \tm 41$ & $52 \tm 71$ & $37 \tm 56$ & $24 \tm 43$ & $\bf 22 \tm 41$ \\
12 &
Rel. pose $E$+$f\lambda$ 7pt (el. $f\lambda$) \cite{kukelova2017clever} & 
$19$ & $\bf 51 \tm 70$ & $\bf 51 \tm 70$ & $\bf 51 \tm 70$ & $\bf 51 \tm 70$ & $\bf 51 \tm 70$ & $\bf 51 \tm 70$ \\
13 &
Rolling shutter pose \cite{saurer2015minimal} & 
$8$ & $\bf 47 \tm 55$ & $\bf 47 \tm 55$ & $\bf 47 \tm 55$ & $\bf 47 \tm 55$ & $\bf 47 \tm 55$ & $\bf 47 \tm 55$ \\
14 &
Triangulation (sat. im.) \cite{zheng2015minimal} & 
$27$ & $\bf 87 \tm 114$ & $\bf 87 \tm 114$ & $88\tm 115$ & $88 \tm 115$ & $88 \tm 115$ & $\bf 87 \tm 114$ \\
15 &
Abs. pose refractive P5P \cite{haner2015absolute} & 
$16$ & $\bf {\color{blue}57 \tm 73}$ & $\bf {\color{blue}57 \tm 73}$ & $240 \tm 256$ & $112 \tm 128$ & $199 \tm 215$ & $68 \tm 93$ \\
16 &
Abs. pose quivers \cite{kuang2013pose} & 
$20$ & $\bf {\color{blue}65 \tm 85}$ & $66 \tm 86$ & $169 \tm 189$ & $-~~~~$ & $68 \tm 88$ & $68 \tm 92$ \\
17 &
Unsynch. rel. pose \cite{albl2017two} & 
$16$ & $159 \tm 175$ & $\bf {\color{blue}139 \tm 155}$ & $159 \tm 175$ & $-~~~~$ & $299 \tm 315$ & $150 \tm 168$ \\
18 &
Optimal PnP (Hesch) \cite{hesch2011direct} & 
$27$ & $\bf 87 \tm 114$ & $\bf 87 \tm 114$ & $88 \tm 115$ & $88 \tm 115$ & $88 \tm 115$ & $\bf 87 \tm 114$ \\
19 &
Optimal PnP (Cayley) \cite{nakano2015globally} & 
$40$ & $\bf 118 \tm 158$ & $\bf 118 \tm 158$ & $\bf 118 \tm 158$ & $\bf 118 \tm 158$ & $\bf 118 \tm 158$ & $\bf 118 \tm 158$ \\
20 &
Optimal pose 2pt v2 \cite{svarm2016city} & 
$24$ & $\bf {\color{blue}139 \tm 163^*}$ & $141 \tm 165^*$ & $192 \tm 216$ & $-~~~~$ & $192 \tm 216$ & $176 \tm 200$ \\
21 &
Rel. pose $E$+angle 4pt \cite{li20134} & 
$20$ & $\bf {\color{blue}99 \tm 119^*}$ & $\bf {\color{blue}99 \tm 119^*}$ & $246 \tm 276$ & $-~~~~$ & $183 \tm 249$ & $-$ \\
\hline
\end{tabular}
\caption{A comparison of the elimination templates of our test minimal problems. We follow the notations from~\cite{larsson2018beyond,bhayani2020sparse} for the problems' names. The columns ``std'' and ``nstd'' stand for the templates generated respectively in standard way using Gr\"obner bases and in non-standard way using heuristics. The minimal templates are shown in bold, the templates which are smaller than the state-of-the-art are shown in blue bold, symbol ``$-$'' means a missing template, $d$ is the dimension of the quotient space, $*$: the template is reduced by the method of Subsect.~\ref{subsec:schur}.}
\label{tab:templates}
\end{table*}

\begin{table*}
\centering
\footnotesize
\begin{tabular}{rlcrlcc}
\hline\\[-8pt]
\# & Problem &
$d$ & \multicolumn{2}{c}{\begin{tabular}{cc}\multicolumn{2}{c}{Our} \\\hline ~~std & nstd\end{tabular}} & Original & Syzygy~\cite{larsson2017efficient} \\[4pt]
\hline\\[-6pt]
22 &
Rel. pose $\lambda$+$F$+$\lambda$ 8pt \cite{kukelova2008automatic} & 
$16$ & $\bf {\color{blue}31 \tm 47}$ & $\bf {\color{blue}31 \tm 47}$ & $32 \tm 48$ & $32 \tm 48$ \\
23 &
P3.5P+focal \cite{wu2015p3} & 
$10$ & $\bf {\color{blue}18 \tm 28}$ & $19 \tm 29$ & $20 \tm 43$ & $20\tm 30$ \\
24 &
Gen. P4P+scale \cite{ventura2014minimal} & 
$8$ & $\bf 47 \tm 55$ & $\bf 47 \tm 55$ & $48 \tm 56$ & $\bf 47 \tm 55$ \\
25 &
Rel. pose $E$+angle 4pt v2 \cite{martyushev2020efficient} & 
$20$ & $\bf 16 \tm 36^*$ & $\bf 16 \tm 36^*$ & $\bf 16 \tm 36^*$ & $36 \tm 56$ \\
26 &
Gen. rel. pose $E$+angle 5pt \cite{martyushev2020efficient} & 
$44$ & $\bf 37 \tm 81^*$ & $\bf 37 \tm 81^*$ & $\bf 37 \tm 81^*$ & $317 \tm 361$ \\
27 &
Rel. pose $E$+$fuv$+angle 7pt \cite{martyushev2018self} & 
$6$ & $46 \tm 52$ & $40 \tm 46$ & \tiny$\left\{\hspace{-5pt}\begin{array}{c} \bf 13 \tm 32\\ \bf 19 \tm 32 \\ \bf 11 \tm 20 \\ \bf 14 \tm 20\end{array}\right.$ & $66 \tm 72$ \\
28 &
Rolling shutter R6P \cite{albl2015r6p} & 
$20$ & $\bf {\color{blue}120 \tm 140}$ & $\bf {\color{blue}120 \tm 140}$ & $196 \tm 216$ & $204 \tm 224$ \\
29 &
Opt. pose w dir 4pt \cite{svarm2016city} & 
$28$ & $\bf {\color{blue}134 \tm 162^*}$ & $144 \tm 172^*$ & $280 \tm 252$ & $203 \tm 231$ \\
30 &
Opt. pose w dir 3pt \cite{svarm2016city} & 
$48$ & $397 \tm 445^*$ & $\bf {\color{blue}385 \tm 433^*}$ & $1260 \tm 1278$ & $544 \tm 592$ \\
31 &
$L_2$ 3-view triang. (relaxed) \cite{kukelova2013fast} & 
$31$ & $\bf {\color{blue}217 \tm 248}$ & $281 \tm 312$ & $274 \tm 305$ & $231 \tm 262$\\
32 &
Refractive P6P+focal \cite{haner2015absolute} & 
$36$ & $\bf {\color{blue}126 \tm 162}$ & $178 \tm 214$ & $648 \tm 917$ & $636 \tm 654$\\
33 &
Rel. pose $f\lambda$+$E$+$f\lambda$ 7pt \cite{jiang2014minimal} & 
$68$ & $\bf {\color{blue}209 \tm 277}$ & $255 \tm 323$ & $886 \tm 1011$ & $581 \tm 659$ \\
34 &
Gen. rel. pose + scale 7pt~\cite{kneip2016generalized} & 
$140$ & $\bf 144 \tm 284$ & $\bf 144 \tm 284$ & $-$ & $\bf 144 \tm 284$ \\
\hline
\end{tabular}
\caption{A comparison of the elimination templates of our test minimal problems. The columns ``std'' and ``nstd'' stand for the templates generated from the standard and non-standard quotient ring bases respectively. We follow the notations from~\cite{larsson2017efficient} for the problems' names. The minimal templates are shown in bold, the templates which are smaller than the state-of-the-art are shown in blue bold, symbol ``$-$'' means a missing template, $d$ is the dimension of the quotient space, $*$: the template is reduced by the method of Subsect.~\ref{subsec:schur}.}
\label{tab:templates2}
\end{table*}

\noindent{\bf 1.} The column ``std'' consists of the smallest templates generated in a standard way using Gr\"obner bases either from the entire Gr\"obner fan of the ideal\footnote{We used the software package Gfan~\cite{gfan} to compute Gr\"obner fans.} or from 1,000 randomly selected bases in case the Gr\"obner fan computation cannot be done in a reasonable time. The column ``nstd'' consists of the smallest templates generated from the 500 quotient space bases found by using the random sampling strategy from~\cite{larsson2018beyond}.

\noindent{\bf 2.} The templates marked with $*$ were reduced by the method of Subsect.~\ref{subsec:schur}. The related minimal problem formulations contain a simple sparse polynomial with (almost) all constant coefficients. For example, the formulations of problems \#25 and \#26 contain the quaternion normalization constraint $x^2 + y^2 + z^2 + \sigma^2 = 1$, where $x$, $y$, $z$ are unknowns and the value of $\sigma$ is known. All the multiples of this equation can be safely eliminated from the template by constructing the Schur complement of the respective block.

\noindent{\bf 3.} The polynomial equations for problem \#3 are constructed from the null-space of a $6 \times 9$ matrix. We used the sparse basis of the null-space constructed by the G--J elimination as it leads to a smaller elimination template compared to the dense basis constructed by the SVD.

\noindent{\bf 4.} The $39\times 95$ elimination template for problem \#8 was found \wrt the reciprocal of the action variable $\lambda$ representing the radial distortion, \ie, vector $V$ from~\eqref{eq:Vdef} was defined as $V = \lambda^{-1} \vect{\B} - T_{\lambda^{-1}}\vect{\B}$, where the non-standard basis $\B$ consists of monomials that are all divisible by $\lambda$. In terms of paper~\cite{byrod2009fast}, the set $\B$ constitutes the redundant solving basis as it consists of $56$ monomials whereas the number of solutions to problem \#8 is $52$. The four spurious solutions can be filtered out by removing solutions with the worst values of normalized residuals.

\noindent{\bf 5.} The initial formulation of problem \#15 consists of $5$ degree-$3$ polynomials in $5$ variables: $3$ rotation parameters and $2$ camera center coordinates. As suggested in~\cite{bhayani2020sparse}, we first simplified these polynomials using a G--J elimination on the related Macaulay matrix. After that, $2$ of $5$ polynomials depend only on the rotation variables. The remaining $3$ polynomials depend linearly on the camera center variables. We used $2$ of these polynomials to solve for the camera center and then substitute the solution into the third polynomial resulting in one additional polynomial of degree $4$ in $3$ rotation variables only. Hence our formulation of the problem consists of $3$ polynomials in $3$ variables: $1$ polynomial of degree $4$ and $2$ polynomials of degree $2$. It is important to note that (i) the coefficients of the degree-$4$ polynomial are linearly (and quite easily) expressed in terms of the coefficients of the $3$ initial polynomials and (ii) this elimination process does not introduce any spurious roots. We also note that the problem has the following $2$-fold symmetry: if $x$, $y$, $z$ are the rotation parameters for the Cayley-transform representation, then replacing $x \to y/z$, $y \to -x/z$ and $z \to -1/z$ leaves the polynomial system unchanged. It follows that the problem has no more than $8$ ``essentially distinct'' solutions and hence the template for this problem could be further reduced.

\noindent{\bf 6.} Problem \#27 was originally solved by applying a cascade of four G--J eliminations to the manually saturated polynomial ideal. We marked the original solver in bold as it is faster than the new single elimination solver (0.4~ms against 0.6~ms).

\noindent{\bf 7.} The initial formulation of problem \#32 consists of $6$ degree-$4$ polynomials in $6$ variables: $3$ rotation parameters, $2$ camera center coordinates and the focal length. Similarly as we did for problem \#15, we first simplified the equations using a G--J elimination on the related Macaulay matrix and then we eliminated the camera center coordinates. This results in $4$ equations in $4$ unknowns: $1$ polynomial of degree $5$, $2$ of degree $3$ and $1$ of degree $2$. As in the case of problem \#15, eliminating variables does not introduce any spurious solutions. We also note that the problem has a $4$-fold symmetry meaning that the number of its ``essentially distinct'' roots is not more than $9$. It follows that the template for this problem could be further reduced.

\noindent{\bf 8.} The implementation of the new AG, as well as the Matlab solvers for all the minimal problems from Tab.~\ref{tab:templates} and Tab.~\ref{tab:templates2}, are available at \href{http://github.com/martyushev/EliminationTemplates}{http://github.com/martyushev/EliminationTemplates}. In SM Sec.~\ref{sec:results}, we test the speed and numerical stability of our solvers.

\subsection{Relative pose with unknown focal length and radial distortion}

The problem of relative pose estimation of a camera with unknown but fixed focal length and radial distortion can be minimally solved from seven point correspondences in two views. It was first considered in paper~\cite{jiang2014minimal}, where it was formulated as a system of $12$ polynomial equations: $1$ equation of degree $2$, $1$ of degree $3$, $2$ of degree $5$, $3$ of degree $6$ and $5$ of degree $7$. The $5$ unknowns are: the radial distortion parameter $\lambda$ for the division model from~\cite{fitzgibbon2001simultaneous}, the reciprocal square of the focal length $f^{-2}$ and the thee entries $F_{32}$, $F_{13}$, $F_{23}$ of the fundamental matrix $F$. The related polynomial ideal has degree $68$ meaning that the problem generally has $68$ solutions.

We started from the same formulation of the problem as in the original paper~\cite{jiang2014minimal}. We did not manage to construct the Gr\"obner fan for the related polynomial ideal in a reasonable amount of time (about 24 hours). Instead, we randomly sampled 1,000 weighted monomial orderings so that the respective reduced Gr\"obner bases are all distinct. We avoided weight vectors where a one entry is much smaller than the others, since the monomial orderings for such weights usually lead to notably larger templates. We also constructed 500 heuristic bases of the quotient ring by using the random sampling strategy from~\cite{larsson2018beyond}. Then, we used our automatic generator to construct elimination templates for all the bases (both standard and non-standard) and for all the action variables. The smallest template we found this way has size $209 \times 277$. It corresponds to the standard basis for the weighed monomial ordering with $f^{-2} > F_{32} > F_{13} > F_{23} > \lambda$ and the weight vector $w = \begin{bmatrix}135 & 81 & 98 & 107 & 68\end{bmatrix}^\top$. The action variable is $\lambda$.

\begin{figure}
\centering
\begin{tabular}{cc}
\ig{0.21}{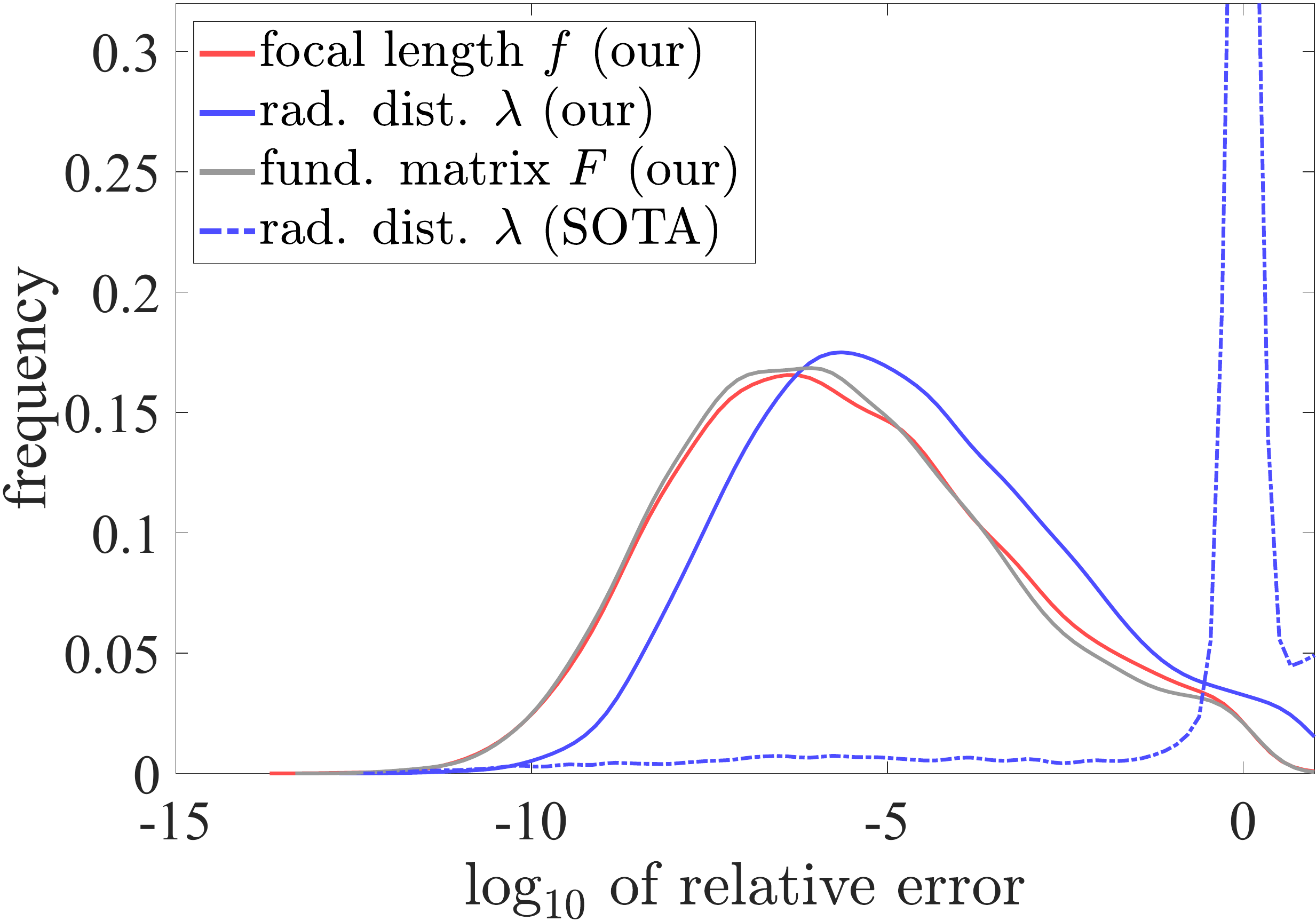} &
\ig{0.21}{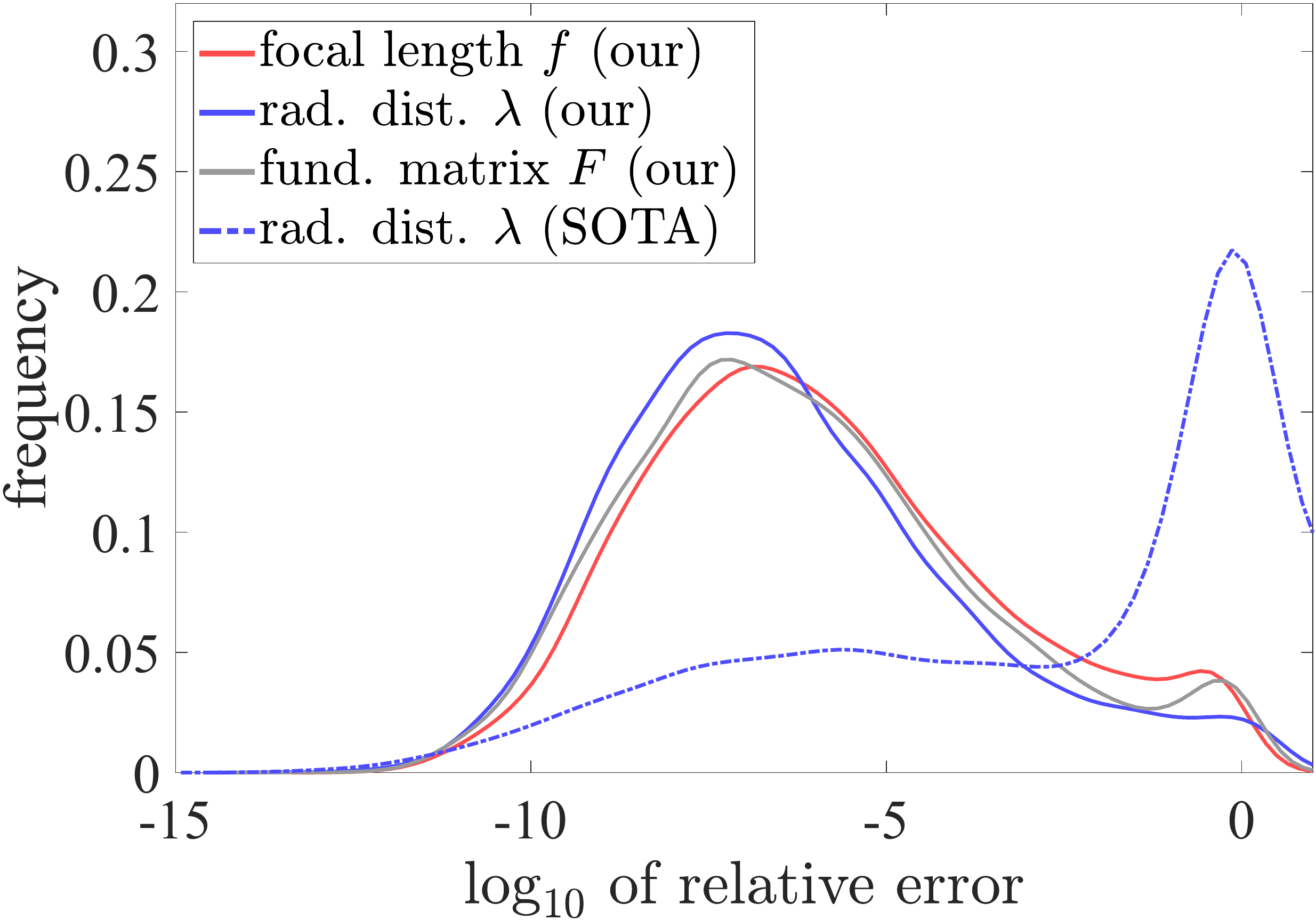} \\
(a) $f = 0.5$, $\lambda = -0.1$ & (b) $f = 2$, $\lambda = -0.2$ \\[5pt]
\ig{0.21}{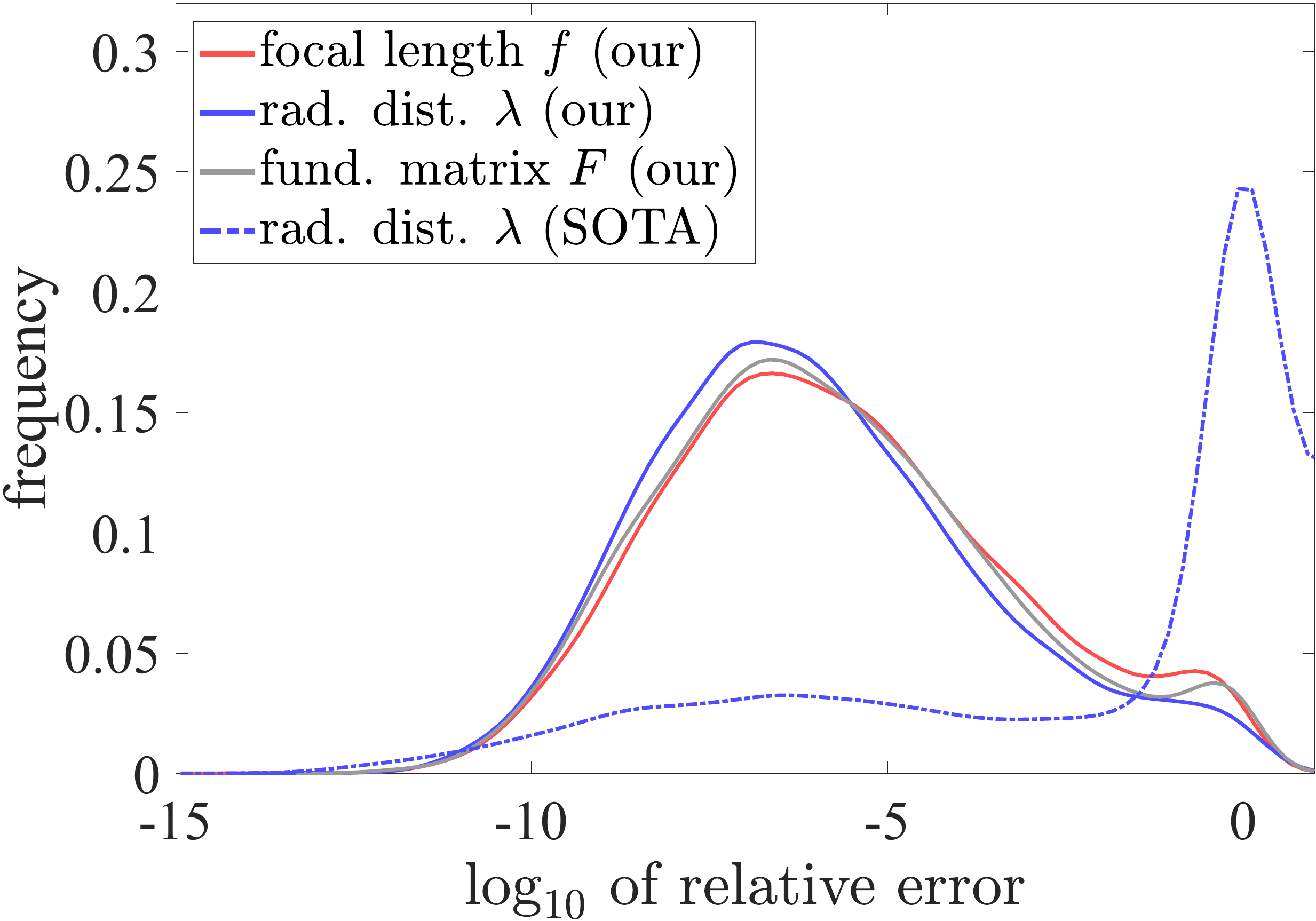} &
\ig{0.21}{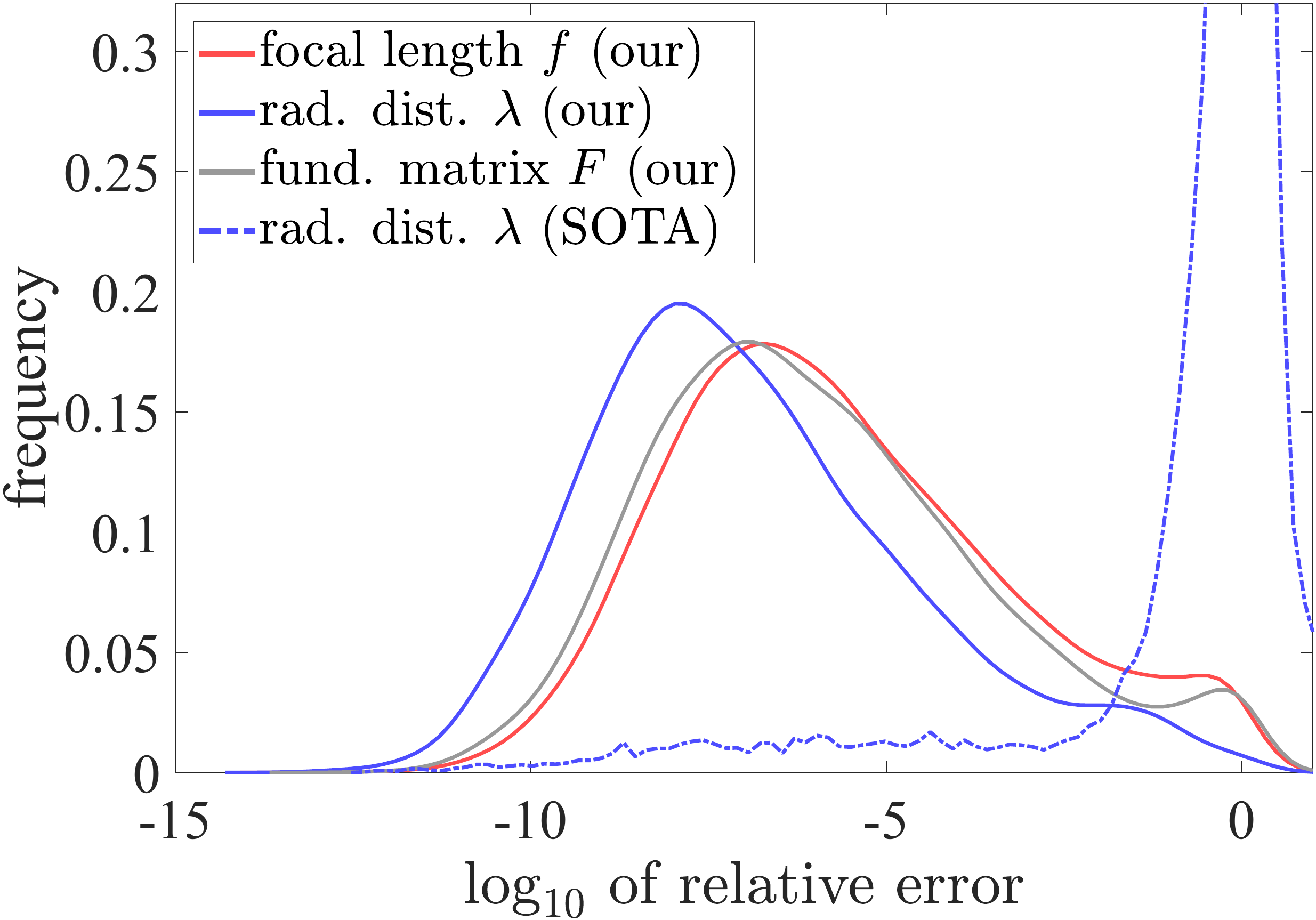} \\
(c) $f = 1$, $\lambda = -0.5$ & (d) $f = 3$, $\lambda = -0.9$
\end{tabular}
\caption{The distribution of relative errors for problem \#33 from Tab.~\ref{tab:templates2} on $10^4$ trials for different values of focal length $f$ and radial distortion $\lambda$. For comparison, we also added the relative error distribution for $\lambda$ obtained by the state-of-the-art (SOTA) solver from~\cite{oskarsson2021fast}.}
\label{fig3}
\end{figure}

The solver from paper~\cite{jiang2014minimal}, based on the elimination template of size $886\times 1011$, is not publically available. However, the results reported in the paper assume that the solver from~\cite{jiang2014minimal} is much slower than our one (400~ms against 8.5~ms), while the both solvers demonstrate comparable numerical accuracy. The solver based on the $581\times 659$ template generated by the AG from~\cite{larsson2017efficient} is almost twice slower (about 16~ms) than our solver. Moreover, the solver~\cite{larsson2017efficient} it is unstable and requires additional stability improving techniques, \eg, column pivoting~\cite{byrod2009fast}. Hence we compared our solver with the only publicly available state-of-the-art solver from the recent paper~\cite{oskarsson2021fast}.

We modeled a scene consisting of seven points viewed by two cameras with unknown but shared focal length $f$ and radial distortion parameter $\lambda$. The distance between the first camera center and the scene is $1$, the scene dimensions (w $\tm$ h $\tm$ d) are 1 $\tm$ 1 $\tm$ 0.5 and the baseline length is $0.3$.

We tested the numerical accuracy of our solver by constructing the distributions of relative errors for the focal length $f$, radial distortion parameter $\lambda$ and fundamental matrix $F$ on noise-free image data. We only kept the roots satisfying the following ``feasibility'' conditions: (i) $f^{-2}$ is real; (ii) $f^{-2} > 0$; (iii) $-1 \leq \lambda \leq 1$. The results for different values of $f$ and $\lambda$ are shown in Fig.~\ref{fig3}.

Our solver failed (\ie, found no feasible solutions) in approximately $2\%$ of trials. The average runtime for the solver from~\cite{oskarsson2021fast} was $2.9$~ms which is almost $3$ times less than the execution time for our solver ($8.5$~ms). However, we note that the main parts of the solver from~\cite{oskarsson2021fast} are written in C++, whereas our algorithm is fully implemented in Matlab. This provides a room for further speed up of our solver.

\section{Conclusion}
\label{sec:concl}

We developed a new method for constructing small and stable elimination templates for efficient polynomial system solving of minimal problems. We presented the state-of-the-art templates for many minimal problems with substantial improvement for harder problems.

\bibliographystyle{amsplain}
\bibliography{biblio}

\renewcommand{\textfraction}{0.05}
\renewcommand{\topfraction}{0.8}
\renewcommand{\bottomfraction}{0.8}
\renewcommand{\textfraction}{0.1}
\renewcommand{\floatpagefraction}{0.8}

\twocolumn[
\begin{center}
{\Large {\bf \TitleText}\\[1ex]
Supplementary Material\\[3ex]}
\begin{tabular}{ccc}
\begin{minipage}{0.3\linewidth}
\centering \large \NameEM
\end{minipage}&
\begin{minipage}{0.3\linewidth}
\centering \large \NameJV
\end{minipage}&
\begin{minipage}{0.3\linewidth}
\centering \large \NameTP
\end{minipage}
\end{tabular}
\vspace{3ex}
\end{center}
]

\setcounter{theorem}{0}
\setcounter{prop}{0}

\noindent Here we give additional details for the main paper

\noindent{\em
E.~Martyushev, J.~Vrablikova, T.~Pajdla. \TitleText. CVPR~2022.}\href{http://github.com/martyushev/EliminationTemplates}\\{\href{http://github.com/martyushev/EliminationTemplates}{http://github.com/martyushev/EliminationTemplates}}

We present some basic notions from algebraic geometry, proofs, examples of constructing elimination templates, numerical details, and additional experiments demonstrating the numerical stability.

\section{Monomial orderings}
\label{sec:basic}

A \emph{monomial ordering} $>$ on $[X]$ is a total ordering satisfying (i) $p > 0$ for all $p \in [X]$ and (ii) if $p > q$, then $p\,s > q\,s$ for all $p, q, s \in [X]$. We are particularly interested in the following two orderings:
\begin{enumerate}
\item \emph{graded reverse lex ordering} (grevlex) compares monomials first by their total degree, and breaks ties by smallest degree in $x_k$, $x_{k - 1}$, etc.

\item \emph{weighted-degree ordering} \wrt a weight vector $w \in \mathbb R^k_+$ compares monomials first by their weighted degree (the dot product of $w$ with the exponent vector $\begin{bmatrix}\alpha_1 & \ldots & \alpha_k\end{bmatrix}^\top$), and breaks ties by reverse lexicographic order as in grevlex.
\end{enumerate}

\section{Proof of Theorem~\ref{thm:ElimTempl}}
\label{sec:proof}

The following theorem is not a new result of this work. It is ``folclore'' in algebraic geometry and has been used, \eg, in~\cite{byrod2007improving,larsson2017efficient,larsson2018beyond}, but we could not find it formulated clearly and concisely in the literature. Thus, we present it here for the sake of completeness.

\begin{theorem}
The elimination template is well defined, \ie, for any $s$-tuple of polynomials $F = (f_1, \ldots, f_s)$ such that ideal $\langle F \rangle$ is zero-dimensional, there exists a set of shifts $A \cdot F$ satisfying both conditions from Definition~\ref{def:ElimTempl}.
\end{theorem}

\begin{proof}
Let us first show that there is a set of polynomials $A$ such that all reducible monomials $\mathcal R$ appear in the support of $A \cdot F$. Let $G = \{g_1, \ldots, g_l\}$ be a reduced Gr\"obner basis of the zero-dimensional ideal $J = \langle F \rangle$ and let $\LM{G}$ be the set of its leading monomials, with $\langle \LM{G}\rangle\ = \langle \LM{J}\rangle$~\cite[p.~78 Definition 5]{Cox-IVA-2015}. Let $\B$ be the set of standard monomials representing a linear basis $B$ of the quotient ring $\mathbb K[X]/J$ for $G$. Then, for every reducible monomial $r \in \mathcal R = \{a\,b \cln b\in \B\} \setminus \B$ for $B$ and any $a \in \mathbb K[X]$, there holds true $r \in \langle \LM{G} \rangle$, because $r$ is a multiple of some $b \in \B$ but $r$ is not an element of $\B$. The set $\B$ of the standard monomials is finite~\cite[p.~251 Theorem 6]{Cox-IVA-2015}. Thus, $\mathcal R$ is finite too, and we can write $\mathcal R = \{r_1, \ldots, r_n\}$. For every $k \in \N$, $k \leq n$, we can write $r_k = m_{k1}g_1 + \ldots + m_{kl}g_l + p_k$, where $m_{ki} \in \mathbb K[X]$, $\deg(r_k) \geq \deg(m_{ki}g_i)$ for every $i \in \{1, \ldots, l\}$, and $p_k \in \mathbb K[X]$ is the polynomial satisfying $\overline{r_k}^{G} = p_k = \overline{p_k}^{G}$~\cite[p.~64 Theorem 3]{Cox-IVA-2015}. Moreover, for every $i \in \{1, \ldots, l\}$, there exists $q_{ij} \in \mathbb K[X]$ such that $g_i \in G$ can be written as $g_i = \sum_j q_{ij}f_j$. We can write $r_k = \sum_i \sum_j m_{ki} q_{ij}f_j + p_k$. Let $\{m_{ki}q_{ij}\} \in A_j$ for all $i \in \{1, \ldots, l\}$, $j \in \{1, \ldots, s\}$, $k \in \{1, \ldots, n\}$. Then, the Macaulay matrix $M(A \cdot F)$ has a non-zero element in every column corresponding to a monomial from $\mathcal R$.

Let us next show that the eliminated matrix $\widetilde M(A \cdot F)$ contains a pivot in every column corresponding to a monomial from $\mathcal R$. Denote by $\widetilde P$ the set of polynomials $\widetilde P = \widetilde M(A \cdot F) \cdot [X]_{A\cdot F}$. For a set $S$ denote $\hull(S)$ the linear space over $\mathbb K$ spanned by the elements of $S$, \ie $\hull(S) = \{\sum_j c_js_j \cln s_j \in S, c_j \in \mathbb K\}$.

Suppose $\{m_{ki}q_{ij}\} \in A_j$ for all $i \in \{1, \ldots, l\}$, $j \in \{1, \ldots, s\}$, $k \in \{1, \ldots, n\}$. For every $k$ we have $r_k = \sum_i\sum_j m_{ki} q_{ij}f_j + p_k$, hence $r_k - p_k \in \hull(A \cdot F) = \hull(\widetilde P)$. The polynomial $p_k$ is a linear combination of elements from $\overline{\B}$, thus the polynomial $r_k - p_k$ contains only one reducible monomial $r_k$ and no excessive monomials. Since $\widetilde M(A \cdot F)$ is in the reduced row echelon form, there is a row in $\widetilde M(A \cdot F)$ corresponding to the polynomial $r_k - p_k \in \hull(\widetilde P)$ with zero coefficients at all excessive monomials and all reducible monomials except for $r_k$. Hence, there is a pivot in every column of $\widetilde M(A \cdot F)$ corresponding to a reducible monomial. It follows that $\widetilde M(A \cdot F)$ must have the form~\eqref{eq:tildeM} meaning that $M(A \cdot F)$ is the elimination template.
\end{proof}

\section{Examples}
\label{sec:examples}

In this section, we provide several examples of constructing elimination templates and using them to compute solutions of polynomial systems.

\begin{example}
\label{exa:2conics}
In the first example we demonstrate th construction of the elimination template for a set of two polynomials in $\mathbb Q[x,y]$. We derive the action matrix and show how to extract the solution of the system from the action matrix. 

Let $J = \langle F \rangle$, where $F = \{f_1, f_2\} = \{x^2 + y^2 - 1, x^2 + xy + y^2 - 1\} \subset \mathbb Q[x, y]$. The Gr\"obner basis of $J$ \wrt grevlex with $x > y$ is $G = \{xy, x^2 + y^2 - 1, y^3 - y\}$. The standard basis of $\mathbb Q[x, y]/J$ is $\B = \{y^2, y, x, 1\}$. If $x$ is the action variable, then the action matrix is
\[
T_x = \begin{bmatrix}
0 & 0 & 0 & 0 \\
0 & 0 & 0 & 0 \\
-1 & 0 & 0 & 1 \\
0 & 0 & 1 & 0
\end{bmatrix}.
\]
Let us construct vector $V$ define in Eq.~\eqref{eq:Vdef}:
\[
V = x\vect{\B} - T_x \vect{\B} = 
\begin{bmatrix}
xy^2 \\ x^2 + y^2 - 1 \\ xy \\ 0
\end{bmatrix}.
\]
Since $V \subset J$, see Sec.~\ref{sec:constr}, there exists matrix $H_0$ such that $V = H_0 \vect{F}$. By tracing the computation of the Gr\"obner basis $G$ we found
\[
H_0 = \begin{bmatrix}
-y & y \\
1 & 0 \\
-1 & 1 \\
0 & 0
\end{bmatrix}.
\]
It follows that it is enough to take the set of shifts $A\cdot F = \{yf_2, yf_1, f_2, f_1\}$. We divide the set of monomials $[X]_{A\cdot F}$ into the subsets $\overline{\B} = \B \cap [X]_{A\cdot F} = \{y^2, y, 1\}$, $\mathcal R = \{xy^2, xy, x^2\}$ and $\mathcal E = \{x^2y, y^3\}$. This yields the elimination template
\begin{multline*}
M(A\cdot F) = \begin{bmatrix} M_{\mathcal E} & M_{\mathcal R} & M_{\overline{\B}}\end{bmatrix} \\=
\kbordermatrix{& x^2y & y^3 & xy^2 & xy & x^2 & y^2 & y & 1 \\
yf_2 & 1 & 1 & 1 & 0 & 0 & 0 & -1 & 0 \\ 
yf_1 & 1 & 1 & 0 & 0 & 0 & 0 & -1 & 0 \\ 
f_2 & 0 & 0 & 0 & 1 & 1 & 1 & 0 & -1 \\
f_1 & 0 & 0 & 0 & 0 & 1 & 1 & 0 & -1}.
\end{multline*}
The reduced row echelon form of $M(A\cdot F)$ has the form
\begin{multline*}
\widetilde M(A\cdot F) =
\left[
\begin{array}{c|c|c}
\widetilde M_{\mathcal E} & 0 & *\\\hline
0 & I & \widetilde M_{\B}
\end{array}\right]
\\=
\left[
\begin{array}{cc|ccc|cccc}
1 & 1 & 0 & 0 & 0 & 0 & -1 & 0 & 0 \\\hline
0 & 0 & 1 & 0 & 0 & 0 & 0 & 0 & 0 \\
0 & 0 & 0 & 1 & 0 & 0 & 0 & 0 & 0 \\
0 & 0 & 0 & 0 & 1 & 1 & 0 & 0 & -1
\end{array}\right].
\end{multline*}
Then the action matrix $T_x$ is read off as $\begin{bmatrix}-\widetilde M_{\B} \\ P\end{bmatrix}$, where $P = \begin{bmatrix}0 & 0 & 1 & 0\end{bmatrix}$ satisfies $x = P \vect{\B}$.

The eigenvalues of $T_x$, \ie $\{0, \pm 1\}$, where the geometric multiplicity of the eigenvalue $\lambda = 0$ equals 2, \ie the eigen-space associated with $\lambda = 0$ is 2-dimensional. Hence, the $x$-components of the roots are $0, 0, -1, 1$. The $y$-components can be derived from the eigenvectors resulting in the following roots: $\{(0, -1), (0, 1), (-1, 0), (1, 0)\}$.
\end{example}

\begin{example}
\label{example2}
In this example we demonstrate the usage of non-standard bases of the quotient space. Having an action matrix related to the standard basis $\widehat B$, we can construct the action matrix related to a non-standard basis $\B$ by a change-of-basis matrix. Another option is to construct a set of shifts and divide its monomials so that the basis monomials are the ones from the non-standard basis $\B$. The action matrix derived from the resulting elimination template is the action matrix related to $\B$.

Let $J = \langle F \rangle$, where $F = \{f_1, f_2\} = \{x^3 + y^2 - 1, x - y - 1\} \subset \mathbb Q[x, y]$. The Gr\"obner basis of $J$ \wrt grevlex with $x > y$ is $G = \{x - y - 1, y^3 + 4y^2 + 3y\}$. The standard basis of $\mathbb Q[x, y]/J$ is $\widehat \B = \{1, y, y^2\}$. If $x$ is the action variable, then the related action matrix is
\[
\widehat T_x = \begin{bmatrix}
1 & 1 & 0 \\
0 & 1 & 1 \\
0 & -3 & -3
\end{bmatrix}.
\]

Now let us consider the non-standard basis $\B = \{x^2, y, 1\}$. The respective change-of-basis matrix $S$, \ie a matrix satisfying $\overline{\vect{\B}}^G = S\, \vect{\widehat \B}$, has the form
\[
S = \begin{bmatrix}
1 & 2 & 1 \\
0 & 1 & 0 \\
1 & 0 & 0
\end{bmatrix}.
\]
Then the matrix of the action operator in the basis $\B$ is
\[
T_x = S\,\widehat T_x S^{-1} = \begin{bmatrix}
-1 & 2 & 2 \\
1 & -1 & -1 \\
0 & 1 & 1
\end{bmatrix}.
\]

The vector $V$ define in Eq.~\eqref{eq:Vdef} has the form
\[
V = x\vect{\B} - T_x \vect{\B} = 
\begin{bmatrix}
x^3 + x^2 - 2y - 2 \\ -x^2 + xy + y + 1 \\ x - y - 1
\end{bmatrix}.
\]
Since $V \subset J$, see Sec.~\ref{sec:constr}, there exists matrix $H_0$ such that $V = H_0 \vect{F}$. By tracing the computation of the Gr\"obner basis $G$ we found
\[
H_0 = \begin{bmatrix}
1 & x + y + 1 \\
0 & -x - 1 \\
0 & 1
\end{bmatrix}.
\]
It follows that it is enough to take the set of shifts $A\cdot F = \{xf_2, yf_2, f_2, f_1\}$. We divide the set of monomials $[X]_{A\cdot F}$ into the subsets $\overline{\B} = \B = \{x^2, y, 1\}$, $\mathcal R = \{x^3, xy, x\}$ and $\mathcal E = \{y^2\}$. This yields the elimination template
\begin{multline*}
M(A\cdot F) = \begin{bmatrix} M_{\mathcal E} & M_{\mathcal R} & M_{\B}\end{bmatrix} \\=
\kbordermatrix{& y^2 & x^3 & xy & x & x^2 & y & 1 \\
xf_2 & 0 & 0 & -1 & -1 & 1 & 0 & 0 \\ 
yf_2 & -1 & 0 & 1 & 0 & 0 & -1 & 0 \\ 
f_2 & 0 & 0 & 0 & 1 & 0 & -1 & -1 \\
f_1 & 1 & 1 & 0 & 0 & 0 & 0 & -1}.
\end{multline*}
The reduced row echelon form of $M(A\cdot F)$ is as follows
\begin{multline*}
\widetilde M(A\cdot F) =
\left[
\begin{array}{c|c|c}
\widetilde M_{\mathcal E} & 0 & *\\\hline
0 & I & \widetilde M_{\B}
\end{array}\right]
\\=
\left[
\begin{array}{c|ccc|ccc}
1 & 0 & 0 & 0 & -1 & 2 & 1 \\\hline
0 & 1 & 0 & 0 & 1 & -2 & -2 \\
0 & 0 & 1 & 0 & -1 & 1 & 1 \\
0 & 0 & 0 & 1 & 0 & -1 & -1
\end{array}\right].
\end{multline*}
Then the action matrix $T_x$ is exactly $-\widetilde M_{\B}$. The eigenvalues of $T_x$, \ie $\{-2, 0, 1\}$, give us the $x$-components of the roots. The $y$-components can be derived from the eigenvectors, which are
\[
\begin{bmatrix}
4 \\ -3 \\ 1
\end{bmatrix},
\quad
\begin{bmatrix}
0 \\ -1 \\ 1
\end{bmatrix},
\quad
\begin{bmatrix}
1 \\ 0 \\ 1
\end{bmatrix}.
\]
Hence, we get the roots $\{(-2, -3), (0, -1), (1, 0)\}$.
\end{example}

\begin{example}
In this example we consider a set of polynomials with the same structure as in Example~\ref{example2} but with different coefficients. We use the same elimination template as in the previous example and just plug in the corresponding coefficient. Then we derive the action matrix. We again use the non-standard basis $\B = \{x^2,y,1\}$. 

Let $F = \{f_1, f_2\} = \{x^3 - \sqrt2 y^2 - 3, x - \sqrt3 y + 4\} \subset \mathbb R[x, y]$.

We can use the same set of shifts $A\cdot F = \{xf_2, yf_2, f_2, f_1\}$ to construct the elimination template
\[
M(A\cdot F) =
\kbordermatrix{& y^2 & x^3 & xy & x & x^2 & y & 1 \\
& 0 & 0 & -\sqrt3 & 4 & 1 & 0 & 0 \\ 
& -\sqrt3 & 0 & 1 & 0 & 0 & 4 & 0 \\ 
& 0 & 0 & 0 & 1 & 0 & -\sqrt3 & -1 \\
& \sqrt2 & 1 & 0 & 0 & 0 & 0 & -3}.
\]
Finding the reduced row echelon form of $M(A\cdot F)$ results in the following action matrix:
\[
T_x = \begin{bmatrix}
\frac{\sqrt2}{3} & \frac{8\sqrt6}{3} & -\frac{16\sqrt2}{3} - 3 \\
-\frac{\sqrt3}{3} & -4 & \frac{16\sqrt3}{3} \\
0 & -\sqrt3 & 4
\end{bmatrix}.
\]
Finally, from the eigenvectors of $T_x$ we derive the roots of $F = 0$: $\{(2.955, 4.015), (-1.242 + 1.423i, 1.592 + 0.822i), (-1.242 - 1.423i, 1.592 - 0.822i)\}$.
\end{example}

\begin{example}
\label{example5}
In this example we consider a non-radical ideal. We use the standard basis $\B$. It is enough to use a set of shifts such that not all the monomials from $\B$ are included. One can then add zero columns to the elimination template. To get the full action matrix, we need to add the permutation matrix from Eq.(\ref{eq:actmat}) to the matrix obtained from the elimination template. Then we derive roots of the system from the action matrix and its eigenvectors. 

Let $J = \langle F \rangle$, where $F = \{f_1, f_2\} = \{x^2 - y^2, y^2 - x\} \subset \mathbb Q[x, y]$. The reduced Gr\"obner basis \wrt grevlex with $x > y$ is $G = \{y^2 - x, x^2 - x\}$ and the standard basis of $\mathbb Q[x, y]/J$ is $\B = \{xy, x, y, 1\}$.

If $y$ is the action variable, then the action matrix is
\[
T_y = \begin{bmatrix}
0 & 0 & 1 & 0 \\
0 & 0 & 1 & 0 \\
1 & 0 & 0 & 0 \\
0 & 1 & 0 & 0
\end{bmatrix}.
\]
Let us construct vector $V$ define in Eq.~\eqref{eq:Vdef}:
\[
V = y\vect{\B} - T_y \vect{\B} = 
\begin{bmatrix}
xy^2 - x \\ y^2 - x \\ 0 \\ 0
\end{bmatrix}.
\]
Since $V \subset J$, see Sec.~\ref{sec:constr}, there exists matrix $H_0$ such that $V = H_0 \vect{F}$. By tracing the computation of the Gr\"obner basis $G$ we found
\[
H_0 = \begin{bmatrix}
1 & x + 1 \\
0 & 1 \\
0 & 0 \\
0 & 0
\end{bmatrix}.
\]
It follows that it is enough to take the set of shifts $A\cdot F = \{xf_2, f_2, f_1\}$. We divide the set of monomials $[X]_{A\cdot F}$ into the subsets $\overline{\B} = \B \cap [X]_{A\cdot F} = \{x\}$, $\mathcal R = \{xy^2, y^2\}$ and $\mathcal E = \{x^2\}$. This yields the elimination template
\begin{multline*}
M(A\cdot F) = \begin{bmatrix} M_{\mathcal E} & M_{\mathcal R} & M_{\overline{\B}}\end{bmatrix} \\=
\kbordermatrix{& x^2 & xy^2 & y^2 & x \\
xf_2 & -1 & 1 & 0 & 0 \\
f_2 & 0 & 0 & 1 & -1 \\
f_1 & 1 & 0 & -1 & 0}.
\end{multline*}
The reduced row echelon form of $M(A\cdot F)$ is the matrix
\[
\widetilde M(A\cdot F) = \left[\begin{array}{c|c|c}
\widetilde M_{\mathcal E} & 0 & *\\\hline
0 & I & \widetilde M_{\overline{\B}}
\end{array} \right] = \left[\begin{array}{c|cc|c}
1 & 0 & 0 & -1 \\\hline
0 & 1 & 0 & -1 \\
0 & 0 & 1 & -1
\end{array}\right].
\]
We can add zero columns corresponding to the basic monomials from $\B \setminus \overline{\B} = \{xy, y, 1\}$ to the matrix $\widetilde M(A\cdot F)$. This yields
\[
\left[\begin{array}{c|c|c}
\widetilde M_{\mathcal E} & 0 & *\\\hline
0 & I & \widetilde M_{\B}
\end{array} \right] =
\left[\begin{array}{c|cc|cccc}
1 & 0 & 0 & 0 & 0 & -1 & 0 \\\hline
0 & 1 & 0 & 0 & 0 & -1 & 0 \\
0 & 0 & 1 & 0 & 0 & -1 & 0
\end{array} \right].
\]
Then the action matrix $T_y$ is read off as $\begin{bmatrix}
-\widetilde M_{\B} \\ P
\end{bmatrix}$, where $P = \begin{bmatrix}
1 & 0 & 0 & 0 \\
0 & 1 & 0 & 0
\end{bmatrix}$ satisfies $\begin{bmatrix}xy \\ x\end{bmatrix} = P \vect{\B}$.

The eigenvalues of $T_y$ are $\{0, \pm 1\}$. The geometric multiplicity of the eigenvalue $\lambda = 0$ equals $1$, whereas its algebraic multiplicity is $2$ implying that $T_y$ is non-diagonalizable. The $y$-components of the roots are $0, -1, 1$. The $x$-components can be derived from the eigenvectors resulting in the following roots: $\{(0, 0), (1, -1), (1, 1)\}$, where the root $(0, 0)$ is of multiplicity~$2$.
\end{example}

\section{Proof of Proposition~\ref{prop:schur}}
\label{sec:schur}

The following proposition validates the Schur complement reduction described in Subsec.~\ref{subsec:schur}.

\begin{prop}
Let $M$ be an elimination template represented in the following block form
\begin{equation}
M = \begin{bmatrix}A & B \\ C & D\end{bmatrix},
\end{equation}
where $A$ is a square invertible matrix and its columns correspond to some excessive monomials. Then the Schur complement of $A$, \ie matrix $M/A = D - CA^{-1}B$, is an elimination template too.
\end{prop}

\begin{proof}
Recall that an elimination template is partitioned as $M = \begin{bmatrix}M_{\mathcal E} & M_{\mathcal R} & M_{\overline{\B}} \end{bmatrix}$, where $\mathcal E$, $\mathcal R$ and $\overline{\B}$ are the sets of excessive, reducible and basic monomials respectively. By the definition of template, the reduced row echelon form of $M$ must have the form
\[
\widetilde M = \begin{bmatrix}\widetilde M_{\mathcal E} & 0 & * \\ 0 & I & \widetilde M_{\overline{\B}} \\ 0 & 0 & 0 \end{bmatrix},
\]
where $\begin{bmatrix}\widetilde M_{\mathcal E} \\ 0\end{bmatrix}$ is the reduced row echelon form of matrix $M_{\mathcal E}$. On the other hand, according to the block form~\eqref{eq:blockM}, we have
$
\widetilde M = \begin{bmatrix} \widetilde A & \widetilde B \\ \widetilde C & \widetilde D\end{bmatrix},
$
where $\widetilde A$ is a square invertible submatrix of $\widetilde M_{\mathcal E}$. Thus, $\widetilde A = I$ and $\widetilde C = 0$. Let $\mathcal E_A$ be the set of excessive monomials corresponding to the columns of matrix $A$. Then we have $\widetilde M_{\mathcal E} = \begin{bmatrix} I & * \\ 0 & \widetilde M_{\mathcal E \setminus \mathcal E_A}\end{bmatrix}$. It follows that the reduced row echelon form of $M/A$ is
\[
\widetilde{M/A} = \widetilde D = \begin{bmatrix}\widetilde M_{\mathcal E \setminus \mathcal E_A} & 0 & * \\ 0 & I & \widetilde M_{\overline{\B}} \\ 0 & 0 & 0 \end{bmatrix}
\]
and hence $M/A$ is a template.
\end{proof}

\section{Proof of Proposition~\ref{prop:nco}}
\label{sec:nco}

Here we prove a simple necessary condition for a template to be minimal.
\begin{prop}
Let $M''$ be an elimination template of size $s''\times n''$ whose columns arranged \wrt the partition $\mathcal E \cup \mathcal R \cup \overline{\B}$. Then there exists a template $M$ of size $s\times n$ so that $s \leq s''$, $n \leq n''$ and $n - s = \# \overline{\B}$.
\end{prop}

\begin{proof}
Let $M''$ be an elimination template of size $s''\times n''$. First, we take a maximal subset of independent rows of $M''$ to get template $M'$ of size $s\times n'$ with $s \leq s'$.

Let $M'$ be partitioned as follows
\[
M' = \begin{bmatrix}M'_{\mathcal E} & M'_{\mathcal R} & M'_{\overline{\B}} \end{bmatrix}.
\]
As $M'$ is an elimination template, its reduced row echelon form must be as follows
\[
\widetilde M' = \begin{bmatrix}\widetilde M'_{\mathcal E} & 0 & * \\0 & I & * \end{bmatrix},
\]
where $I$ is the identity matrix of order $\# \mathcal R$. Removing the columns from $M'_{\mathcal E}$ that do not have pivots in $\widetilde M'_{\mathcal E}$ results in matrix $M$ of size $m\times n$, where $n - s = \# \overline{\B}$. Clearly, matrix $M$ is also an elimination template as its reduced row echelon form is given by
\[
\widetilde M = \begin{bmatrix}I & 0 & * \\0 & I & * \end{bmatrix}.
\]
Since the columns from $M'_{\mathcal E}$ that do not have pivots in $\widetilde M'_{\mathcal E}$ do not change the reduced row echelon form of the rest of the matrix~\cite[p.~136]{meyerLA}, it follows that the right most $\# \overline{\B}$ columns in $\widetilde M$ are exactly the same as in $\widetilde M'$.
\end{proof}

\begin{table*}
\centering
\footnotesize
\begin{tabular}{r|ccccc}
\hline\\
Problem \# &
3 (nstd) & 9 (std) & 10 (std) & 15 (std) & 16 (std) \\
\begin{tabular}{c}Error distrib.\vspace{60pt}\end{tabular}\vspace{-30pt} &
\ig{0.145}{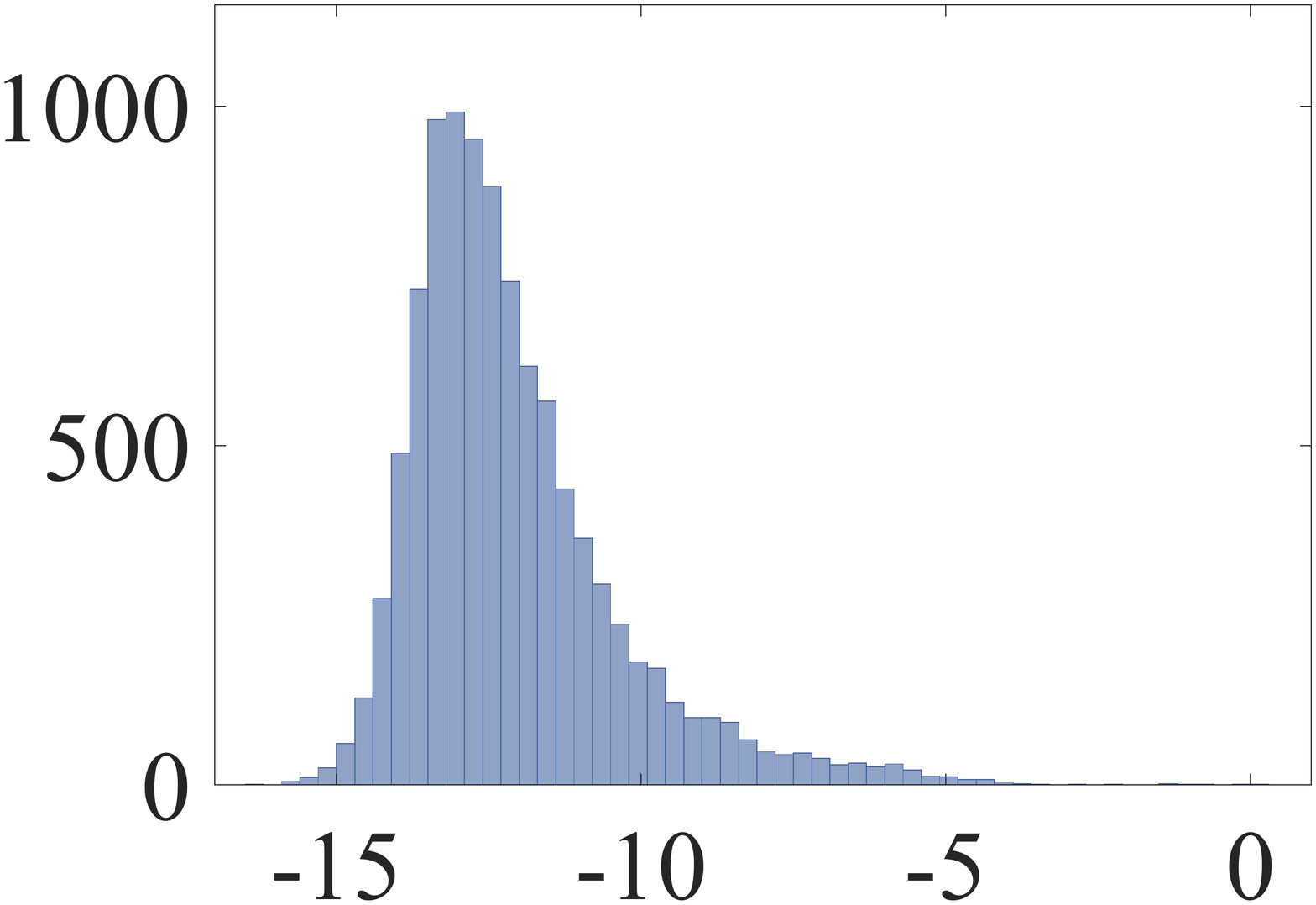} &
\ig{0.145}{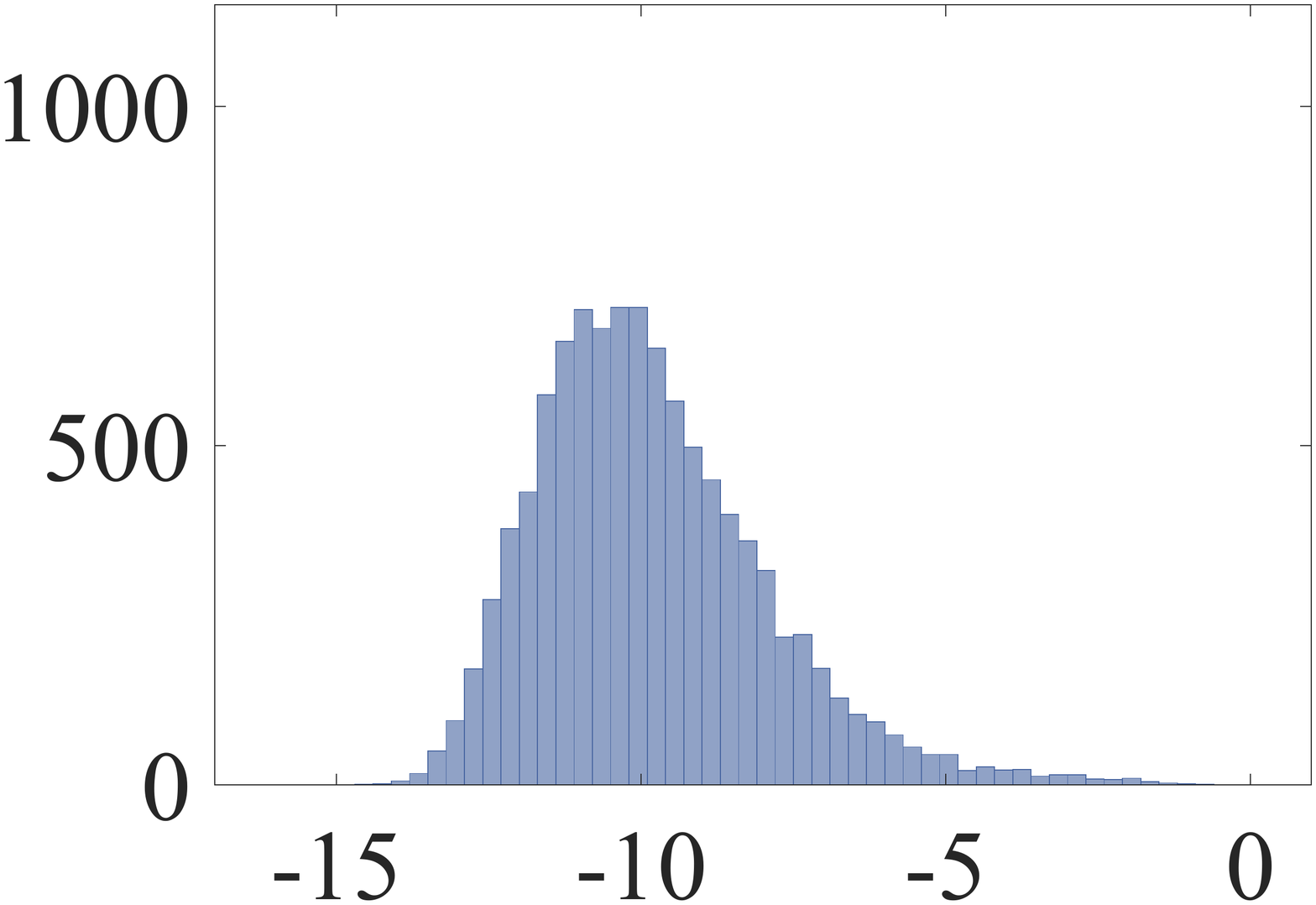} &
\ig{0.145}{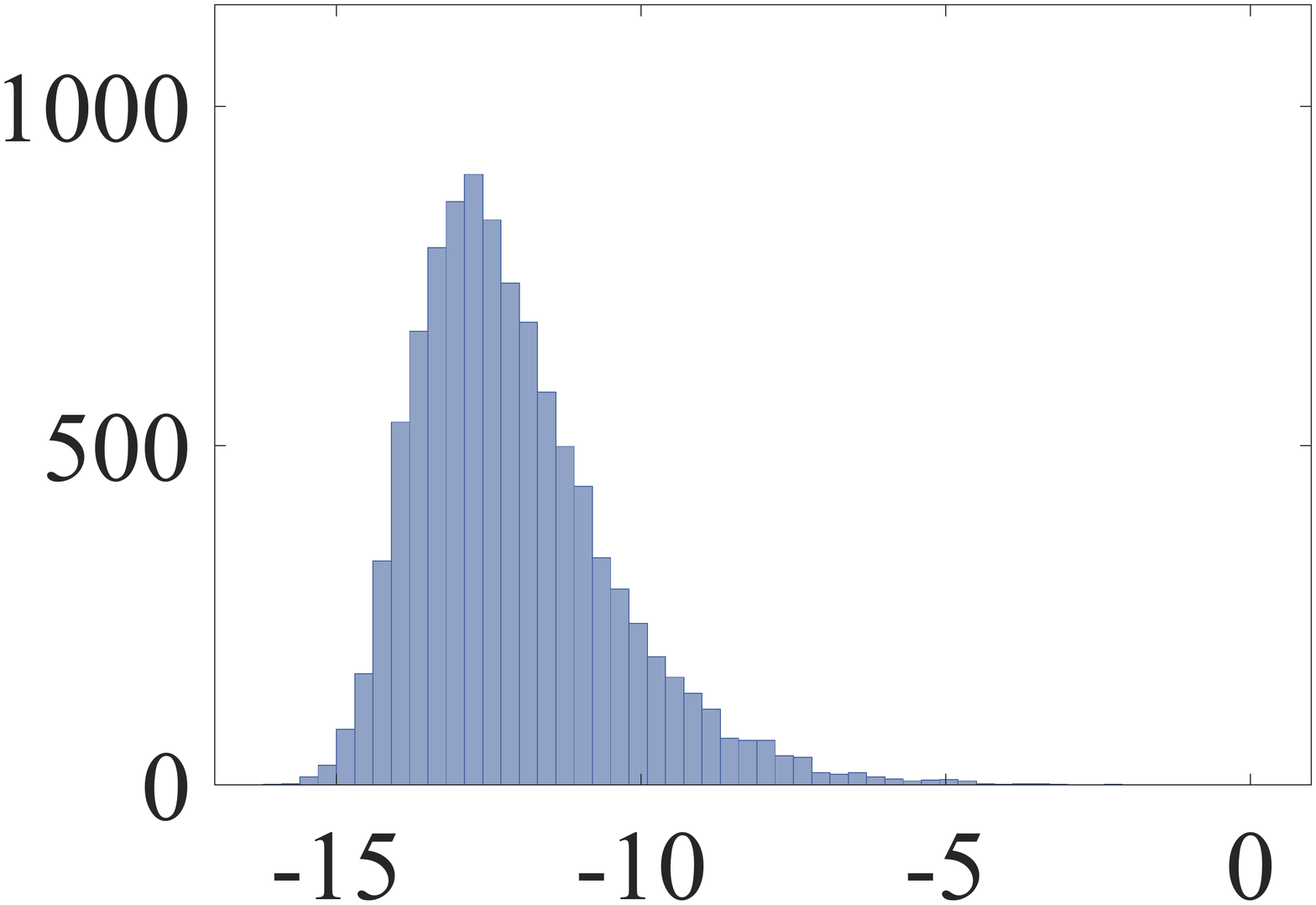} &
\ig{0.145}{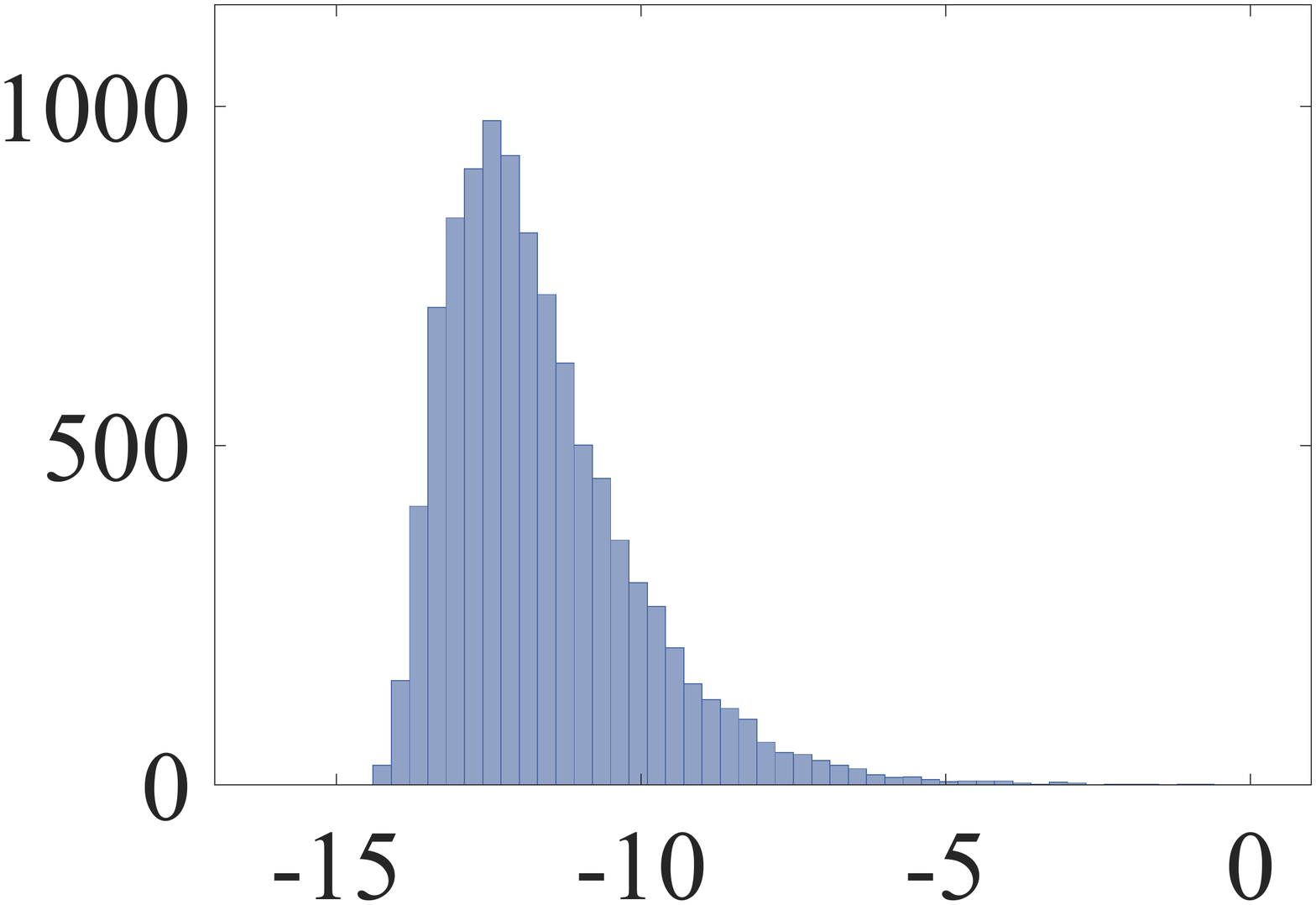} &
\ig{0.145}{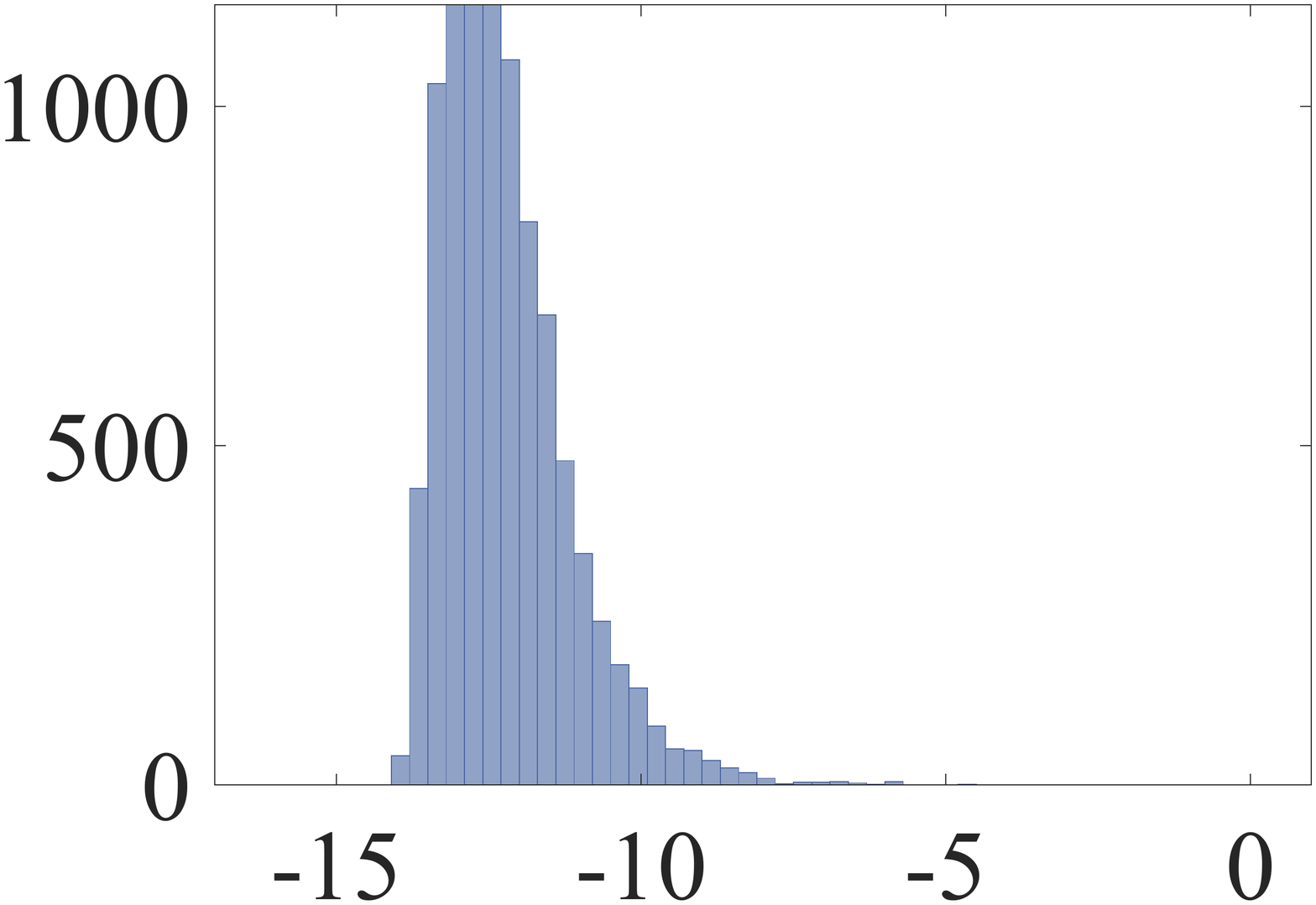} \\
Template size &
$11 \tm 26$ & $76 \tm 100$ & $55 \tm 74$ & $57 \tm 73$ & $65 \tm 85$ \\
$\# \mathcal P$ &
15 & 39 & 30 & 54 & 35 \\
Med. error & 3.30e--13 & 8.08e--11 & 4.19e--13 & 1.04e--12 & 3.41e--13 \\
Ave. time (ms) & $0.4$ & $1.2$ & $3.1$ & $1.0$ & $0.9$ \\
\hline\\
Problem \# &
17 (nstd) & 20 (std) & 21 (std) & 22 (nstd) & 23 (std) \\
\begin{tabular}{c}Error distrib.\vspace{60pt}\end{tabular}\vspace{-30pt} &
\ig{0.145}{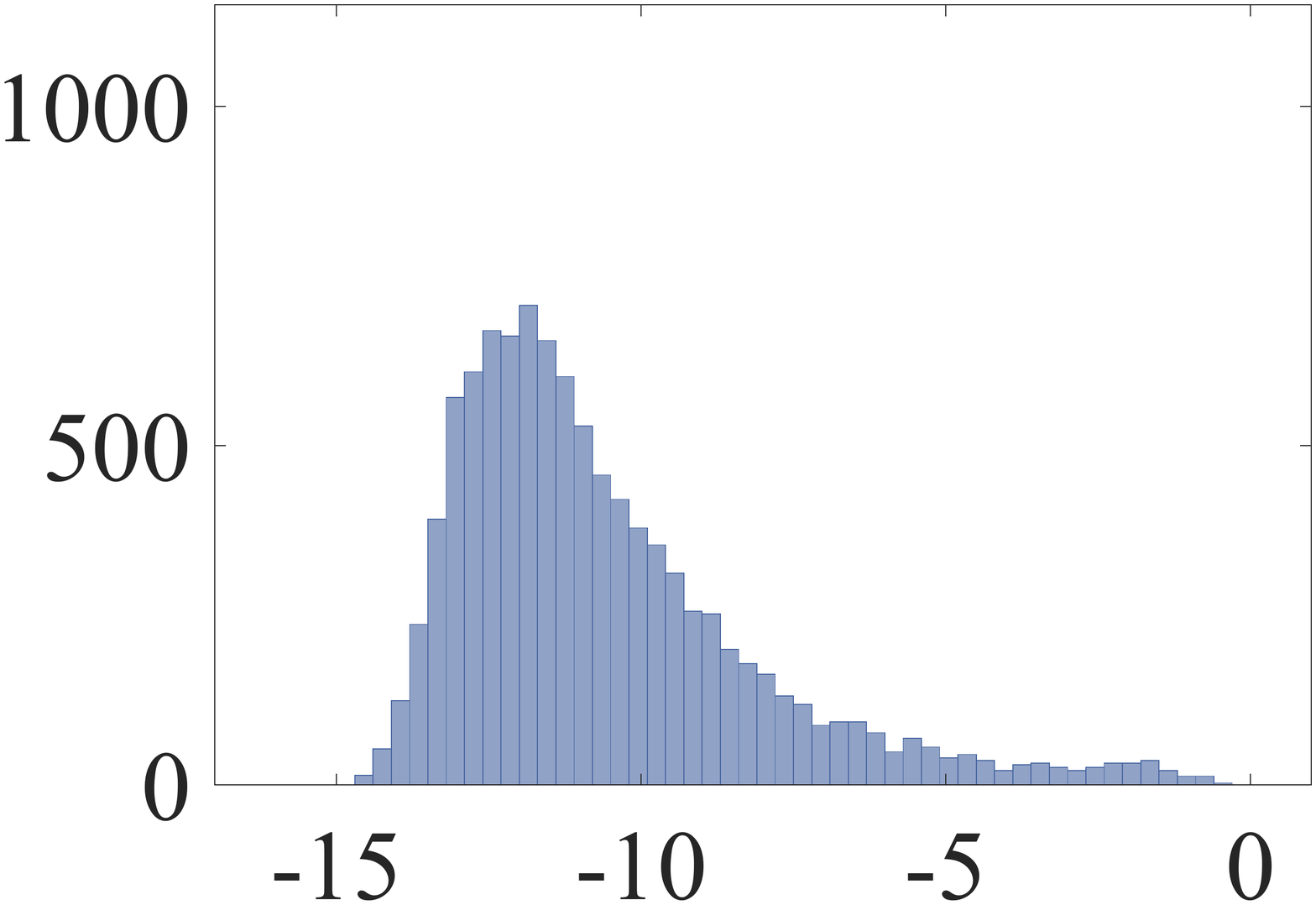} &
\ig{0.145}{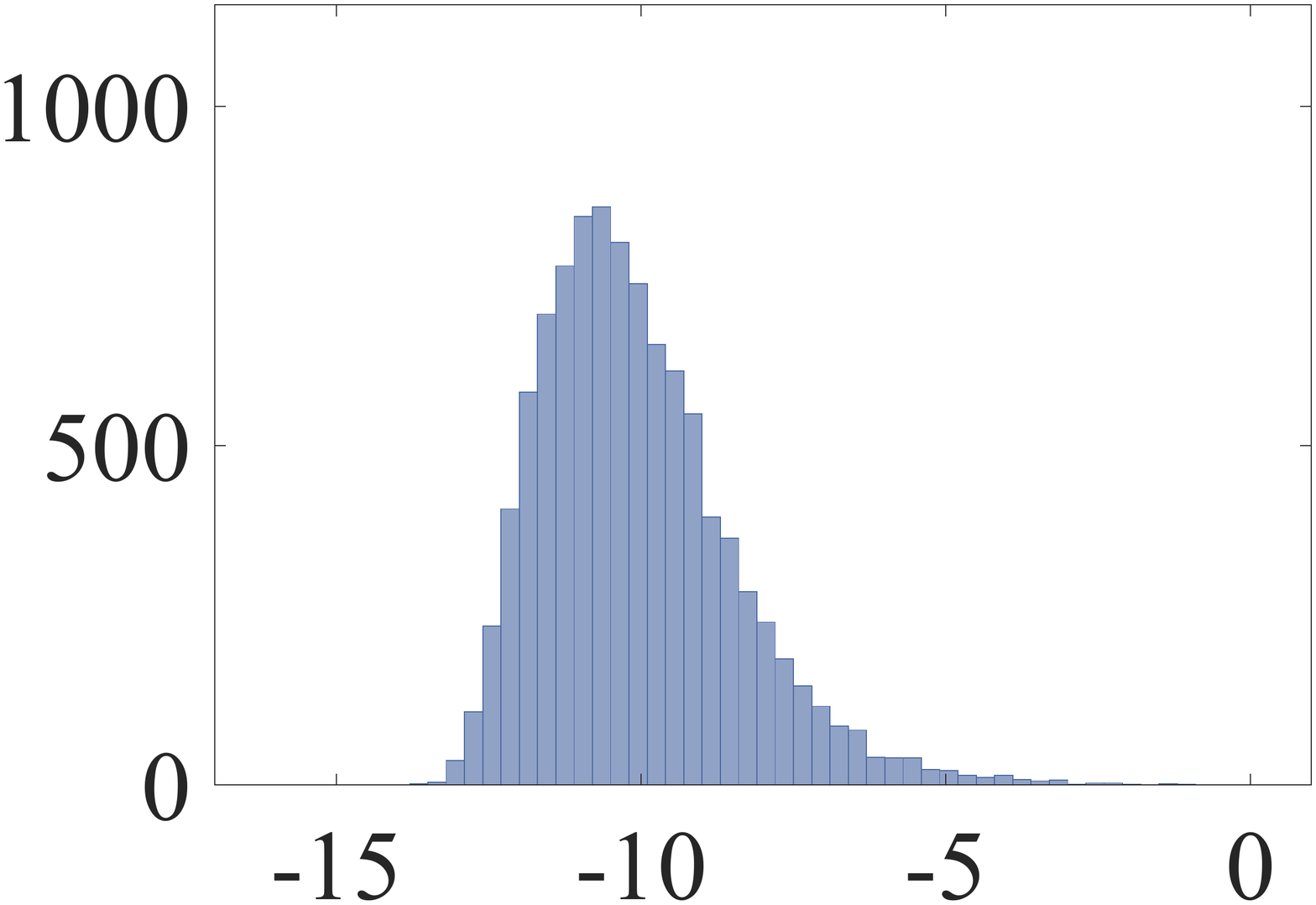} &
\ig{0.145}{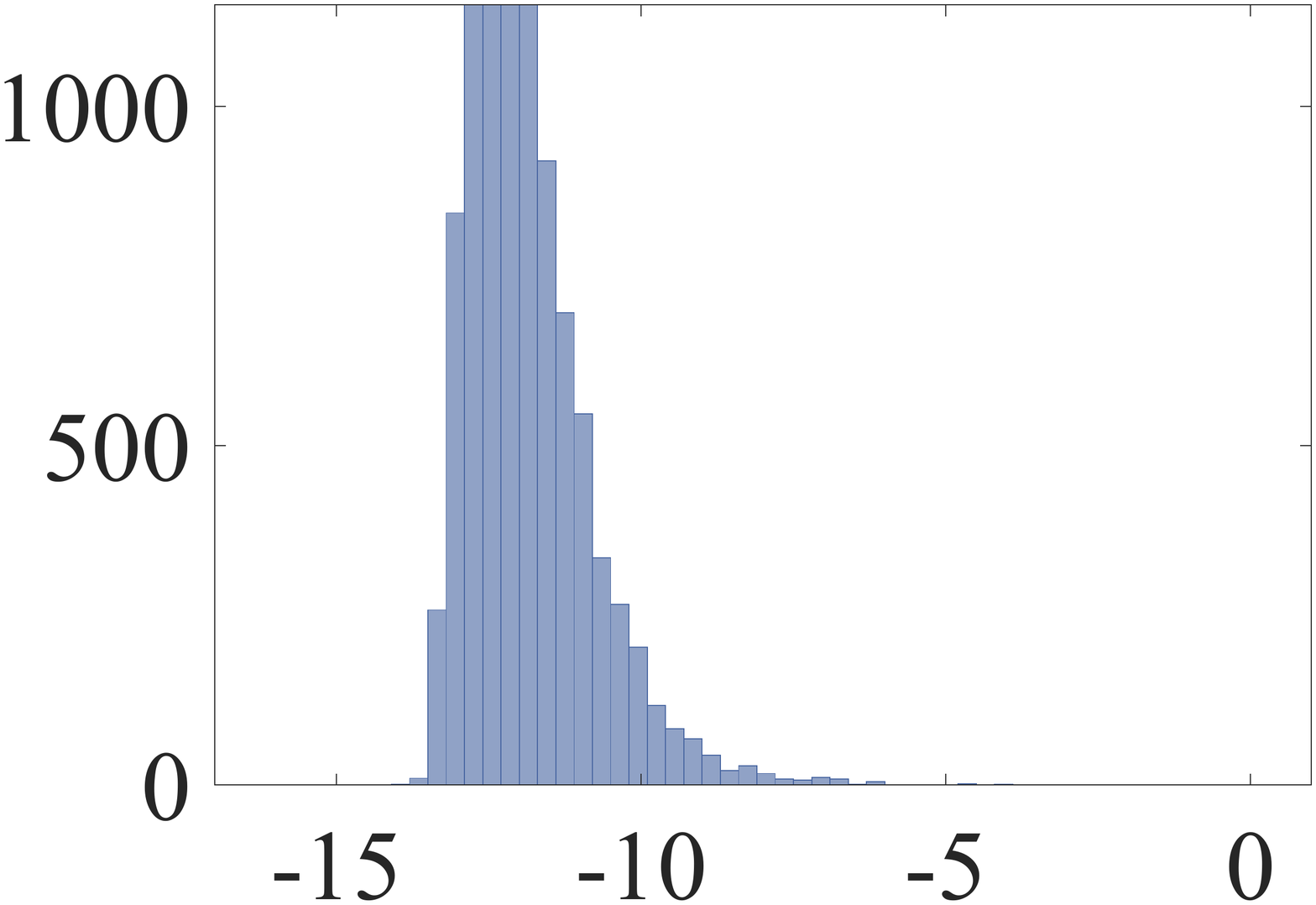} &
\ig{0.145}{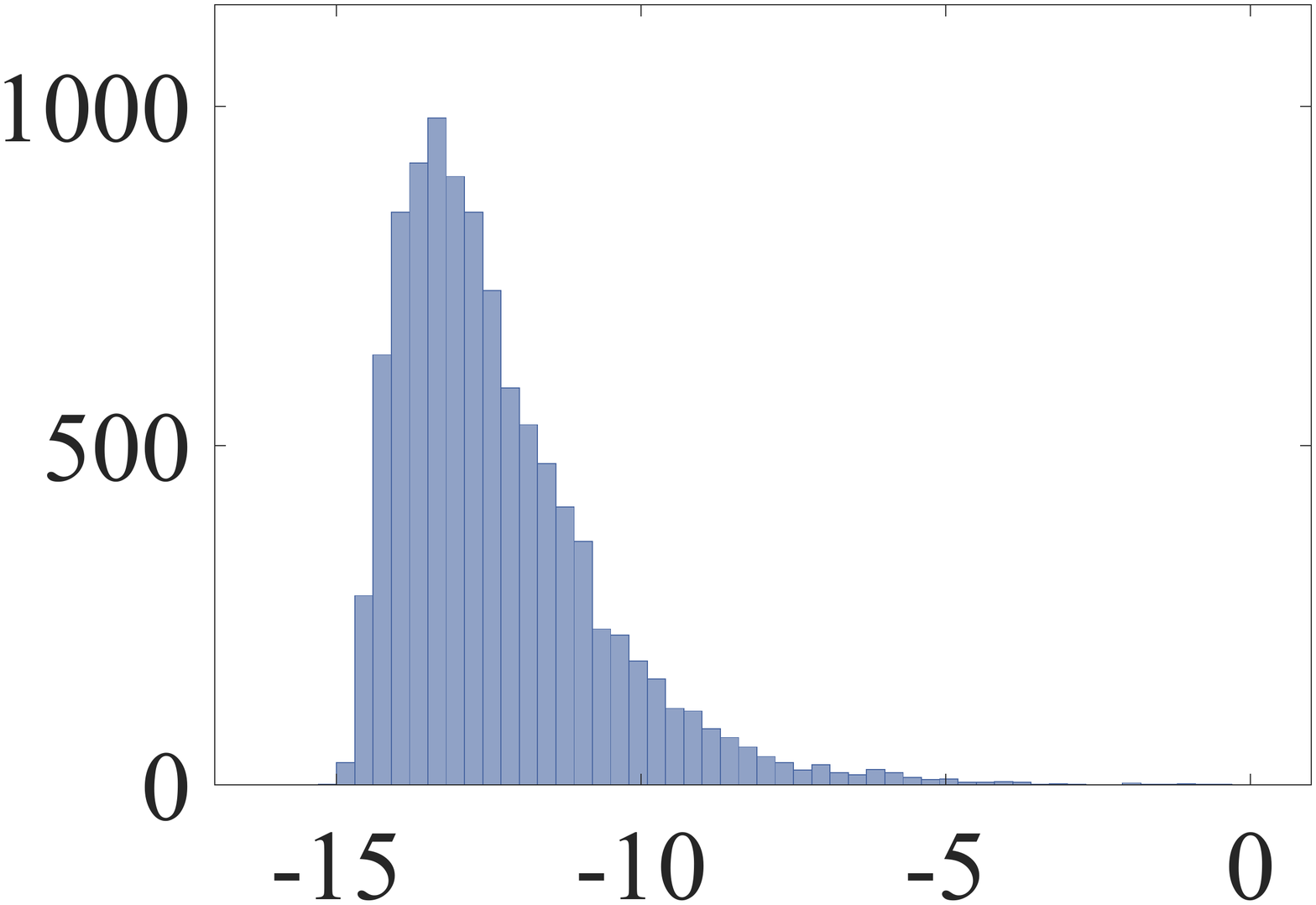} &
\ig{0.145}{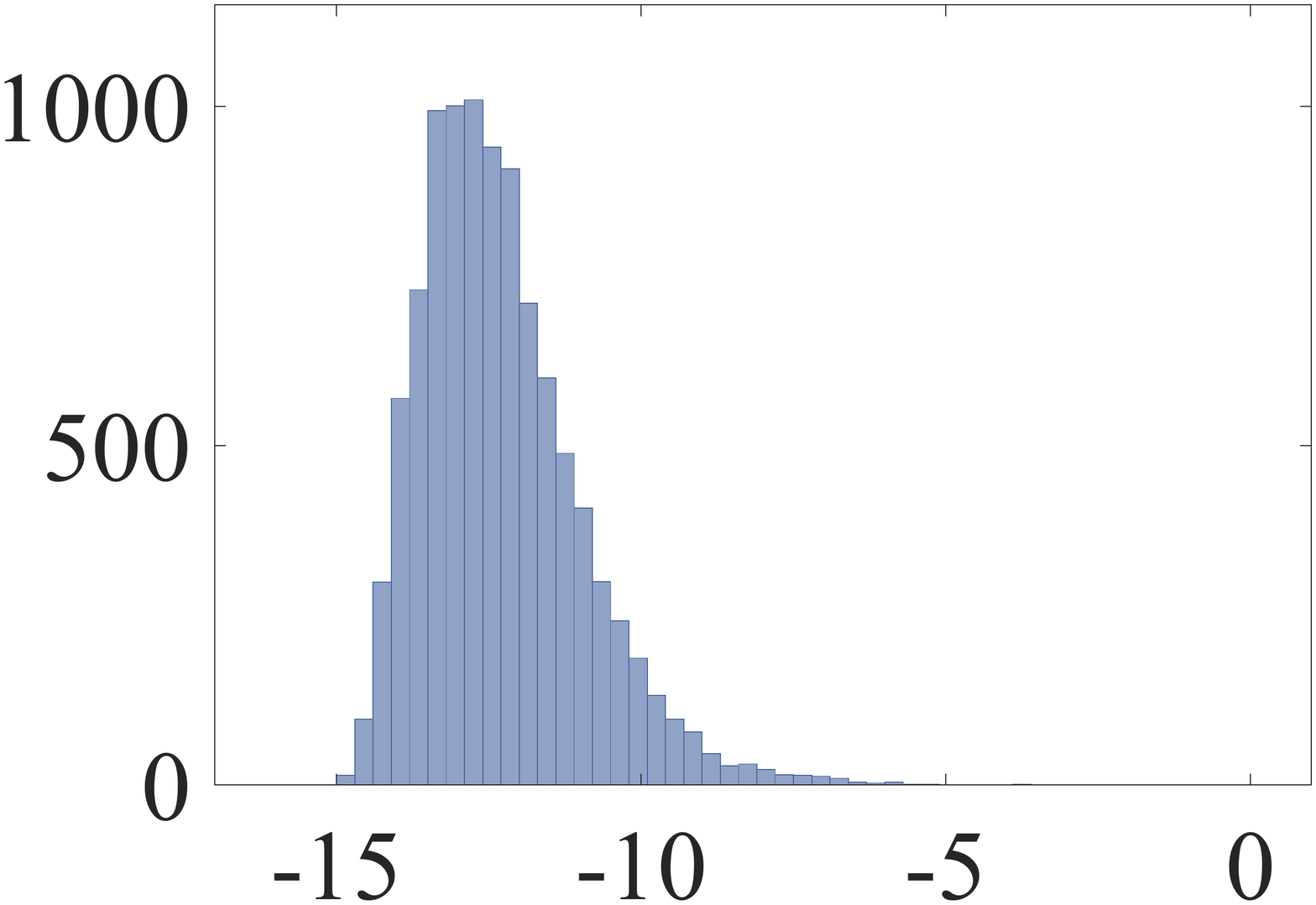} \\
Template size &
$139 \tm 155$ & $139 \tm 163$ & $99 \tm 119$ & $31 \tm 47$ & $18 \tm 28$ \\
$\# \mathcal P$ &
40 & 68 & 48 & 20 & 10 \\
Med. error & 5.52e--12 & 4.67e--11 & 7.95e--13 & 1.71e--13 & 3.06e--13 \\
Ave. time (ms) & $2.5$ & $3.6$ & $2.0$ & $0.5$ & $0.9$ \\
\hline\\
Problem \# &
28 (std) & 29 (std) & 31 (std) & 32 (std) & 33 (std) \\
\begin{tabular}{c}Error distrib.\vspace{60pt}\end{tabular}\vspace{-30pt} &
\ig{0.145}{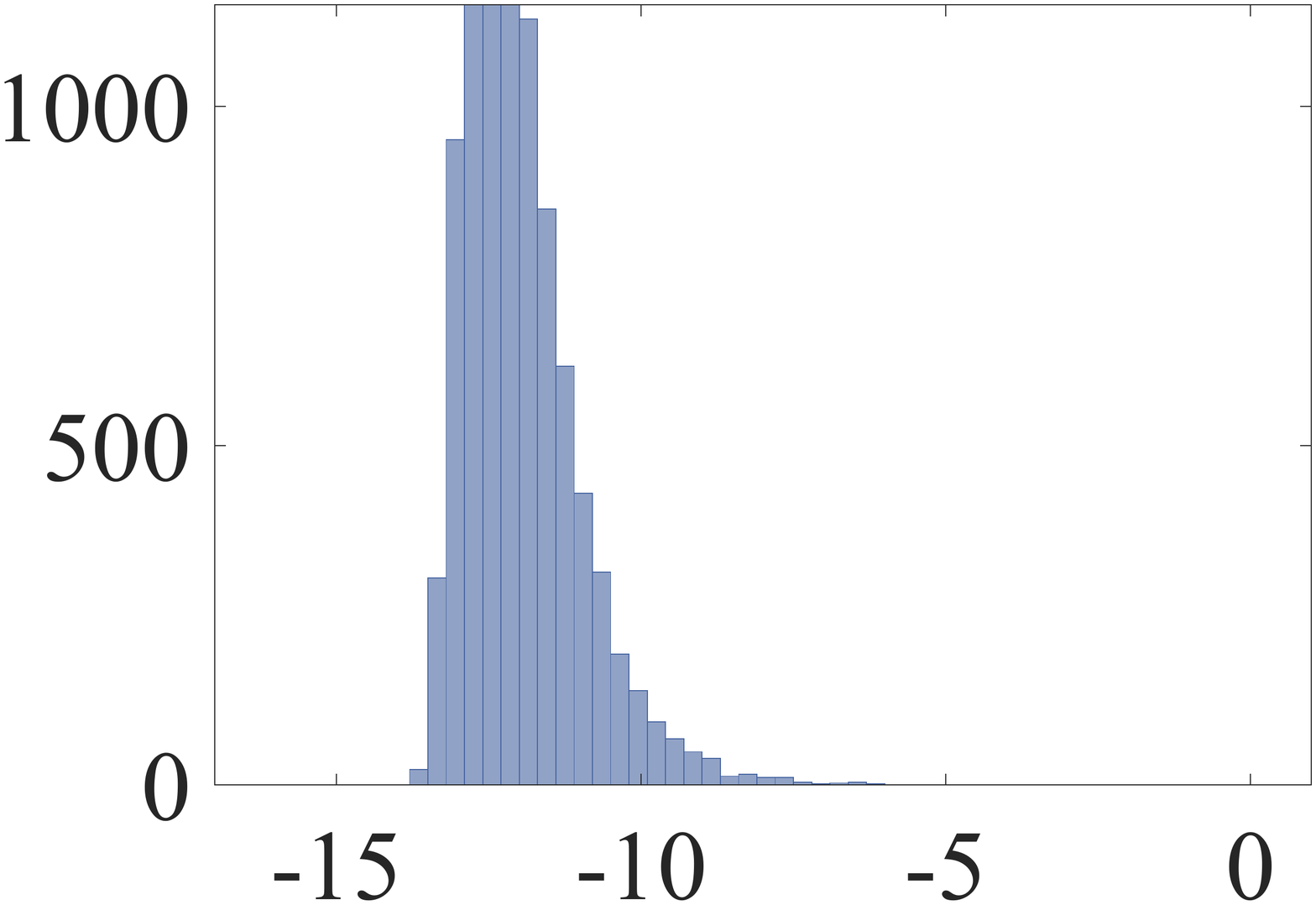} &
\ig{0.145}{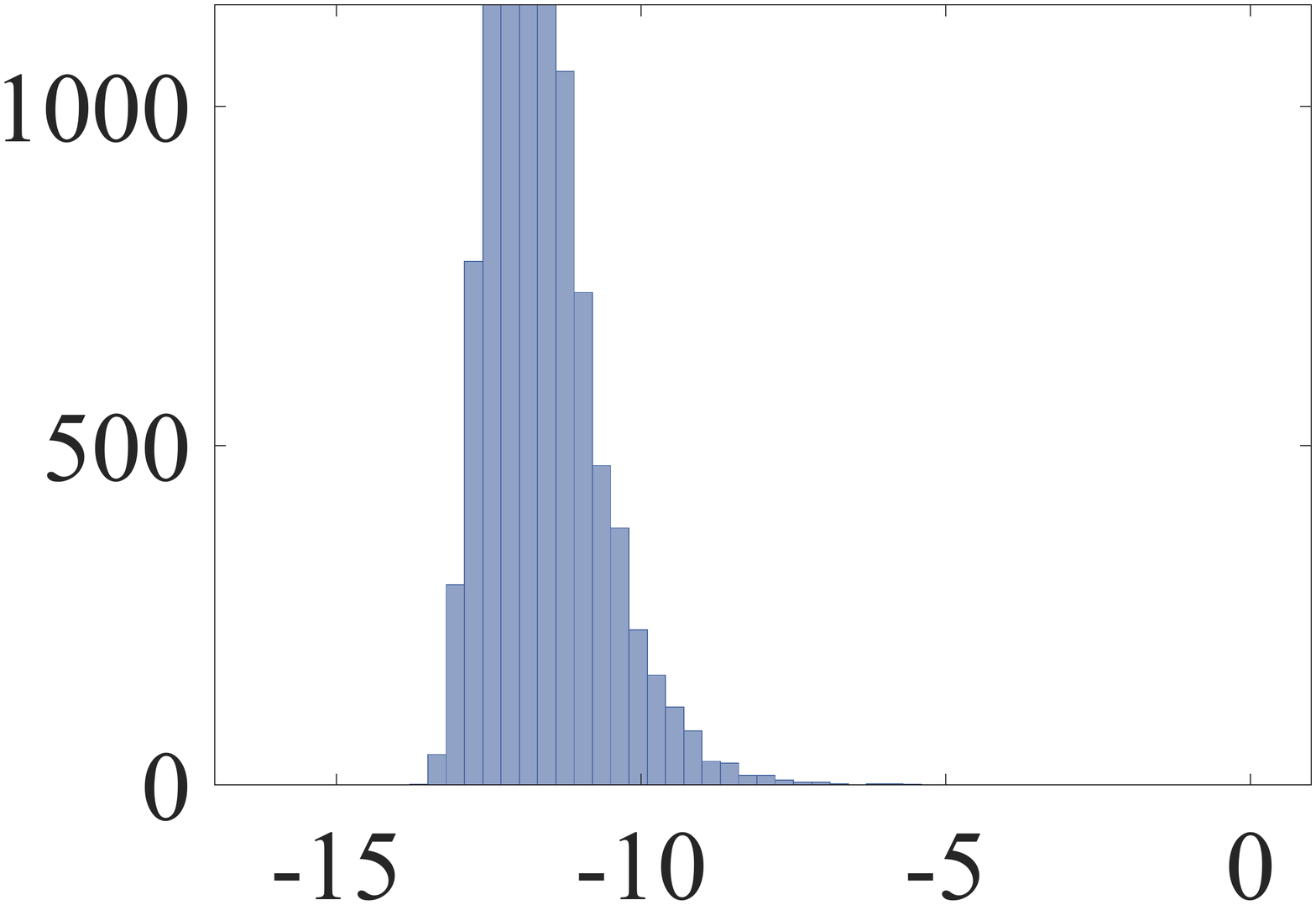} &
\ig{0.145}{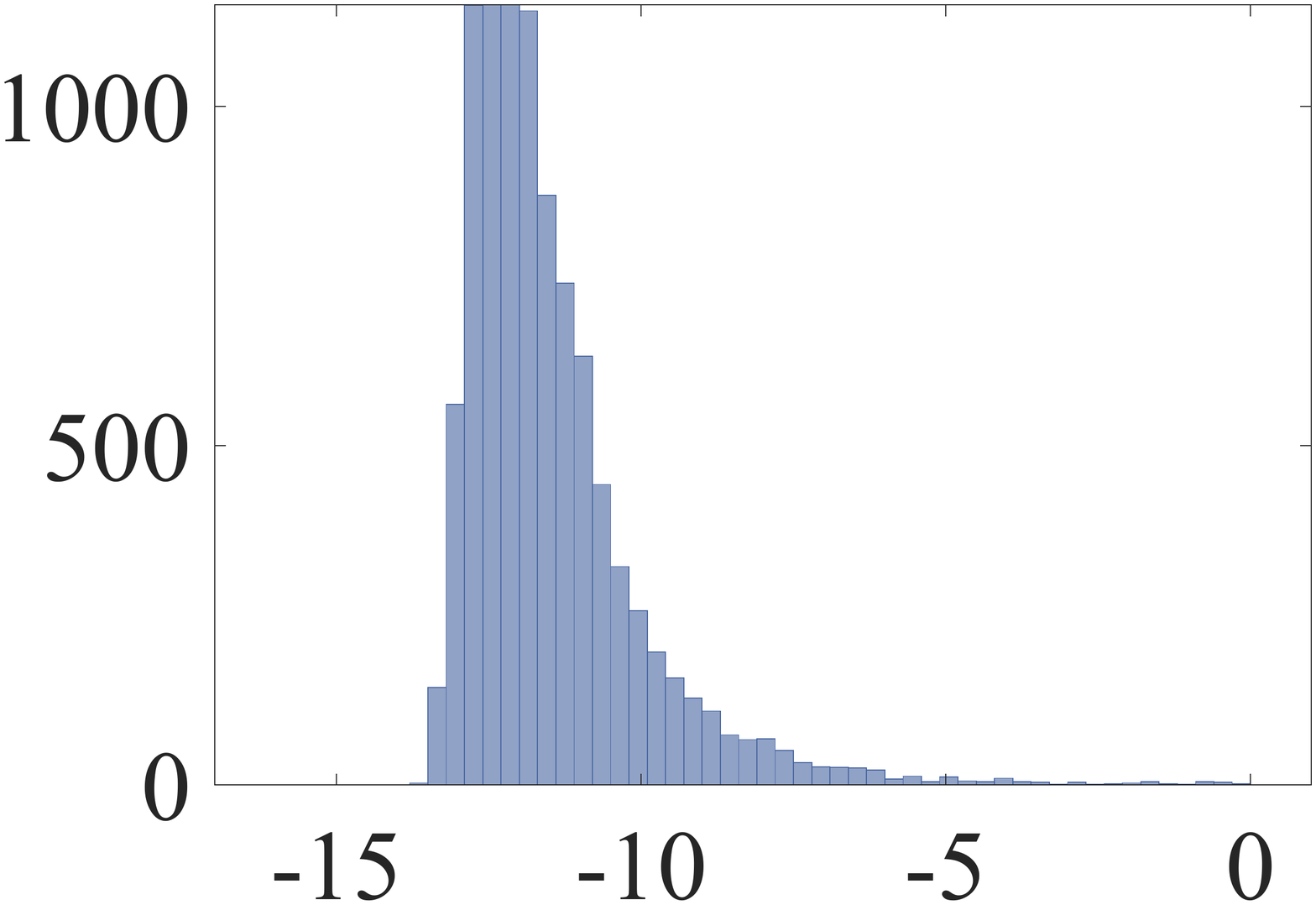} &
\ig{0.145}{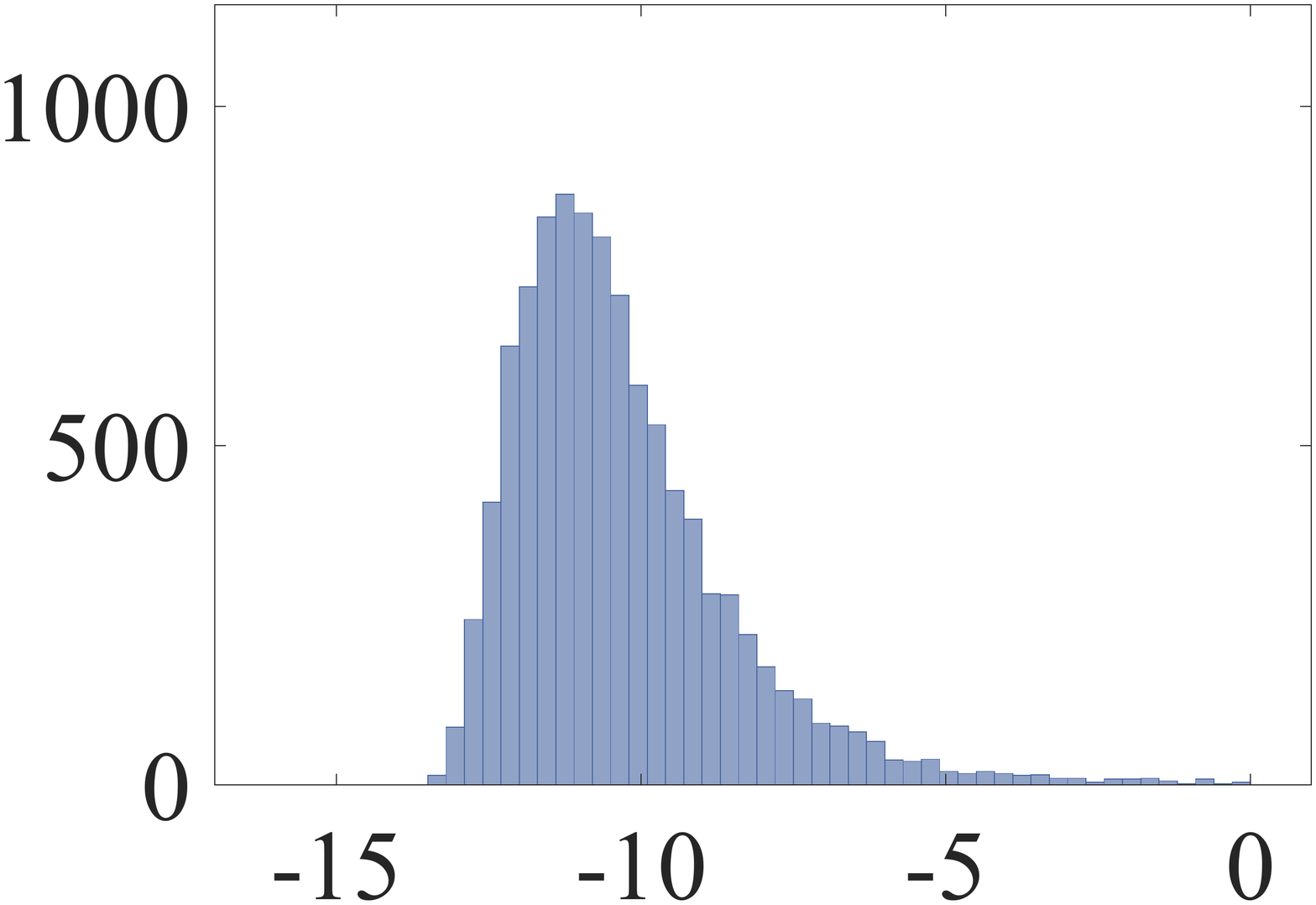} &
\ig{0.145}{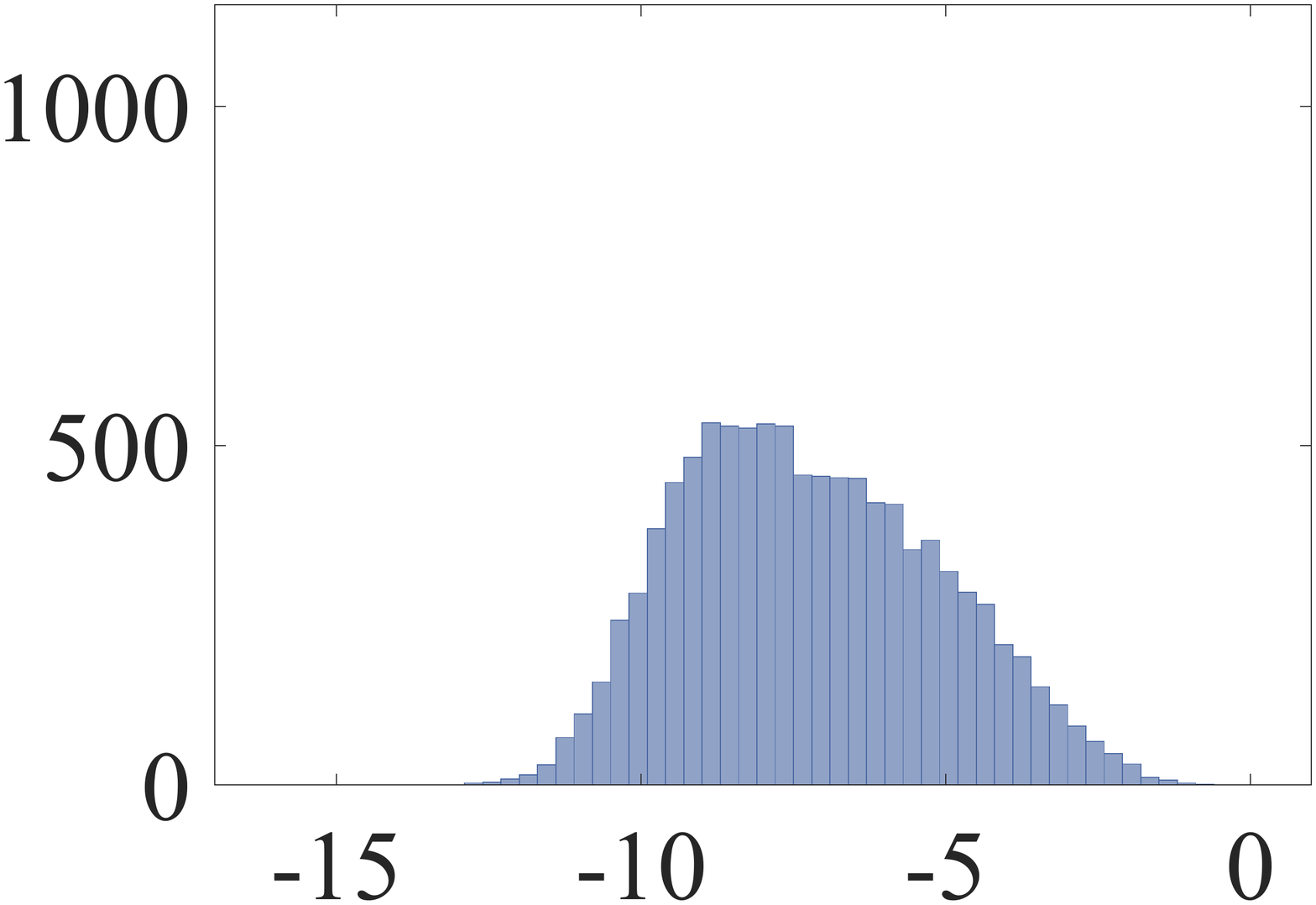} \\
Template size &
$120 \tm 140$ & $134 \tm 162$ & $217 \tm 248$ & $126 \tm 162$ & $209 \tm 277$ \\
$\# \mathcal P$ &
80 & 76 & 85 & 67 & 117 \\
Med. error & 6.11e--13 & 1.63e--12 & 1.31e--12 & 2.09e--11 & 3.77e--08 \\
Ave. time (ms) & $2.6$ & $3.7$ & $4.2$ & $2.3$ & $8.5$ \\
\hline
\end{tabular}
\caption{Tests of numerical accuracy and runtime for some our minimal solvers from Tab.~\ref{tab:templates} and Tab.~\ref{tab:templates2} of the main paper. Each histogram shows $\log_{10}$ of numerical error distribution on $10^4$ trials. It is also shown the template size of each problem and the number of permissible monomials ($\# \mathcal P$) used for the column pivoting strategy from~\cite{byrod2009fast}, see Sec.~\ref{sec:pivot}. For problems \#3 and \#23 the column pivoting was not applied as for those problems $\mathcal P$ is exactly the set of basic monomials. The runtime includes both constructing the coefficient matrix of the initial system and finding its solutions.}
\label{tab:stat}
\end{table*}

\section{Notes on column pivoting}
\label{sec:pivot}

In Subsect.~\ref{subsec:actmat} of the main paper, we read off the action matrix from the reduced row echelon form of the elimination template. For large elimination templates, this method may be impractical for the following two reasons. First, it is slow since constructing the full reduced row echelon form is time-consuming. Second, this approach is often numerically unstable. This means that due to round-off and truncation errors the output roots, when back substituted into the initial polynomials, result in values that are far from being zeros.

Here we recall an alternative approach from~\cite{byrod2007improving,byrod2008column,byrod2009fast} for the action matrix construction. This approach is faster than the one based on the reduced row echelon form and moreover it admits a numerically more accurate generalization.

Let $M$ be an elimination template partitioned as $M = \begin{bmatrix}M_{\mathcal E} & M_{\mathcal R} & M_{\B} \end{bmatrix}$, where $\mathcal E$, $\mathcal R$ and $\B$ are the sets of excessive, reducible and basic monomials respectively. Let the set of basic monomials $\mathcal B$ be partitioned as $\mathcal B = \mathcal B_1 \cup \mathcal B_2$, where $\mathcal B_2 = \{a\, b \cln b \in \mathcal B\} \cap \mathcal B$ and $\mathcal B_1 = \mathcal B \setminus \mathcal B_2$.

The LU decomposition of matrix $M_{\mathcal E}$ can be generally written as $M_{\mathcal E} = \begin{bmatrix}\Pi_{\mathcal E} L_{\mathcal E} & 0 \\ 0 & I\end{bmatrix} \begin{bmatrix}U_{\mathcal E} \\ 0\end{bmatrix}$, where $U_{\mathcal E}$ and $L_{\mathcal E}$ are upper- and lower-triangular matrices respectively, $\Pi_{\mathcal E}$ is a row permutation matrix. Then we define
\[
M' = \begin{bmatrix}(\Pi_{\mathcal E} L_{\mathcal E})^{-1} & 0 \\ 0 & I\end{bmatrix} M = \begin{bmatrix}U_{\mathcal E} & * & * \\ 0 & M'_{\mathcal R} & M'_{\B}\end{bmatrix},
\]
where $M'_{\mathcal R}$ is square and invertible. It follows that $M'_{\mathcal R} \vect{\mathcal R} = -M'_{\B}\vect{\B}$ and hence the action matrix reads
\[
T_a = \begin{bmatrix} -(M'_{\mathcal R})^{-1}M'_{\B} \\ P \end{bmatrix},
\]
where $P$ is a binary matrix, \ie a matrix consisting of $0$ and $1$, such that $\vect{\B_2} = P\vect{\B}$.

\begin{table*}
\centering
\footnotesize
\begin{tabular}{r|ccccc}
\hline\\
Problem \# &
1 (nstd) & 2 (std) & 4 (nstd) & 5 (nstd) & 6 (std) \\
\begin{tabular}{c}Error distrib.\vspace{60pt}\end{tabular}\vspace{-30pt} &
\ig{0.145}{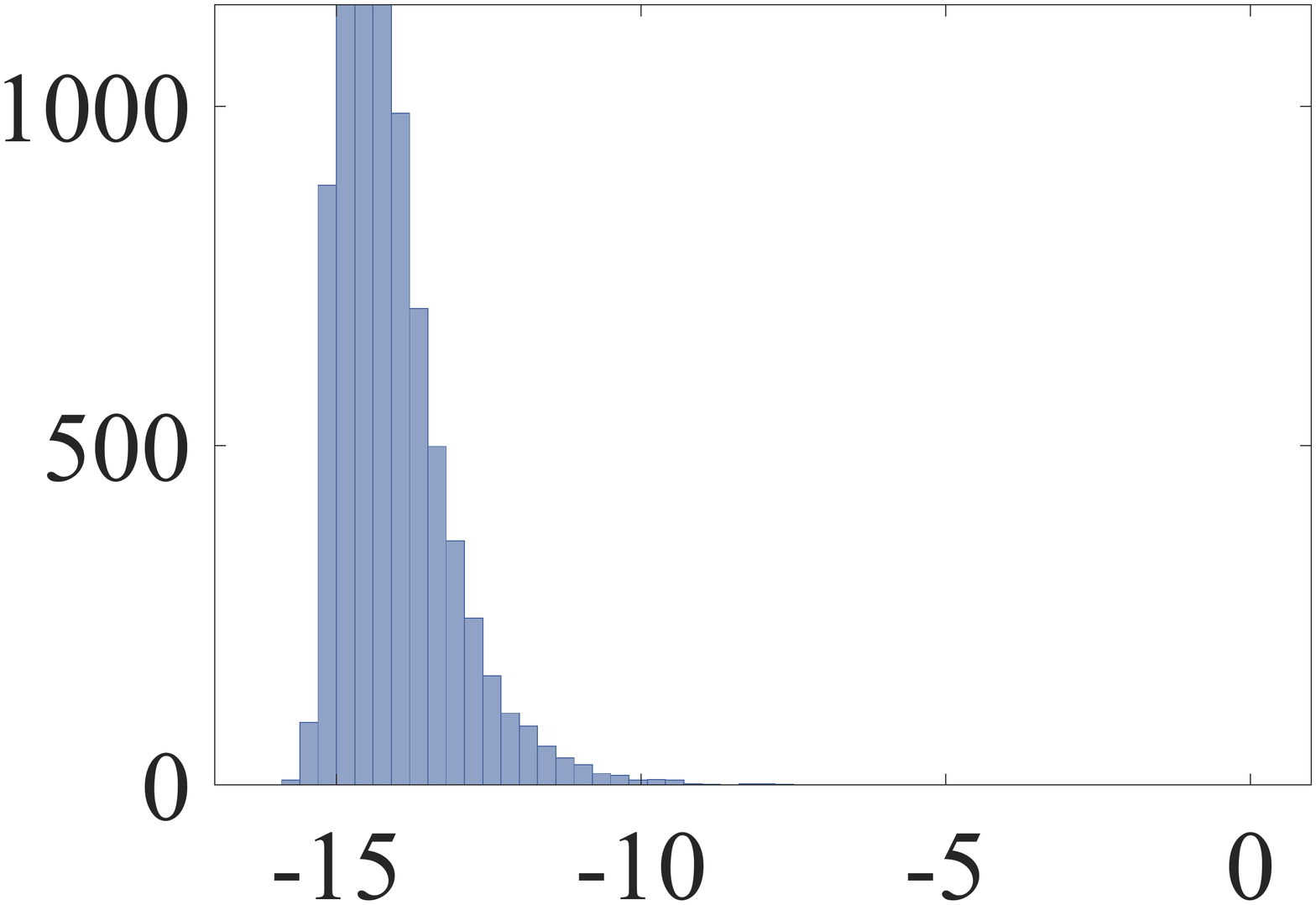} &
\ig{0.145}{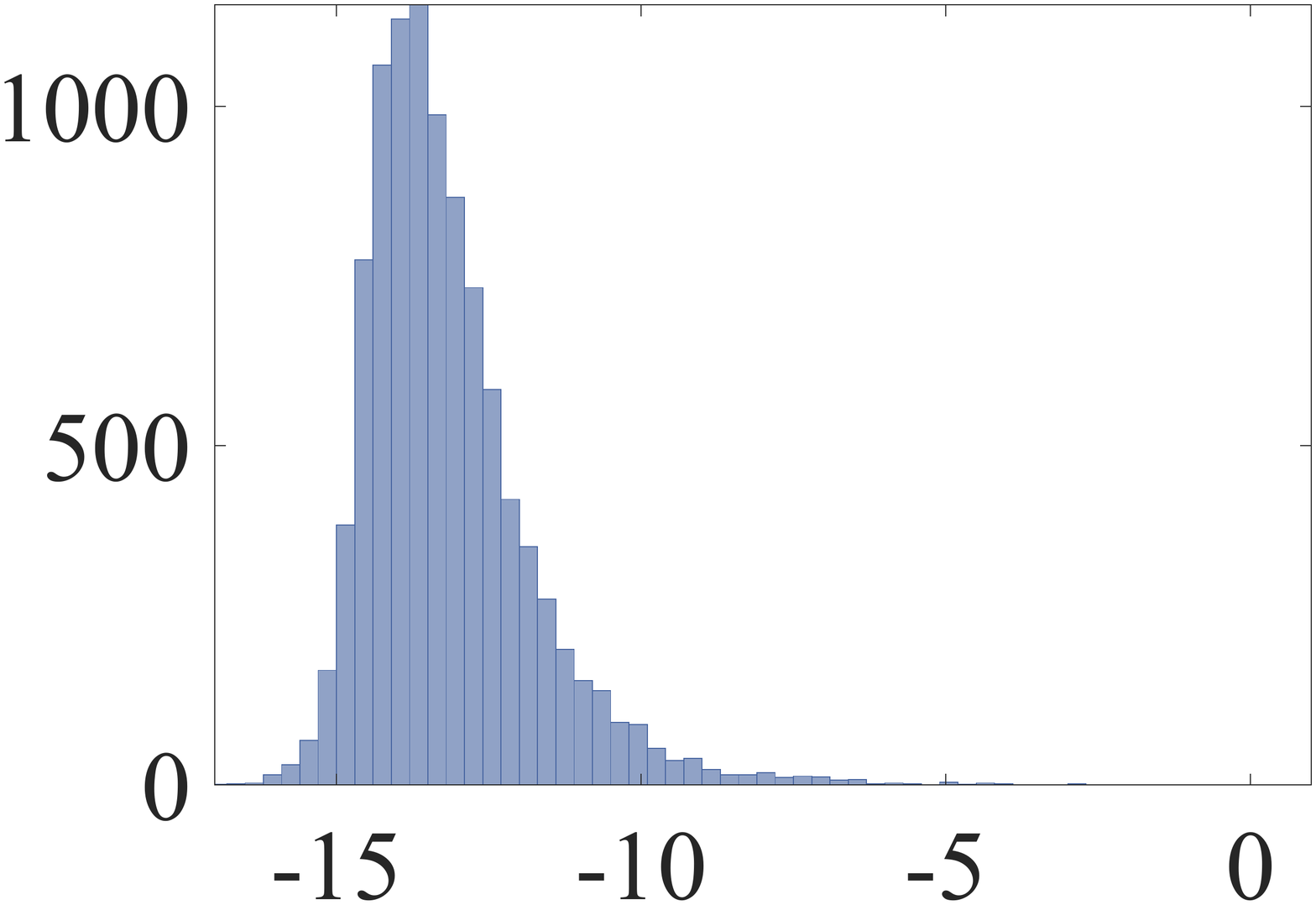} &
\ig{0.145}{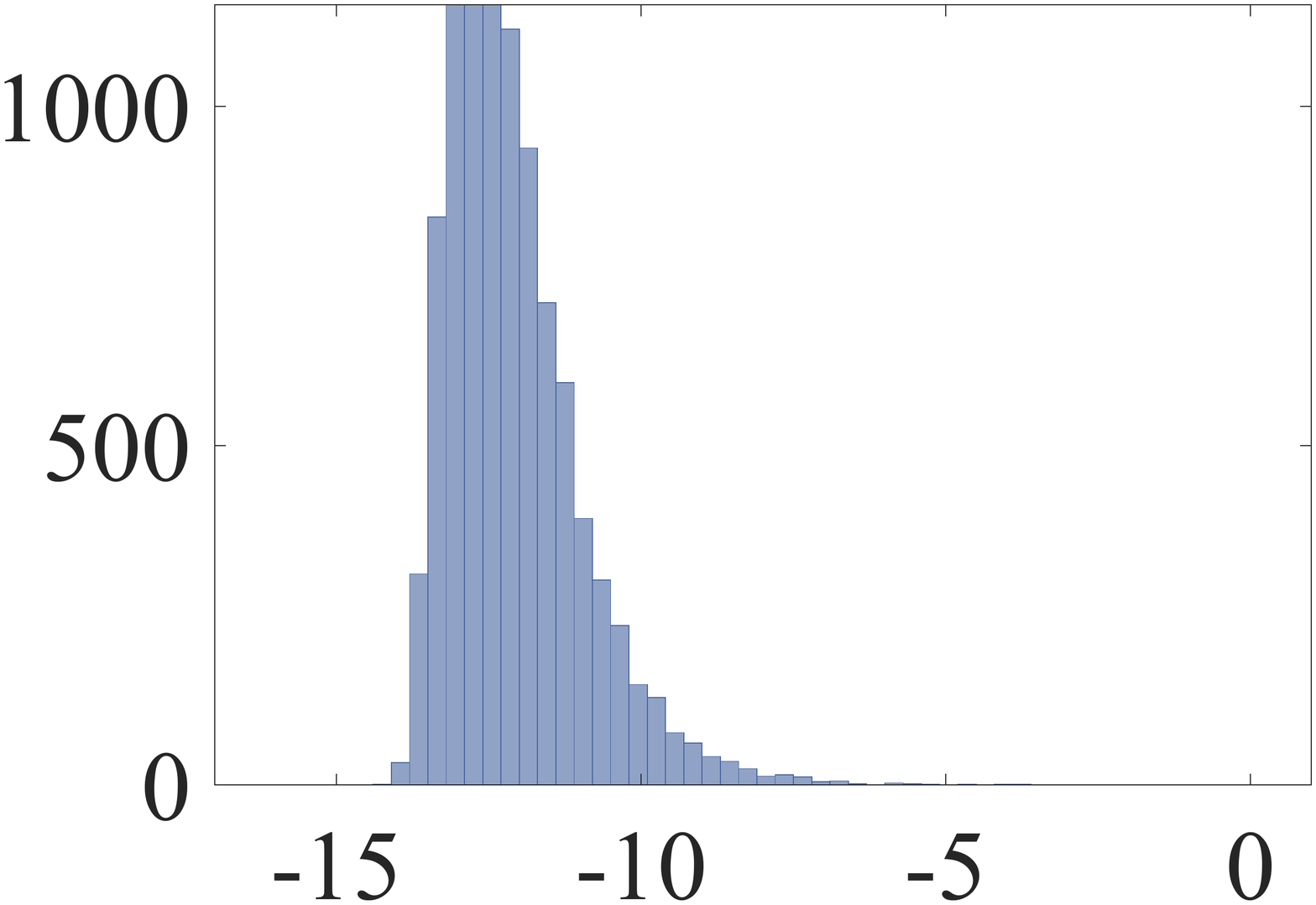} &
\ig{0.145}{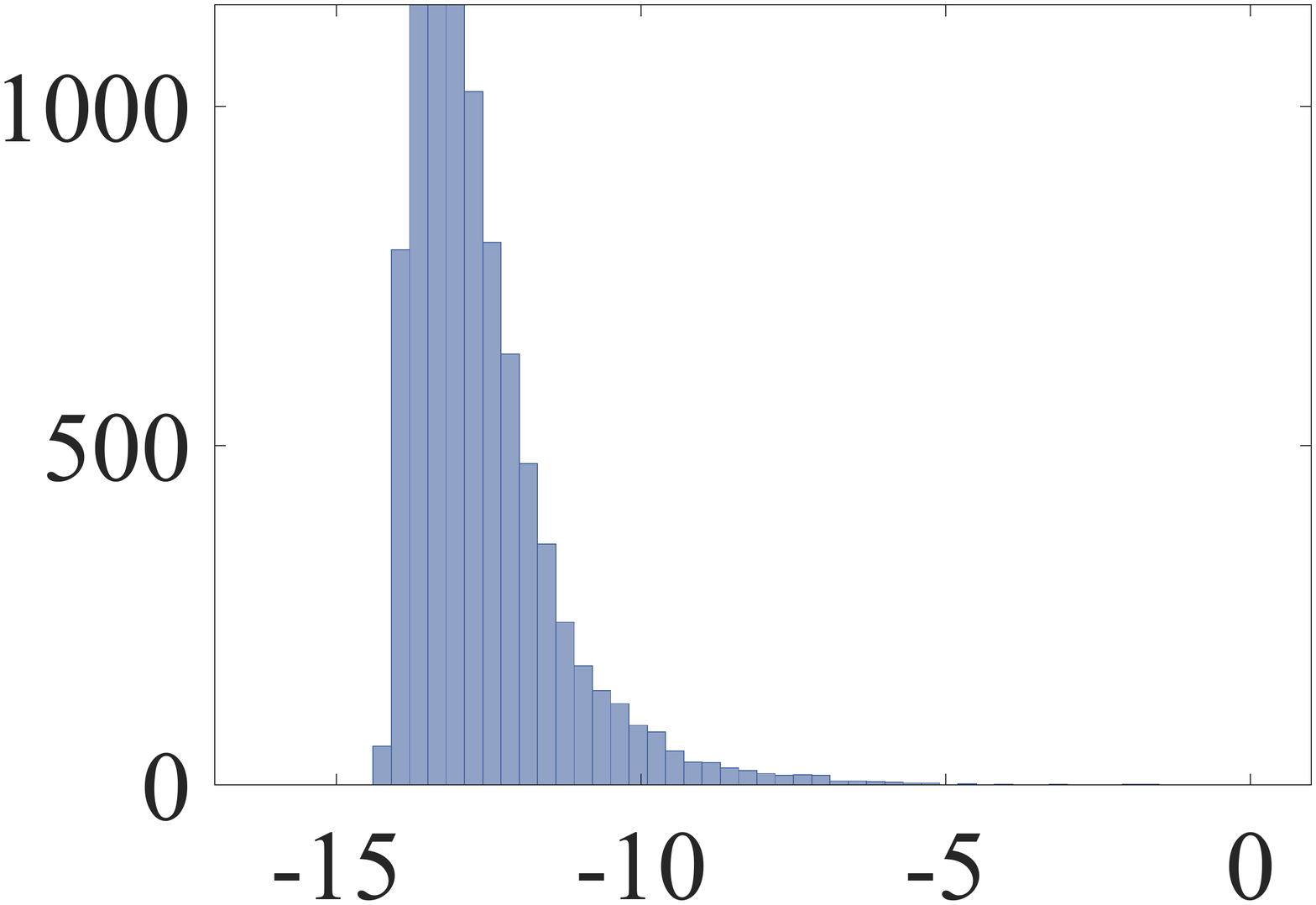} &
\ig{0.145}{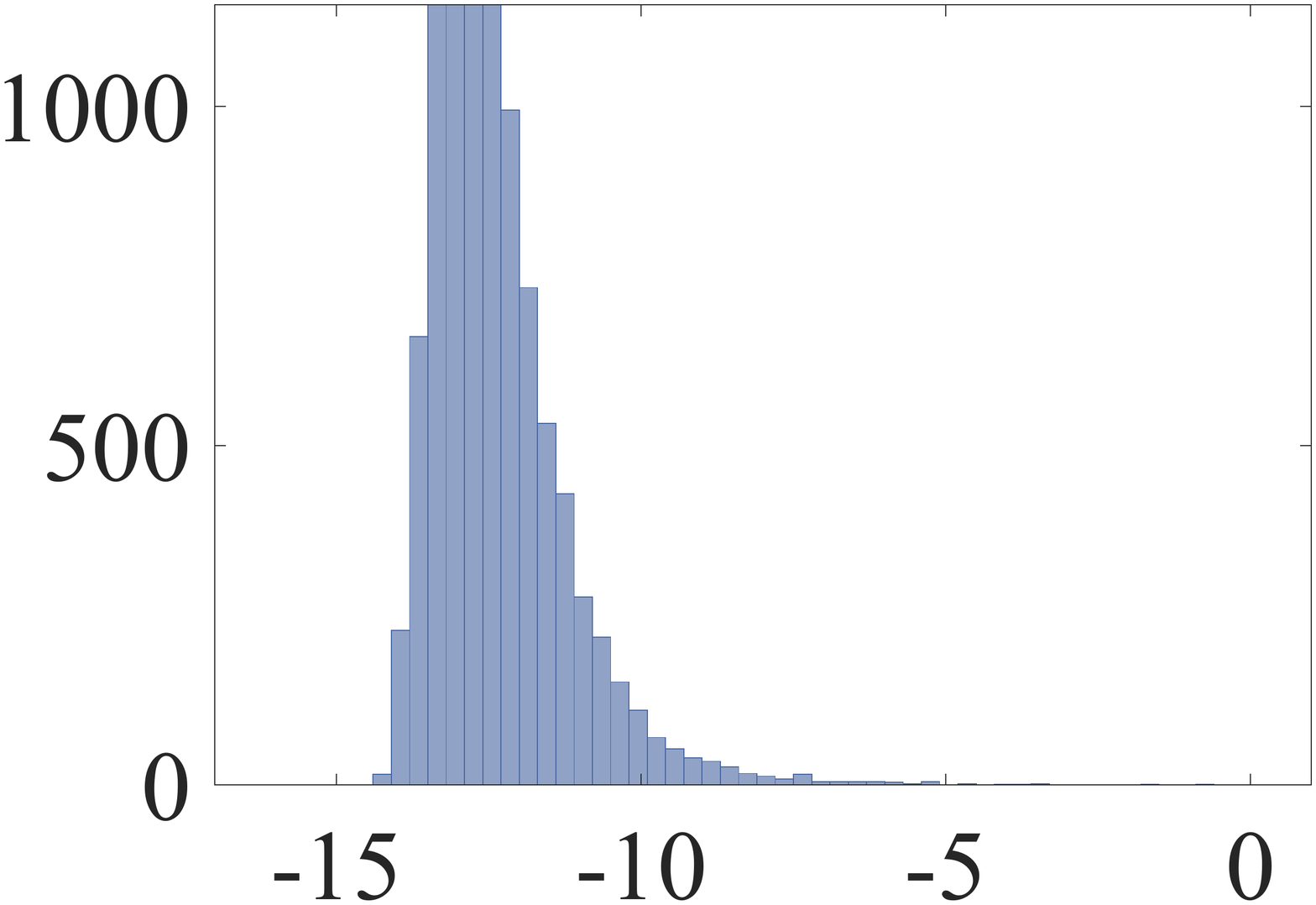} \\
Template size &
$7 \tm 15$ & $11 \tm 20$ & $14 \tm 40$ & $18 \tm 36$ & $52 \tm 68$ \\
$\# \mathcal P$ &
8 & 12 & 30 & 30 & 38 \\
Med. error & 3.67e--15 & 3.52e--14 & 4.64e--13 & 9.10e--14 & 2.46e--13 \\
Ave. time (ms) & $0.2$ & $0.5$ & $1.5$ & $0.5$ & $0.8$ \\
\hline\\
Problem \# &
7 (std) & 8 (nstd) & 11 (nstd) & 12 (std) & 13 (std) \\
\begin{tabular}{c}Error distrib.\vspace{60pt}\end{tabular}\vspace{-30pt} &
\ig{0.145}{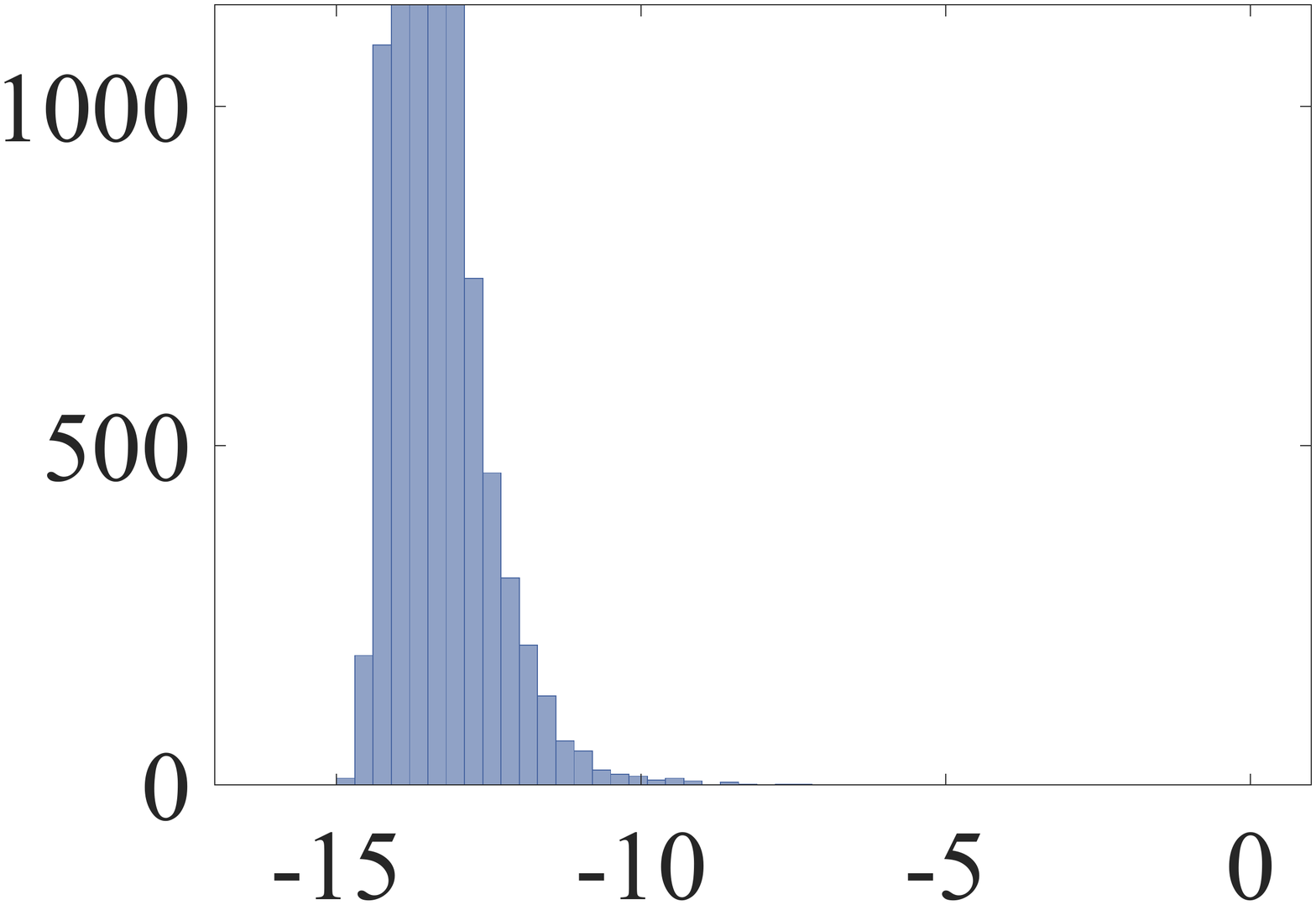} &
\ig{0.145}{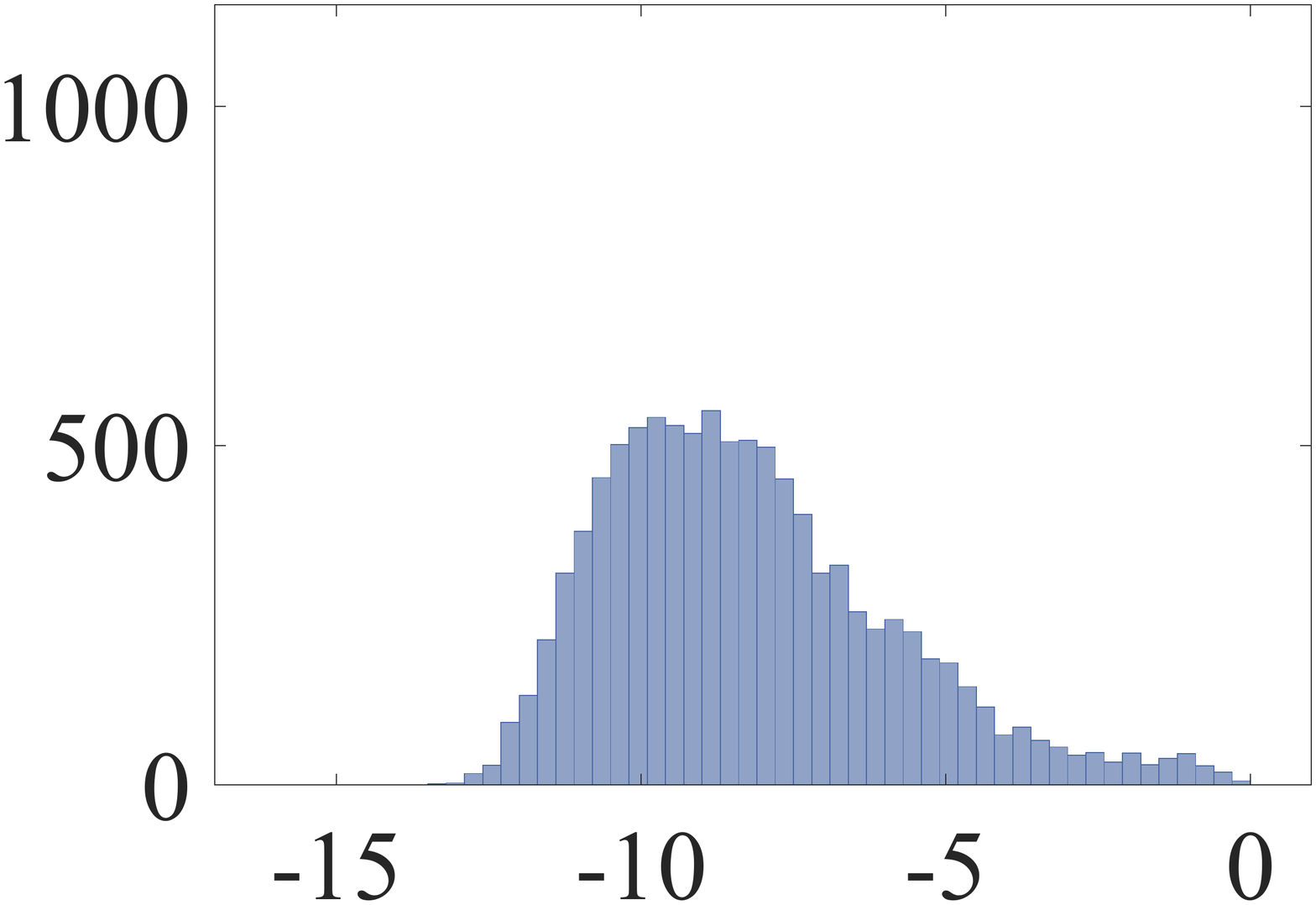} &
\ig{0.145}{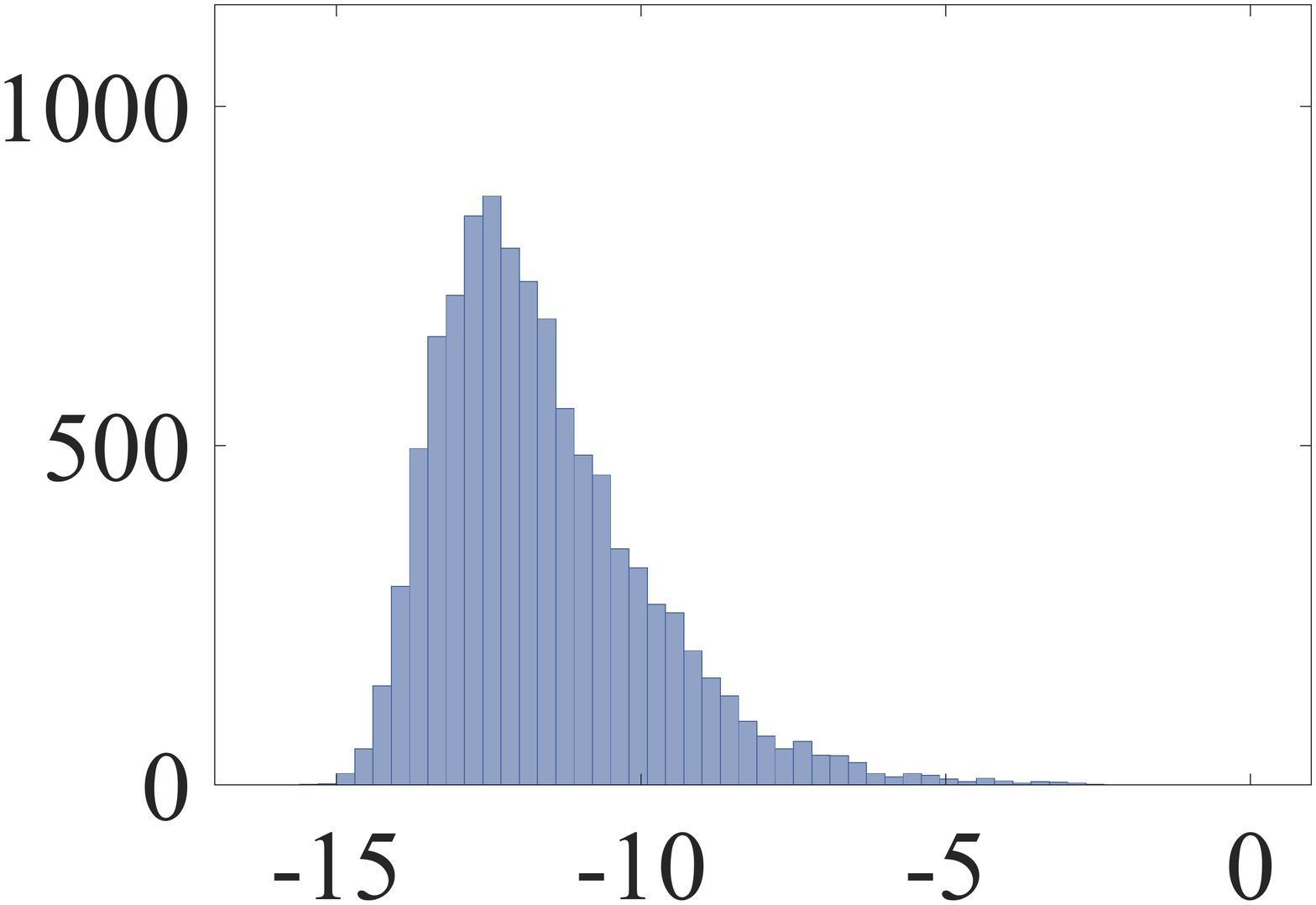} &
\ig{0.145}{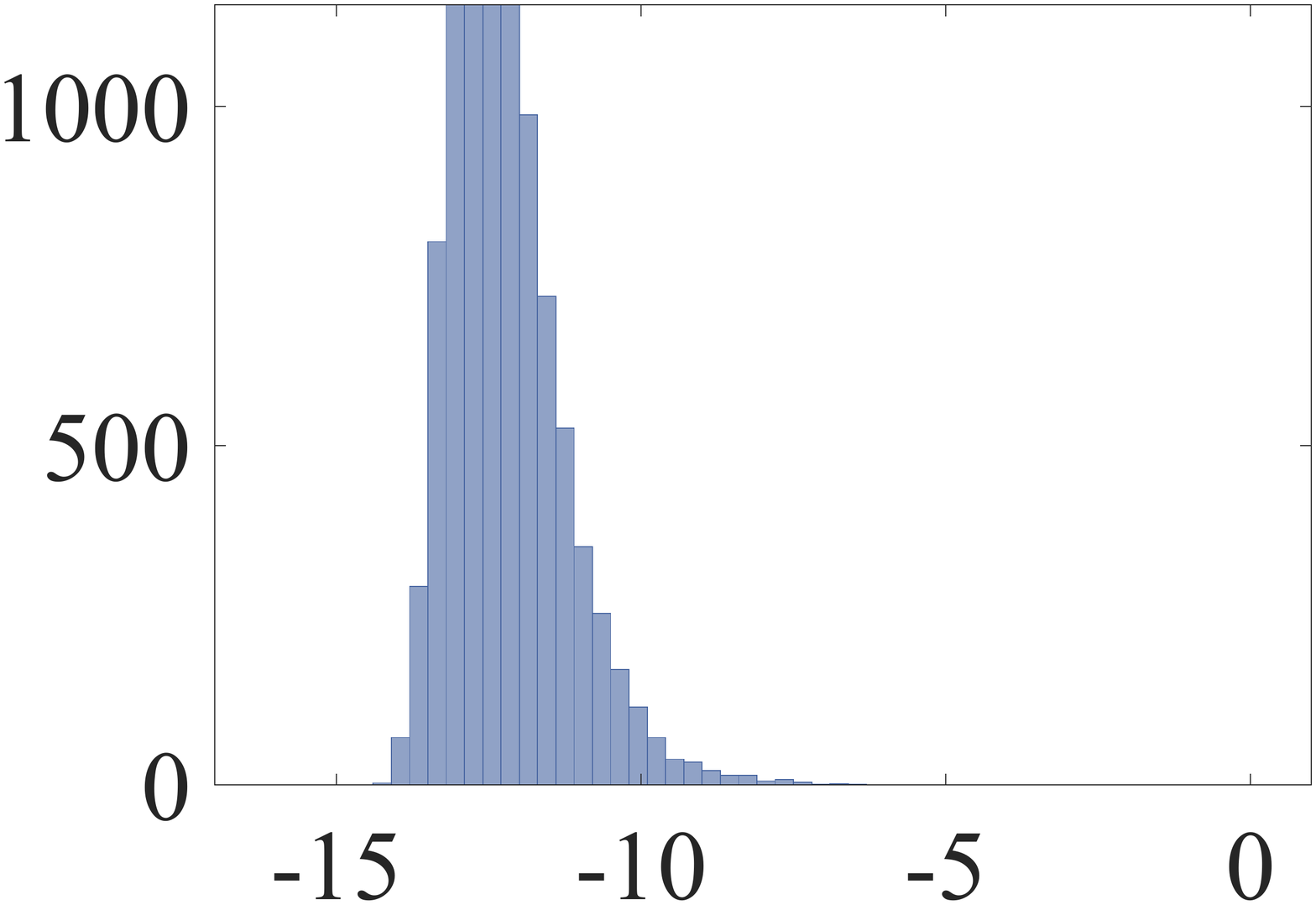} &
\ig{0.145}{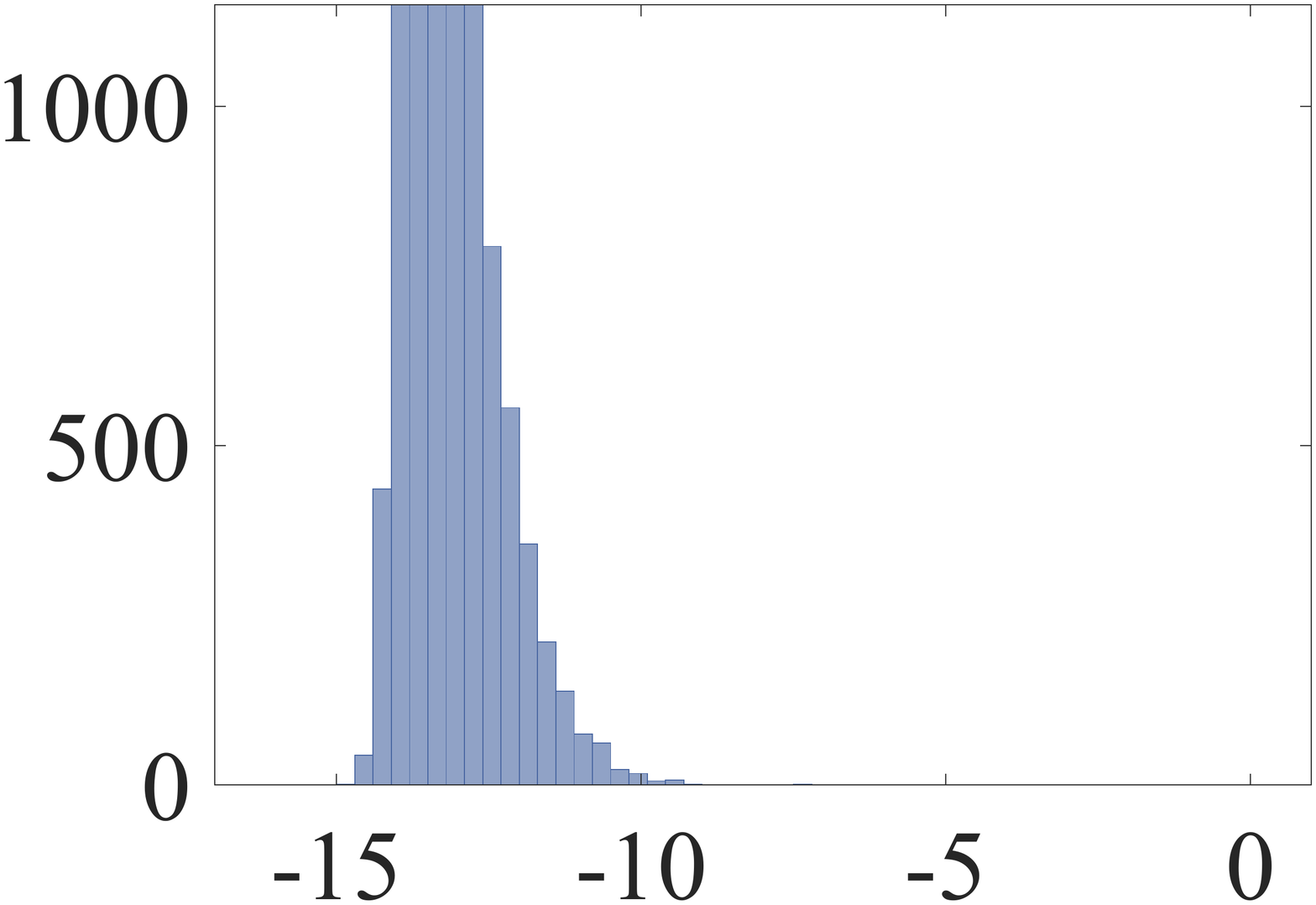} \\
Template size &
$28 \tm 40$ & $39 \tm 95$ & $22 \tm 41$ & $51 \tm 70$ & $47 \tm 55$ \\
$\# \mathcal P$ &
20 & 75 & 26 & 35 & 21 \\
Med. error & 2.87e--14 & 2.63e--09 & 1.11e--12 & 4.11e--13 & 5.99e--14 \\
Ave. time (ms) & $0.8$ & $3.6$ & $0.7$ & $7.2$ & $0.7$ \\
\hline\\
Problem \# &
14 (std) & 18 (std) & 19 (std) & 24 (std) & 25 (nstd) \\
\begin{tabular}{c}Error distrib.\vspace{60pt}\end{tabular}\vspace{-30pt} &
\ig{0.145}{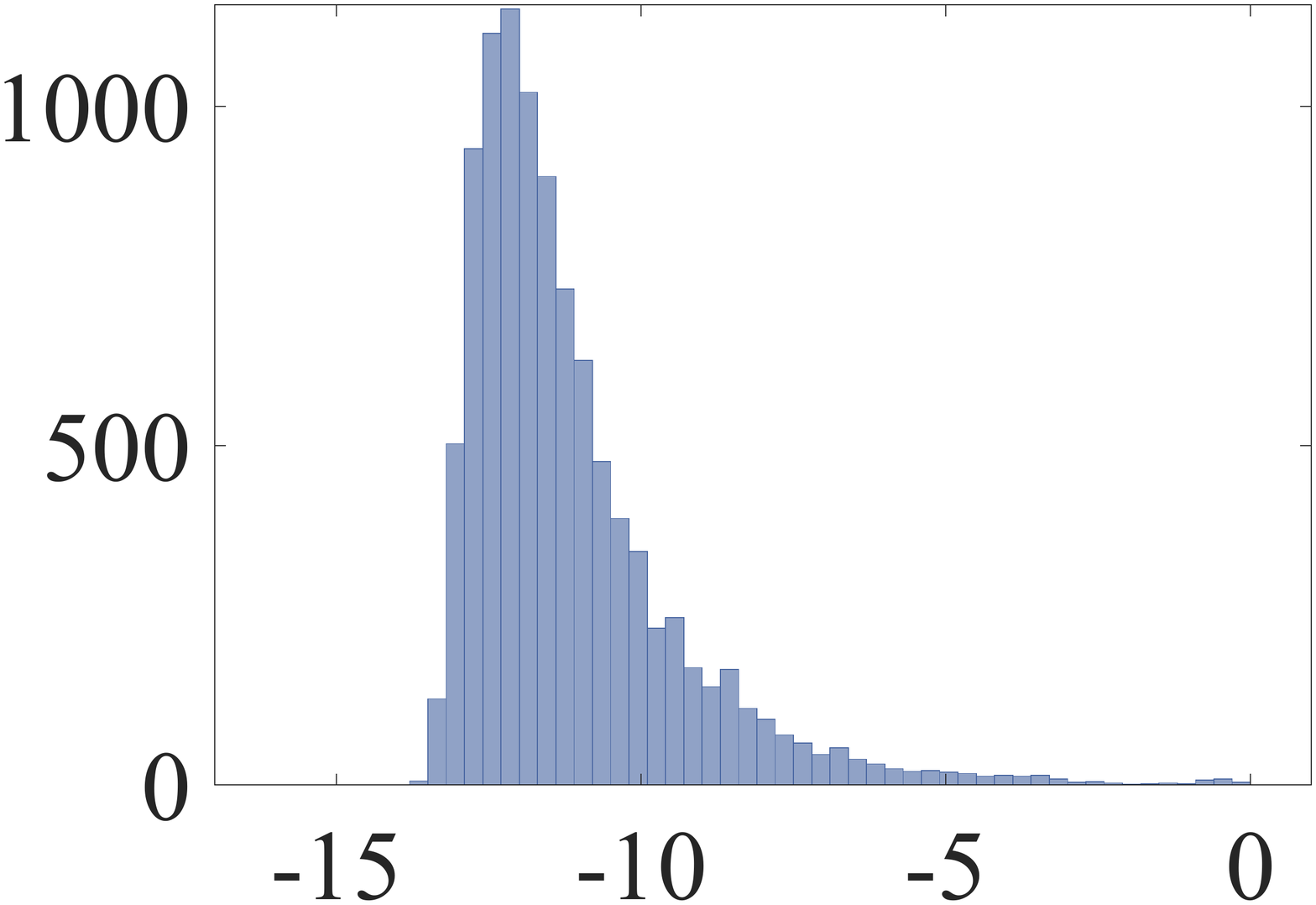} &
\ig{0.145}{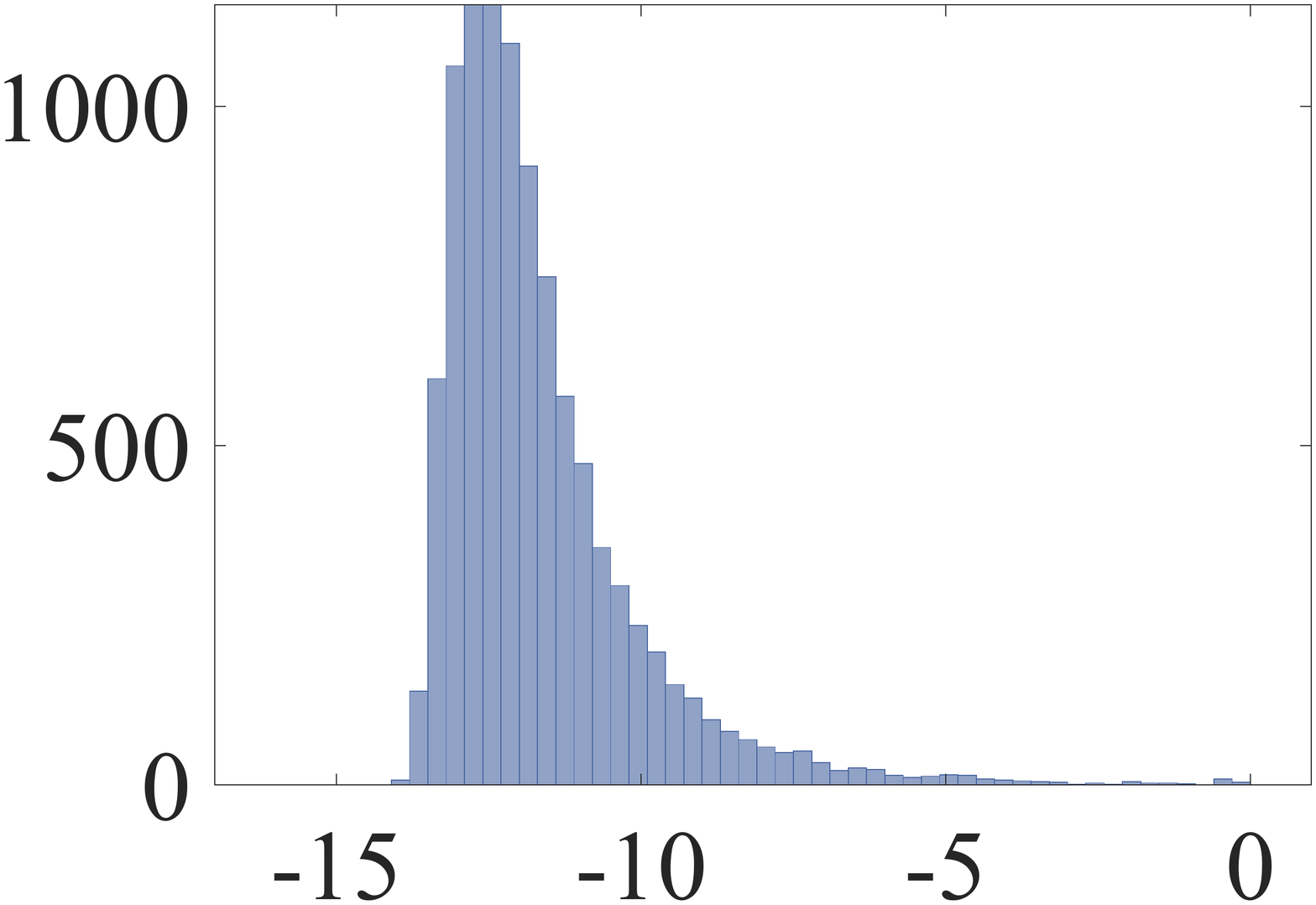} &
\ig{0.145}{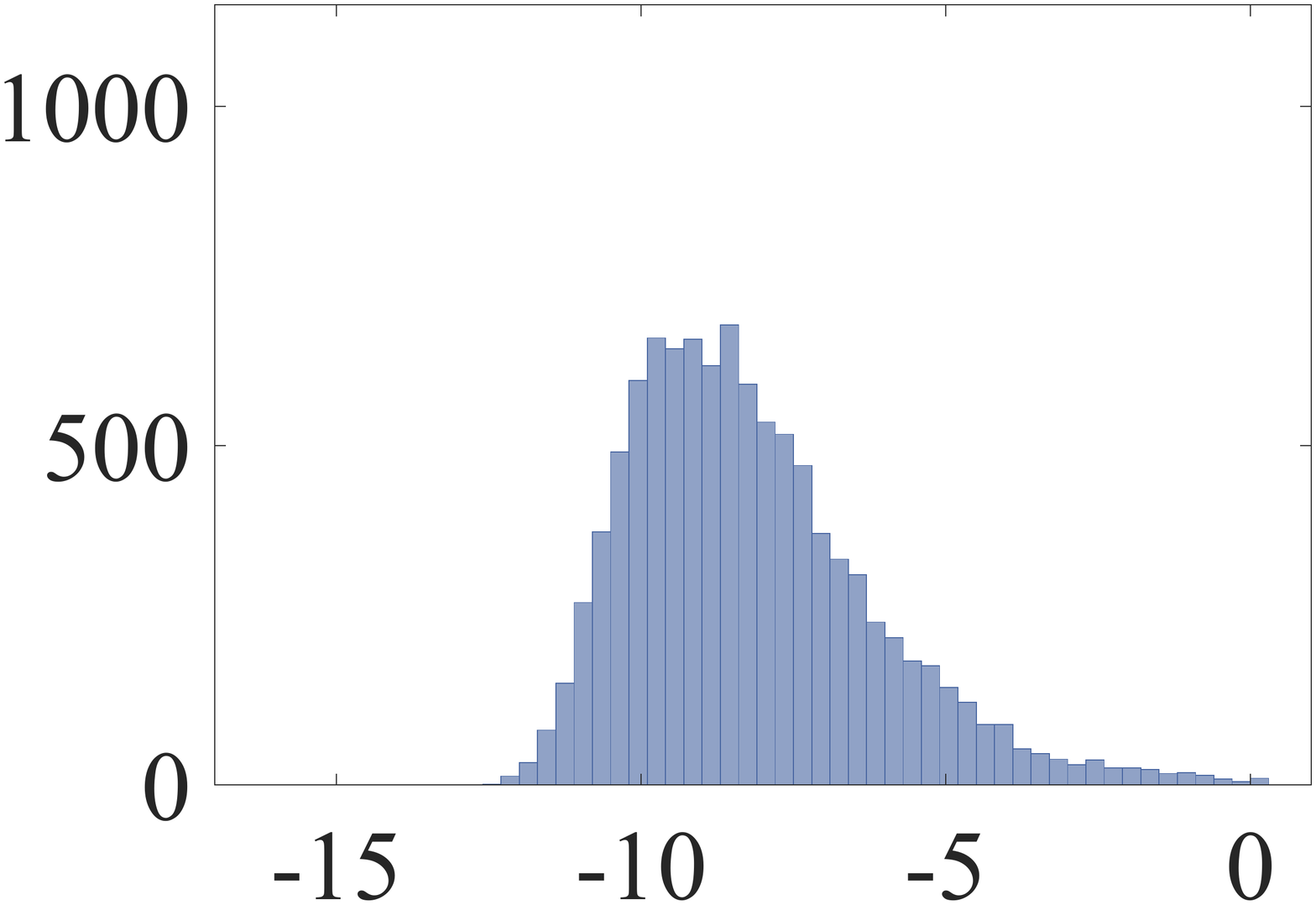} &
\ig{0.145}{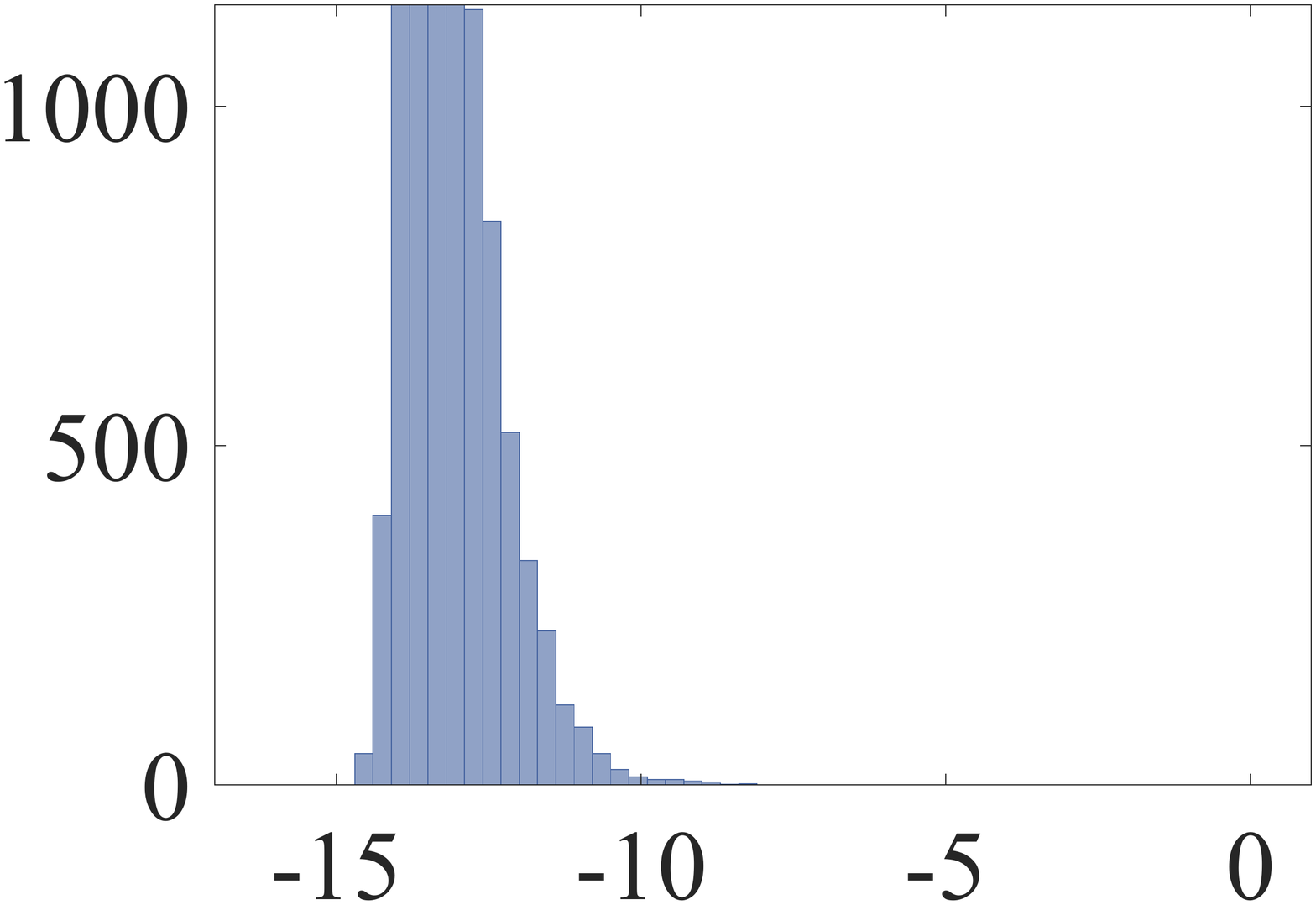} &
\ig{0.145}{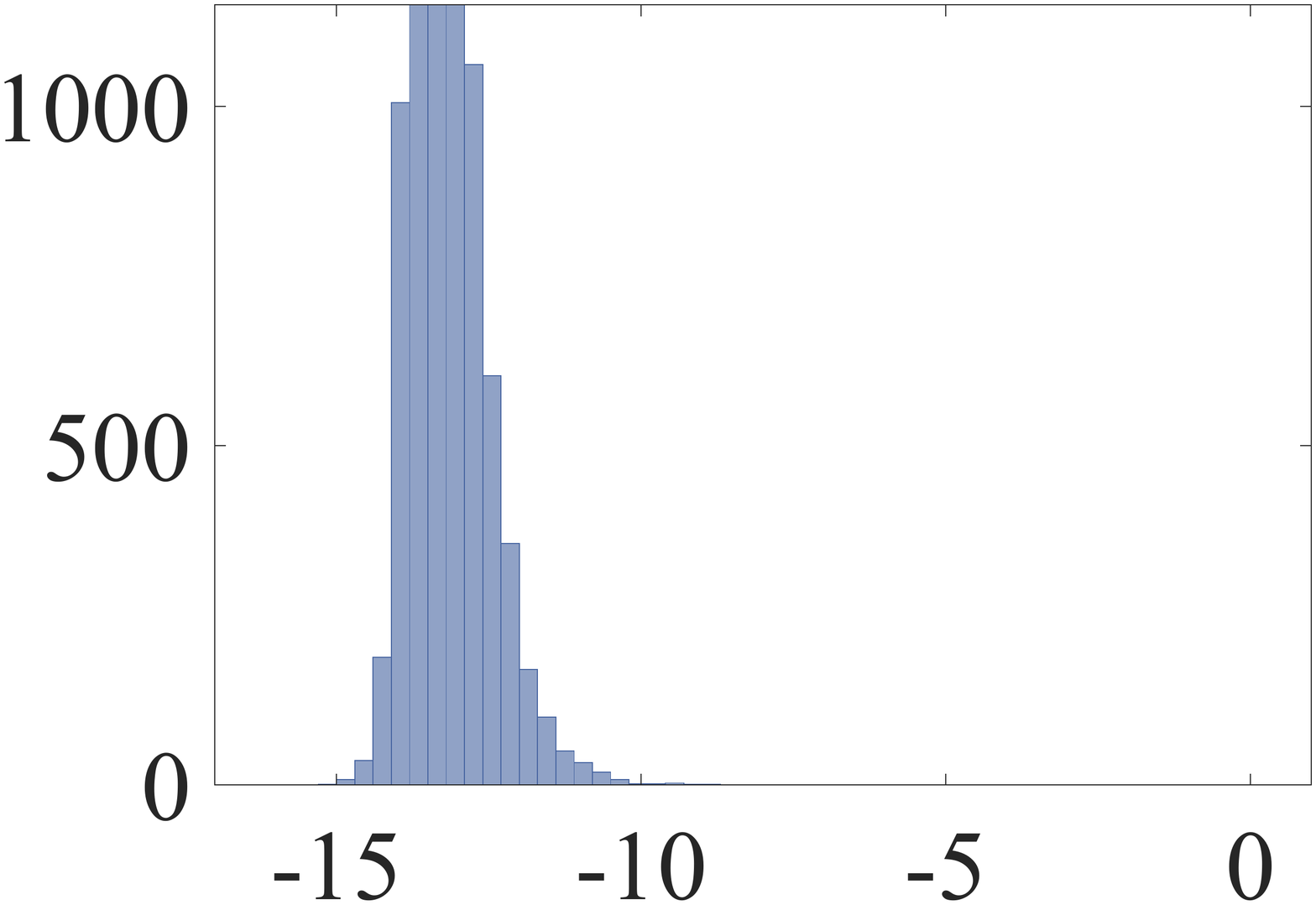} \\
Template size &
$87 \tm 114$ & $87 \tm 114$ & $118 \tm 158$ & $47 \tm 55$ & $16 \tm 36$ \\
$\# \mathcal P$ &
80 & 80 & 40 & 21 & 25 \\
Med. error & 2.24e--12 & 8.00e--13 & 3.01e--09 & 6.10e--14 & 4.78e--14 \\
Ave. time (ms) & $2.4$ & $2.4$ & $2.9$ & $0.7$ & $0.7$ \\
\hline\\
Problem \# &
26 (std) & 27 (nstd) & 30 (nstd) & 34 (std) \\
\begin{tabular}{c}Error distrib.\vspace{60pt}\end{tabular}\vspace{-30pt} &
\ig{0.145}{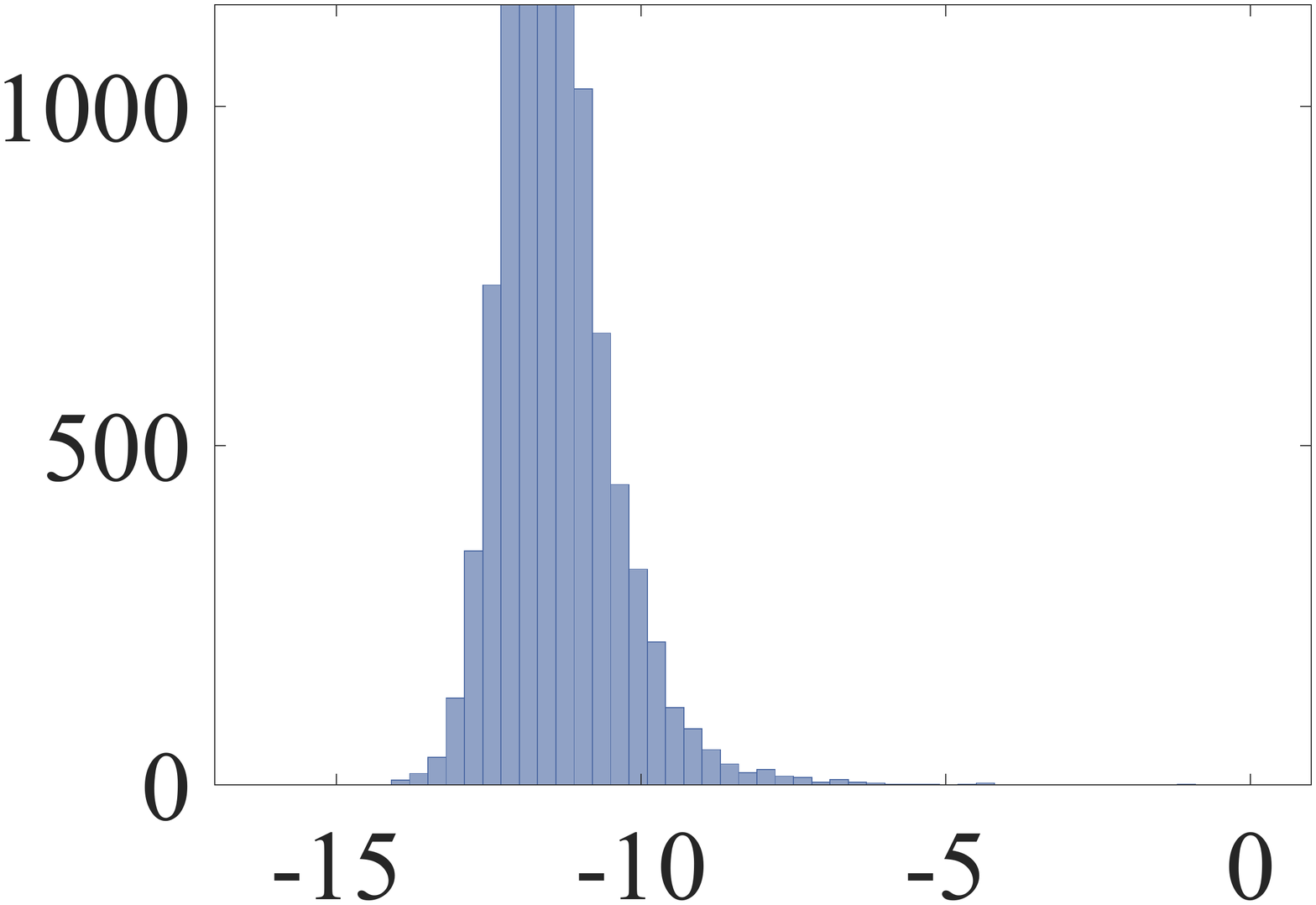} &
\ig{0.145}{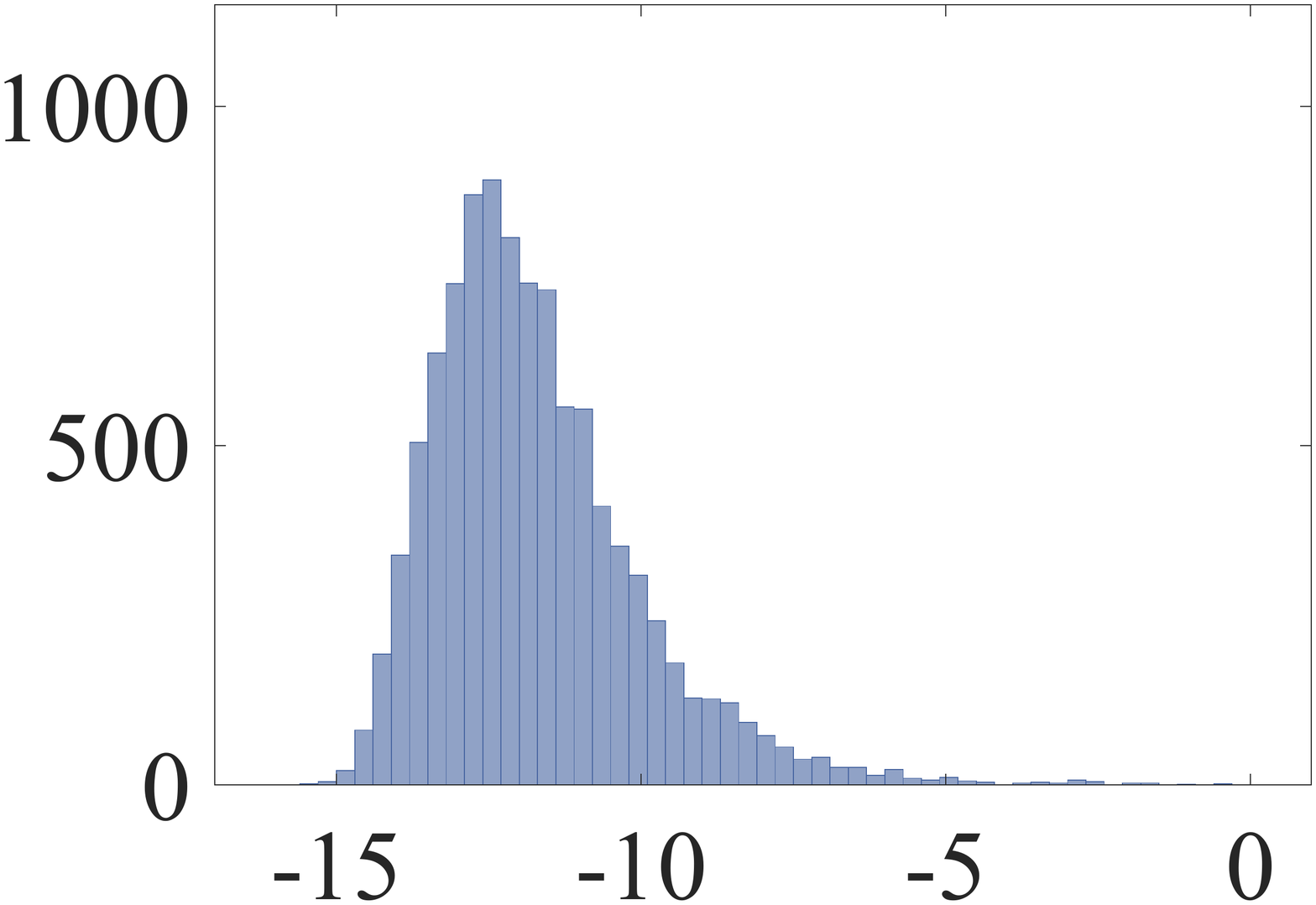} &
\ig{0.145}{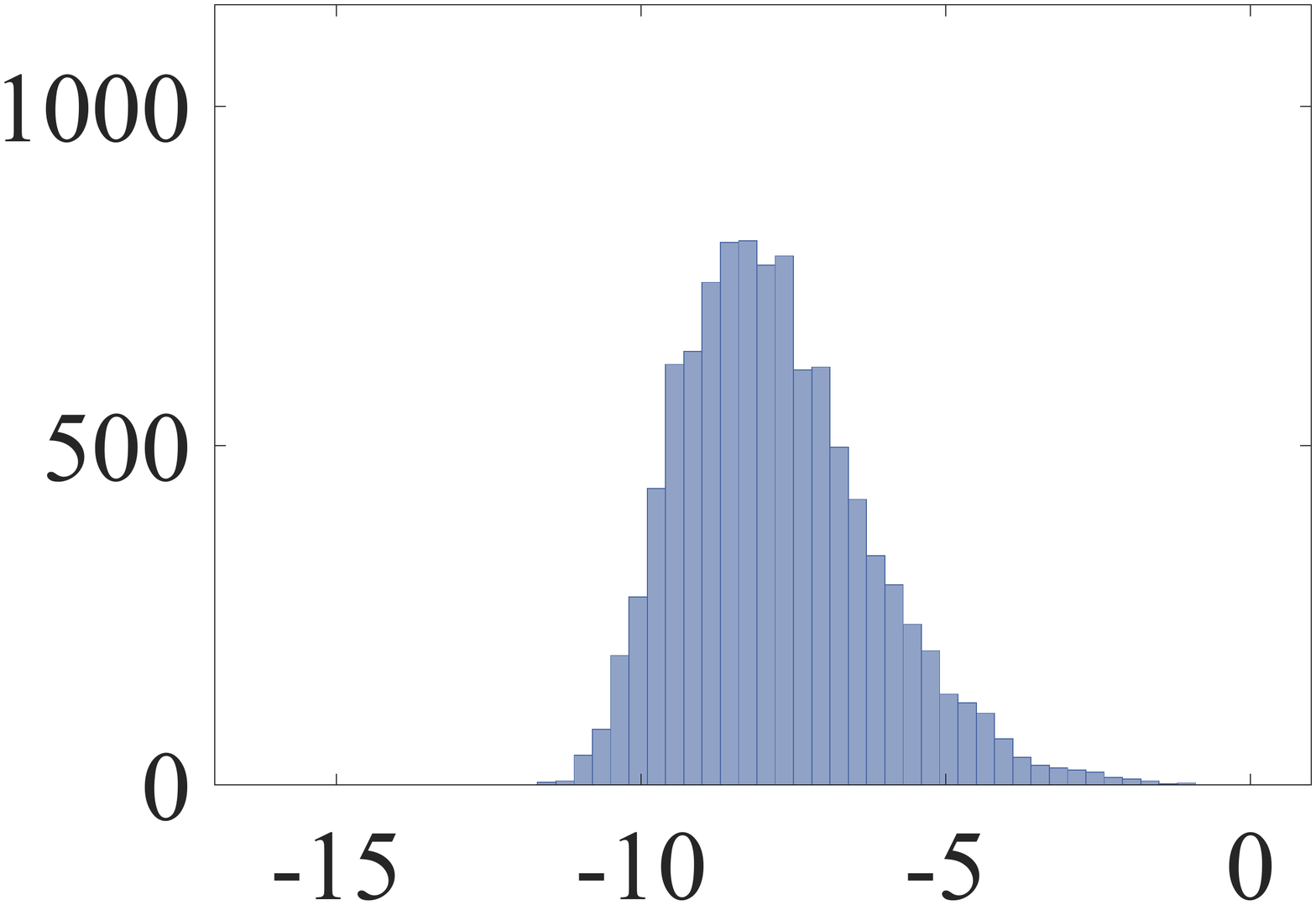} &
\ig{0.145}{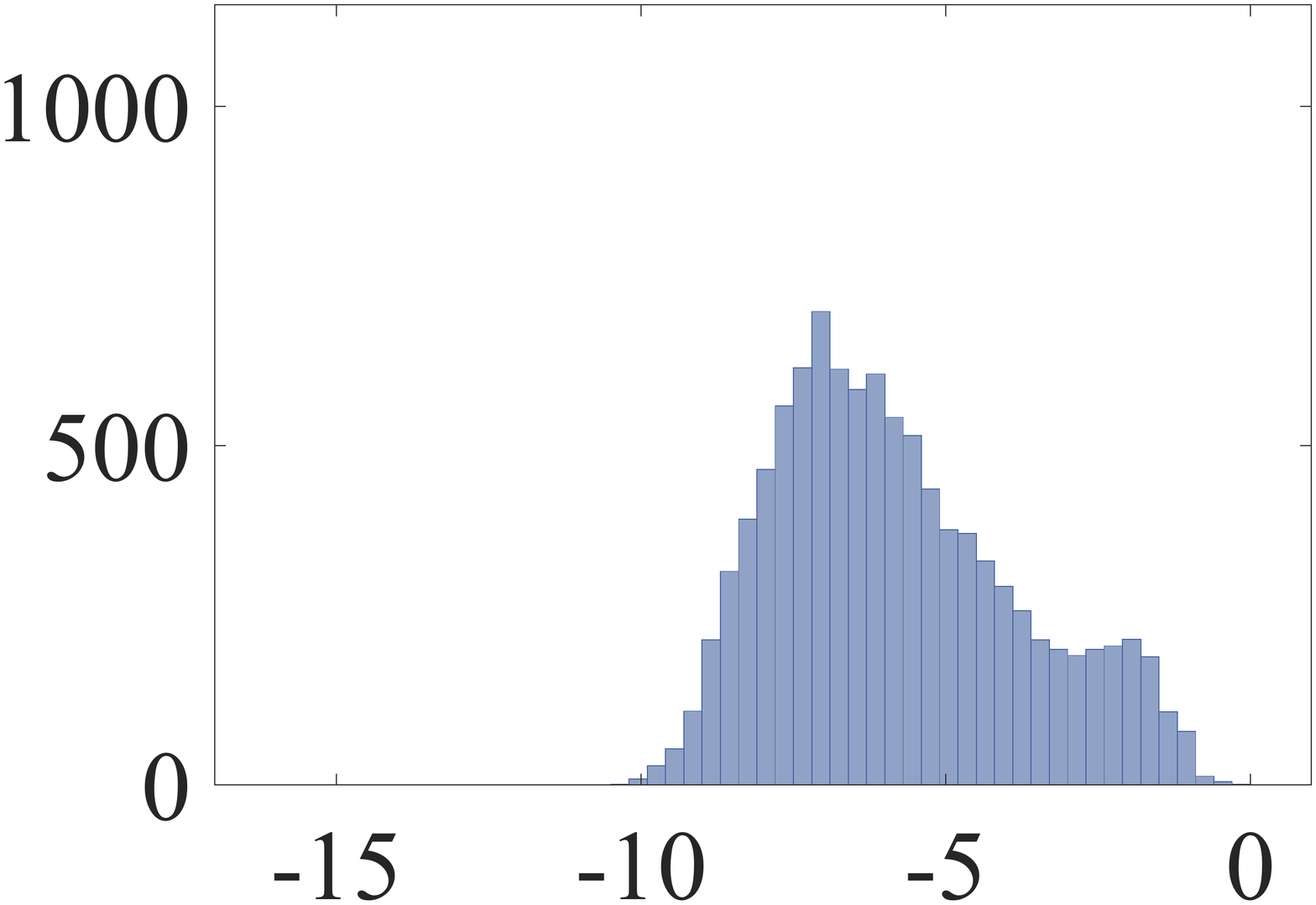} \\
Template size &
$37 \tm 81$ & $40 \tm 46$ & $385 \tm 433$ & $144 \tm 284$ \\
$\# \mathcal P$ &
50 & 7 & 82 & 165 \\
Med. error & 2.96e--12 & 9.27e--13 & 1.09e--08 & 7.47e--07 \\
Ave. time (ms) & $3.1$ & $0.6$ & $19.8$ & $97$ \\
\hline
\end{tabular}
\caption{A continuation of Tab.~\ref{tab:stat} for the remaining 19 minimal solvers. Each histogram shows $\log_{10}$ of numerical error distribution on $10^4$ trials. It is also shown the template size of each problem and the number of permissible monomials ($\# \mathcal P$) used for the column pivoting strategy from~\cite{byrod2009fast}, see Sec.~\ref{sec:pivot}. For problem \#1 the column pivoting was not applied as for this problem $\mathcal P$ is exactly the set of basic monomials. For problem \#19 the column pivoting was not applied as it led to worse results. The runtime includes both constructing the coefficient matrix of the initial system and finding its solutions}
\label{tab:stat3}
\end{table*}

As it was noted in~\cite{byrod2009fast}, matrix $M'_{\mathcal R}$ is often ill conditioned and this is the main cause of numerical instabilities in solving polynomial systems. Also in~\cite{byrod2009fast} the authors proposed the following heuristic method of improving stability. First, the set of basic monomials $\B$ is replaced with the set of \emph{permissible monomials} $\mathcal P = \{p \in \mathcal X \cln a p \in \mathcal X\}$. The partitions for $\mathcal X$ and $M$ now become
\[
\mathcal X = \mathcal E \cup \mathcal R \cup \mathcal P
\quad\text{and}\quad
M = \begin{bmatrix}M_{\mathcal E} & M_{\mathcal R} & M_{\mathcal P} \end{bmatrix}
\]
respectively. Here $\mathcal R = \{a p \cln p \in \mathcal P\} \setminus \mathcal P$ and $\mathcal E$ consists of monomials which are neither in $\mathcal R$ nor in $\mathcal P$. Then the LU decomposition is applied to matrix $\begin{bmatrix}M_{\mathcal E} & M_{\mathcal R}\end{bmatrix}$:
\[
M' = \begin{bmatrix}U_{\mathcal E} & * & * \\ 0 & U_{\mathcal R} & M'_{\mathcal P} \\ 0 & 0 & N'_{\mathcal P}\end{bmatrix},
\]
where $U_{\mathcal E}$, $U_{\mathcal R}$ are upper-triangular matrices and $U_{\mathcal R}$ is square and invertible. This is the starting point for the column pivoting strategy. Let the (pivoted) QR decomposition of matrix $N'_{\mathcal P}$ be
\[
N'_{\mathcal P}\Pi = Q\begin{bmatrix}U_{\mathcal P\setminus \B} & N''_{\B}\end{bmatrix},
\]
where $\Pi$ is the column permutation matrix, $Q$ is orthogonal matrix, $U_{\mathcal P\setminus \B}$ is upper-triangular, square and invertible. Pivoting defined by the matrix $\Pi$ helps to reduce the condition number of $U_{\mathcal P\setminus \B}$ and hence makes the further computation of its inverse matrix numerically more accurate. Let us define
\begin{multline*}
M'' = \begin{bmatrix}I & 0 & 0 \\ 0 & I & 0 \\ 0 & 0 & Q^\top\end{bmatrix} M' \begin{bmatrix}I & 0 & 0 \\ 0 & I & 0 \\ 0 & 0 & \Pi\end{bmatrix} \\= \begin{bmatrix}U_{\mathcal E} & * & * & * \\ 0 & U_{\mathcal R} & M''_{\mathcal P\setminus \B} & M''_{\B} \\ 0 & 0 & U_{\mathcal P\setminus \B} & N''_{\B}\end{bmatrix},
\end{multline*}
where $M'_{\mathcal P}\Pi = \begin{bmatrix}M''_{\mathcal P\setminus \B} & M''_{\B}\end{bmatrix}$. If $\Pi^\top \vect{\mathcal P} = \begin{bmatrix}\vect{\mathcal P\setminus \B} \\ \vect{\B}\end{bmatrix}$, then it follows that
\begin{multline}
\label{eq:RfromBnew}
\begin{bmatrix}\vect{\mathcal R} \\ \vect{\mathcal P\setminus \B}\end{bmatrix} = -\begin{bmatrix}U_{\mathcal R} & M''_{\mathcal P\setminus \B} \\ 0 & U_{\mathcal P\setminus \B}\end{bmatrix}^{-1} \begin{bmatrix}M''_{\B} \\ N''_{\B}\end{bmatrix} \vect{\B} \\= -\begin{bmatrix}U_{\mathcal R}^{-1}M''_{\B} - (U_{\mathcal R}^{-1}M''_{\mathcal P\setminus \B})\,(U_{\mathcal P\setminus \B}^{-1}\,N''_{\B}) \\ U_{\mathcal P\setminus \B}^{-1}N''_{\B}\end{bmatrix} \vect{\B}.
\end{multline}
We note that the set of basic monomials $\B$ depends on the permutation $\Pi$, which in turn depends on the entries of template $M$. Therefore, in general $\B$ can vary depending on problem instance. Since any multiple $a\, b$ for $b \in \B$ belongs to $\mathcal R \cup \mathcal P$, it follows that the action matrix for the new basis $\B$ can be read off from~\eqref{eq:RfromBnew}.

The column pivoting is a universal tool that may significantly enhance numerical accuracy with a certain computational overhead. It can be always applied provided that $\#\mathcal P > \#\B$.

\section{Experimental results}
\label{sec:results}

In this section we test the speed and numerical accuracy of our Matlab solvers for all the minimal problems from Tab.~\ref{tab:templates} and Tab.~\ref{tab:templates2} of the main paper. The experiments were performed on a system with Intel Core i5 CPU @ 2.3~GHz and 8~GB of RAM. The results are presented in Tab.~\ref{tab:stat} and Tab.~\ref{tab:stat3}.

In case the templates for standard and non-standard bases had the same size, we chose the one with smaller numerical error. The column pivoting strategy (see Sec.~\ref{sec:pivot}) was applied for all solvers with $\#\mathcal P > \#\B$. However, for some problems, the set of permissible monomials was manually reduced to improve the speed/accuracy trade-off.

Finally, the numerical error is defined as follows. Let the polynomial system $F = 0$ be written in the form $M(F) U = 0$, where $M(F)$ and $U = \vect{[X]_F}$ are the Macaulay matrix and monomial vector respectively. The matrix $M(F)$ is normalized so that each its row has unit length. Let $\dim \mathbb K[X]/\langle F \rangle = d$, $i$ number all solutions to $F = 0$ including complex ones and $U_i$ be the monomial vector $U$ evaluated at the $i$th solution. We measure the numerical error of our solvers by the value
\[
\Bigl\|M(F) \begin{bmatrix}\frac{U_1}{\|U_1\|_2} & \ldots & \frac{U_d}{\|U_d\|_2} \end{bmatrix}\Bigr\|_2,
\]
where $\|\cdot\|_2$ is the Frobenius norm.

\end{document}